%% file: main.tex
\documentclass{article}

\usepackage{iclr2027_conference,times}

\usepackage{booktabs}
\usepackage{xcolor}
\usepackage{graphicx}
\newcommand{\best}[1]{\textcolor{purple}{\mathbf{#1}}}
\newcommand{\second}[1]{\textcolor{teal}{\mathbf{#1}}}
\usepackage{amsmath,amssymb,amsfonts,amsthm,mathtools,bm}
\usepackage{array}
\usepackage{algorithm}
\usepackage{enumitem}
\usepackage{microtype}
\usepackage{algpseudocode}
\usepackage{multirow}
\usepackage{url}
\usepackage{xspace}
\usepackage{hyperref}
\hypersetup{
    colorlinks,
    linkcolor={red!50!black},
    citecolor={green!50!black},
    urlcolor={blue!80!black}
}

\newtheorem{proposition}{Proposition}

\newtheorem{remark}{Remark}

\newcommand{\D}{\mathcal{D}}
\newcommand{\Xspace}{\mathcal{X}}

\newcommand{\R}{\mathbb{R}}
\newcommand{\E}{\mathbb{E}}
\newcommand{\KL}{\operatorname{KL}}
\newcommand{\N}{\mathcal{N}}
\newcommand{\GP}{\mathcal{GP}}

\newcommand{\tr}{\operatorname{tr}}

\newcommand{\GMVIP}{\textsc{GMVIP}\xspace}

\newcommand{\vx}{\mathbf{x}}
\newcommand{\vy}{\mathbf{y}}
\newcommand{\vz}{\mathbf{z}}
\newcommand{\vw}{\mathbf{w}}
\newcommand{\vb}{\mathbf{b}}
\newcommand{\vh}{\mathbf{h}}
\newcommand{\vf}{\mathbf{f}}
\newcommand{\va}{\mathbf{a}}
\newcommand{\vu}{\mathbf{u}}
\newcommand{\vm}{\mathbf{m}}
\newcommand{\vs}{\mathbf{s}}
\newcommand{\vr}{\mathbf{r}}
\newcommand{\veps}{\boldsymbol{\epsilon}}
\newcommand{\vmu}{\boldsymbol{\mu}}

\newcommand{\vtheta}{\boldsymbol{\theta}}
\newcommand{\veta}{\boldsymbol{\eta}}
\newcommand{\vvartheta}{\boldsymbol{\vartheta}}
\newcommand{\vpsi}{\boldsymbol{\psi}}
\newcommand{\X}{\mathbf{X}}
\newcommand{\Y}{\mathbf{Y}}
\newcommand{\Z}{\mathbf{Z}}
\newcommand{\U}{\mathbf{U}}
\newcommand{\K}{\mathbf{K}}
\newcommand{\Bmat}{\mathbf{B}}
\newcommand{\Cmat}{\mathbf{C}}
\newcommand{\Dmat}{\mathbf{D}}
\newcommand{\Lmat}{\mathbf{L}}
\newcommand{\Smat}{\mathbf{S}}
\newcommand{\Imat}{\mathbf{I}}
\newcommand{\Psib}{\boldsymbol{\Psi}}
\newcommand{\Cov}{\operatorname{Cov}}

\title{Generalized Matheron Variational Implicit Processes}

\author{%
  Luis A. Ortega \\
  Aalborg University \\
  \texttt{laoa@cs.aau.dk}\\
  \And
  Andrés R. Masegosa \\
  Aalborg University\\
  \texttt{arma@cs.aau.dk}\\
  \And
  Thomas D. Nielsen \\
  Aalborg University\\
  \texttt{tdn@cs.aau.dk}
}

\iclrfinalcopy
\begin{document}
\maketitle

\begin{abstract}
Implicit-process priors specify distributions over functions through sample-forward mechanisms such as Bayesian neural networks and stochastic simulators, but their function-space densities are typically unavailable. We introduce Generalized Matheron Variational Implicit Processes (GMVIP), a pathwise variational family for posterior inference with such priors.  For Gaussian-process priors, GMVIP recovers the standard inducing-variable variational GP construction; for general implicit priors, its empirical covariance construction preserves the prior mean and covariance in the population limit. GMVIP constructs posterior samples by drawing a function from the prior and applying a correction anchored at a set of inducing inputs. The effect of this correction away from the inducing inputs is determined directly from prior samples, allowing the posterior to retain the structure and variability of the original implicit process.  The (surrogate) prior and variational posterior use the same pathwise construction and differ only in the distribution of whitened inducing coefficients, yielding a tractable coefficient-space Kullback-Leibler divergence. Experiments on regression, classification, and forecasting with simulator-defined and retrieval-conditioned empirical trajectory priors show that GMVIP is broadly competitive with existing methods.
\end{abstract}

\section{Introduction}

Bayesian neural networks are commonly formulated by placing a prior over
weights and performing posterior inference in parameter space
\citep{mackay1992practical,neal1996bayesian,blundell2015weight,
hernandezlobato2015probabilistic,gal2016dropout}. Although parameter-space
formulations enable familiar inference algorithms, they are high-dimensional
and inherit issues such as non-identifiability---including neuron-permutation
symmetries---strong posterior dependencies, multimodality, and non-injective
mappings from parameters to predictors
\citep{sun2019functional,burt2020understanding,wild2022generalized,
rudner2022tractable}. Since predictions depend only on the induced
input--output map, it is natural to instead formulate Bayesian inference directly in
function space \citep{sun2019functional,ma2019variational,ortega2026uncertainty}.

Function-space inference targets the posterior distribution over functions
rather than over parameters. This perspective aligns the variational object
with the predictive object, but it also introduces a central computational
difficulty: for flexible functional-space priors defined by Bayesian neural networks, neural
samplers, and simulator-defined stochastic processes, the density of the induced
function-space prior is rarely available or tractable
\citep{ma2019variational,ortega2023deep}. Existing approaches address the
intractable process Kullback-Leibler divergence in different ways: finite measurement sets
with Stein score estimates \citep{sun2019functional}, local linear-Gaussian
approximations at finite context sets \citep{rudner2022tractable}, or
alternative discrepancies between Gaussian measures
\citep{wild2022generalized}. \citet{burt2020understanding} analyze when
process-KL objectives are ill-defined and how finite-marginal approximations
behave.

Implicit processes formalize this sample-forward setting.  A draw from an implicit process gives a coherent random function that can be evaluated at any finite input set, but the corresponding finite-dimensional densities are not assumed to be available \citep{ma2019variational}.  Variational Implicit Processes (VIPs) make inference tractable by drawing a finite collection of prior functions and forming a surrogate over linear combinations of those samples \citep{ma2019variational}.  Sparse Implicit Processes introduce inducing-point structure for scalability \citep{rodriguez2022sparse}, and Deep Variational Implicit Processes build deeper implicit-process priors \citep{ortega2023deep}.  Recently, Flow-Transformed Implicit Processes (FTIP) showed that a Gaussian posterior over surrogate coefficients can smooth or collapse asymmetric, heavy-tailed, and multimodal posterior structure, and instead proposed normalizing flows over the finite surrogate variables \citep{ortega2026flow}.  These results suggest that posterior expressiveness is a central design choice in function-space inference.

This paper studies a complementary source of expressiveness. Instead of
changing only the distribution over linear-combination coefficients, we
change the pathwise construction of the surrogate itself. Conditioning
jointly Gaussian variables admits a pathwise representation through  a joint prior
draw plus a covariance-gain correction, commonly called Matheron's rule;
applying it to Gaussian-process marginals yields conditional function samples
\citep{wilson2020efficiently}. The corresponding conditional mean and
covariance follow from standard Gaussian-process conditioning
\citep{rasmussen2006gaussian}. Sparse variational GPs introduce inducing
variables and reduce the KL regularizer to
\(\KL(q(\vu)\|p(\vu))\)
\citep{titsias2009variational,hensman2013gaussian}. We take this
inducing-value viewpoint to implicit processes: sample a prior path, replace
its values at inducing locations, and retain the residual path away from those
locations.

\begin{figure}[t]
    \centering
    \includegraphics[width=0.97\linewidth]{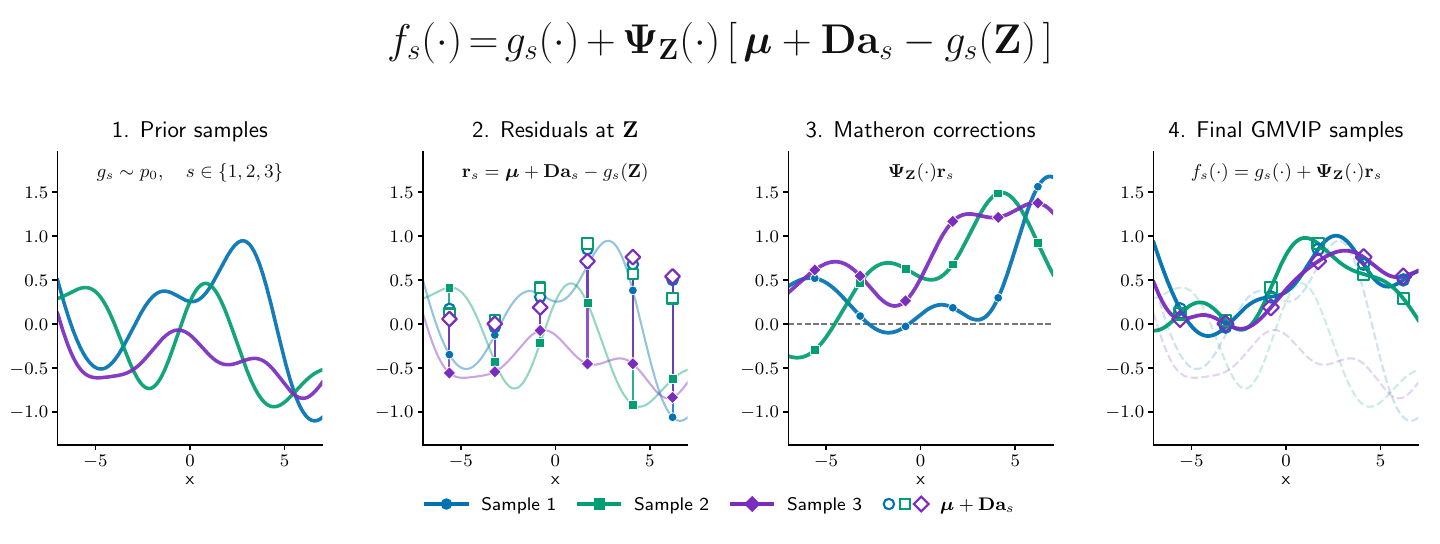}
    \caption{Pathwise construction of a \GMVIP sample. Starting from a prior draw \(g\sim p_0\), we evaluate the path at inducing inputs \(\Z\) and form the residual between the sampled replacement values \(\vmu+\Dmat\va\) and \(g(\Z)\). The Matheron operator \(\Psib_{\Z}\) propagates this residual across the input domain. Adding the resulting smooth correction to \(g\) yields a sample \(f\) that matches the replacement values at \(\Z\).}\label{fig:GMVIP}
\end{figure}

For clarity, we first present the scalar-output construction; the joint
multioutput form is given in Section~\ref{sec:methodology}. Let \(g\sim p_0\)
be a prior function sample, let \(\Z=(\vz_1,\ldots,\vz_M)^\top\) be inducing
locations, and let \(\Psib_{\Z}\) be a replacement operator. For a latent
vector \(\va\in\R^M\), define
\begin{equation}
    f(\cdot)
    = g(\cdot) + \Psib_{\Z}(\cdot)
      \left[\vmu+\Dmat\va-g(\Z)\right].
    \label{eq:intro_map}
\end{equation}
The vector \(\vmu+\Dmat\va\) gives replacement values at the inducing locations, while \(g\) supplies the residual path (Figure~\ref{fig:GMVIP} shows an illustration on the idea).  The surrogate prior uses \(\va\sim\N(\mathbf{0},\Imat_M)\), and the variational posterior replaces this with \(q_\phi(\va)\).  Because the prior and posterior share the same conditional map from \((g,\va)\) to \(f\), the variational regularizer is exactly \(\KL(q_\phi(\va)\|\N(\mathbf{0},\Imat_M))\).  Thus the method keeps the tractability of VIP-style coefficient inference while replacing linear aggregation with a Matheron-style inducing-value replacement. We call the resulting family Generalized Matheron Variational Implicit Processes (\GMVIP).  The construction is modular.  The operator \(\Psib_{\Z}\) can be estimated from empirical prior-sample covariance, giving the closest analogue of GP Matheron conditioning, e.g., \(\Psib_{\Z}=\hat\K(\cdot,\Z)\hat\K(\Z,\Z)^{-1}\) for empirical covariance \(\hat\K\).

Our contributions are: (i) We introduce a general Matheron inducing-value replacement family for variational inference with implicit-process priors.  The family produces coherent posterior function samples and requires only prior sampling and likelihood evaluation. (ii) We show that the augmented KL reduces exactly to a coefficient-space KL, and that under a cardinal right-inverse condition the same identity holds for the induced function marginal.  We also show process coherence for arbitrary query sets. (iii) We show exact GP/SVGP recovery for Gaussian priors and population-level
preservation of the original prior mean and covariance for general implicit
processes.

\section{Related Work}

\paragraph{Function-space variational inference.}
Function-space variational inference replaces approximate posteriors over parameters with approximate posteriors over predictors.  This is especially relevant for Bayesian neural networks, where many parameter configurations can represent the same function.  Functional variational Bayesian neural networks (fBNNs) formulate variational inference directly over the stochastic process induced by a network, typically approximating the process KL on finite measurement sets \citep{sun2019functional}.  Subsequent work clarified the measure-theoretic and optimization issues that arise when the variational object is a stochastic process rather than a finite-dimensional parameter vector \citep{burt2020understanding,wild2022generalized}.  TFSVI constructs tractable finite-dimensional function-space approximations using local linearization and neural network Jacobians \citep{rudner2022tractable}.  \GMVIP shares the function-space objective, but avoids measurement-set density-ratio estimation or linearizations by defining a surrogate prior and posterior with a shared pathwise map.

\paragraph{Implicit processes.}
Implicit processes define stochastic processes through sampling mechanisms
rather than explicit densities \citep{ma2019variational}. VIP constructs
finite surrogates from prior samples and performs variational inference over
surrogate coefficients, yielding a tractable latent KL. Sparse Implicit
Processes (SIP) introduce inducing variables and estimate the intractable
divergence between implicit inducing distributions using adversarial
density-ratio estimation \citep{rodriguez2022sparse}; Deep Variational
Implicit Processes compose implicit-process layers to obtain richer priors
\citep{ortega2023deep}. FTIP makes the VIP surrogate posterior more expressive
by replacing the Gaussian coefficient posterior with a normalizing flow
\citep{ortega2026flow}. Our method is complementary to FTIP. Rather than
enriching \(q(\va)\) with an invertible transformation, we modify the pathwise
surrogate: posterior samples are Matheron replacements of prior paths, not
linear combinations of prior samples. Moreover, VIP/FTIP posterior draws are confined to the finite span \(\mathrm{span}\{g^{(s)}\}_{s=1}^S\) of the \(S\) pre-sampled prior functions, whereas \GMVIP draws a fresh prior path for each sample, preserving functional variation outside this finite-dimensional subspace.

\paragraph{Matheron sampling, inducing variables, and interpolation.}
Applying Matheron's rule to Gaussian-process marginals gives conditional
function draws as a prior path plus a finite-dimensional covariance correction
and underlies efficient pathwise GP posterior sampling
\citep{wilson2020efficiently}; the corresponding Gaussian conditioning
formulas are standard \citep{rasmussen2006gaussian}. Sparse variational GPs
use inducing variables to obtain tractable objectives whose KL term is
\(\KL(q(\vu)\|p(\vu))\)
\citep{titsias2009variational,hensman2013gaussian}. Structured kernel
interpolation, instantiated by KISS-GP, interpolates kernel quantities from
inducing grid points to scale GP inference \citep{wilson2015kernel}.
\GMVIP borrows the pathwise inducing-value structure but applies it to priors
that are only sampleable.

The empirical covariance-gain correction has the same linear-update algebra
as statistical interpolation \citep{lorenc1981global} and the analysis step of
the ensemble Kalman filter
\citep{evensen1994sequential,burgers1998analysis}; ensemble Kalman smoothing
applies analogous ensemble-covariance corrections to earlier states
\citep{evensen2000smoother}. GMVIP differs by optimizing a variational law
over inducing replacement values and regularizing it with a
coefficient-space KL.

\section{Background: Function-Space Variational Inference and Implicit Processes}
\label{sec:background_ips}

Let \(\D=\{(\vx_n,\vy_n)\}_{n=1}^N\) denote the observed dataset, with \(\vx_n\in\Xspace\) and outputs \(\vy_n\in\R^C\). We write \(\X=(\vx_1,\ldots,\vx_N)^\top\) and \(\Y=(\vy_1,\ldots,\vy_N)^\top\), and assume a likelihood factorization
\begin{equation}
    p(\Y\mid f,\X)=\prod_{n=1}^N p(\vy_n\mid f(\vx_n)).
    \label{eq:likelihood_background}
\end{equation}
Bayesian prediction is naturally a distribution over functions, since the likelihood depends on the predictor \(f\) rather than on a particular parameterization of \(f\). This observation motivates function-space variational inference, where the goal is to approximate the posterior process \(p(f\mid\D)\) directly. In its ideal form, one optimizes
\begin{equation}
    \mathcal L(q)
    =
    \E_{q(f)}[\log p(\Y| f,\X)]
    -
    \KL(q(f)\|p_0(f)),
    \label{eq:function_space_elbo_background}
\end{equation}
where \(p_0\) is the prior process. This objective is conceptually attractive
because it compares distributions over predictors, avoiding some of the
redundancies and symmetries of weight-space inference. However, the
process-level KL is rarely available. Existing methods use finite measurement
sets with Stein score estimation \citep{sun2019functional}, analyze or
approximate finite-dimensional marginals \citep{burt2020understanding}, replace
the KL with discrepancies between Gaussian measures
\citep{wild2022generalized}, construct local linear-Gaussian approximations
\citep{rudner2022tractable}, or introduce inducing representations with
adversarial density-ratio estimation \citep{rodriguez2022sparse}. These
approximations introduce additional choices: where to place measurement or
inducing points, how to estimate finite-dimensional densities or ratios, and
how to ensure that the resulting objectives define coherent posterior samples.

\section{Methodology}
\label{sec:methodology}

We introduce \GMVIP as a pathwise variational family for implicit-process priors.
The central idea is to preserve the residual structure of a prior sample while
replacing its values at a finite set of inducing inputs. This gives a posterior
process that remains sample-forward, has coherent joint function samples, and
admits a finite-dimensional KL.

Let \(\Z=(\vz_1,\ldots,\vz_M)^\top\) be inducing inputs and let
\(g\sim p_0\) be a sample from the implicit-process prior. Given a whitened
coefficient vector \(\va\in\R^{d_a}\), where \(d_a=M\) for scalar outputs and
\(d_a=MC\) for a joint \(C\)-output model, define the replacement inducing values
\begin{equation}
    \vu(\va)=\vmu+\Dmat\va ,
    \label{eq:inducing_values_bold}
\end{equation}
where \(\vmu\) and \(\Dmat\) are fixed, prior-estimated, or determined by the
chosen replacement operator. A \GMVIP sample is obtained by replacing the prior
path at \(\Z\) and propagating the correction to arbitrary inputs:
\begin{equation}
    T(g,\va)(\cdot)
    =
    g(\cdot)+\Psib_{\Z}(\cdot)\left[\vu(\va)-g(\Z)\right].
    \label{eq:general_map_bold}
\end{equation}
The operator \(\Psib_{\Z}\) controls how inducing-value corrections are extended
away from \(\Z\). The surrogate prior and posterior share this same pathwise
map:
\begin{align}
    p_{\Z}(g,\va,f)
    &=p_0(g)\,p(\va)\,\delta[f=T(g,\va)], \label{eq:surrogate_prior_posterior_bold}\\
    q_{\Z,\phi}(g,\va,f)
    &=p_0(g)\,q_\phi(\va)\,\delta[f=T(g,\va)] .
    \label{eq:surrogate_prior_posterior_bold2}
\end{align}
Thus inference changes only the law of the finite coefficient vector \(\va\); the
residual path is still supplied by the original implicit prior. We use
\(p(\va)=\N(\mathbf{0},\Imat_{d_a})\) and
\(q_\phi(\va)=\N(\vm_{a},\Smat_{a})\), with \(\vm_a\in\R^{d_a}\) and covariance
\(\Smat_a\in\R^{d_a\times d_a}\). Samples are drawn by reparameterization. For a minibatch
\(\mathcal B\), the Monte Carlo objective is
\begin{equation}
    \widehat{\mathcal L}_t(\phi,\vpsi)
    =
    \frac{N}{|\mathcal B|R_{\mathrm{tr}}}
    \sum_{r=1}^{R_{\mathrm{tr}}}\sum_{n\in\mathcal B}
    \log p_{\vpsi} \big(\vy_n\mid T(g^{(r)},\va^{(r)})(\vx_n)\big)
    -\KL \left(q_\phi(\va)\|p(\va)\right),
    \label{eq:single_elbo_bold}
\end{equation}
with \(g^{(r)}\sim p_0\), \(\va^{(r)}\sim q_\phi\), and
\(R_{\mathrm{tr}}\) Monte Carlo samples per optimization step. Here \(\vpsi\)
denotes the likelihood parameters. The most direct instantiation estimates the Matheron correction from prior
samples. Given a fixed prior bank \(\mathcal G=\{g^{(b)}\}_{b=1}^B\),
\(g^{(b)}\sim p_0\),
define
\begin{align}
    \hat\vmu(\cdot)
    &=
    \frac{1}{B}\sum_{b=1}^B g^{(b)}(\cdot),
    &
    \hat\K(\cdot,\Z)
    &=
    \frac{1}{B-1}\sum_{b=1}^B
    \bigl(g^{(b)}(\cdot)-\hat\vmu(\cdot)\bigr)
    \bigl(g^{(b)}(\Z)-\hat\vmu(\Z)\bigr)^\top .
\end{align}
Let \(\Lmat\) be the Cholesky decomposition of \(\hat\K(\Z,\Z)\). The empirical replacement operator is
\begin{equation}
    \Psib^{\mathrm{emp}}_{\Z}(\cdot)
    =
    \hat\K(\cdot,\Z)\hat\K(\Z,\Z)^{-1},
    \qquad
    \vmu=\hat\vmu(\Z),
    \qquad
    \Dmat=\Lmat .
    \label{eq:empirical_operator_bold}
\end{equation}
Therefore,
\begin{equation}
    f(\cdot)
    =
    g(\cdot)
    +
    \hat\K(\cdot,\Z)\hat\K(\Z,\Z)^{-1}
    \left[
        \hat\vmu(\Z)+\Lmat\va-g(\Z)
    \right].
    \label{eq:empirical_sample_bold}
\end{equation}
This is the prior-matched \GMVIP operator: it uses the covariance geometry
induced by the implicit prior itself. Its quality therefore depends on the
prior-bank estimate, but when this estimate is reliable it provides the closest
analogue of GP Matheron sampling for non-Gaussian implicit processes. The complete training and evaluation procedures are given in
Algorithms~\ref{alg:gmvip-training} and~\ref{alg:gmvip-evaluation} in
Appendix~\ref{app:training-evaluation}. Surrogate priors are already used in
VIP~\citep{ma2019variational}; Appendix~\ref{app:surrogate_prior} studies how
much these priors differ from the original prior.

\paragraph{Multioutput case.}
For vector-valued outputs \(f(\vx)\in\R^{C}\), \GMVIP supports both
joint and channel-wise multioutput parameterizations. In the joint
parameterization, we stack inducing evaluations and coefficients across
outputs and construct the empirical inducing covariance
\(\hat\K(\Z,\Z)\in\R^{MC\times MC}\), including cross-output
covariance blocks. The variational posterior
\(q_\phi(\va)=\N(\vm_a,\Smat_a)\) may likewise use a full
\(MC\times MC\) covariance. For larger output spaces, we also use a channel-wise approximation in which
the empirical operator and variational covariance are block-diagonal across
outputs.

\subsection{Theoretical properties}

We now state the main properties of the construction. Proofs, together with
additional results on prior-distortion bounds,
prior reconstruction error, and sample-path regularity, are given in
Appendix~\ref{app:theory}.

\begin{proposition}[Process validity]
\label{prop:process_validity}
Assume \(p_0\) is a valid stochastic process, that \(\Psib\) is row-consistent and that the same latent pair
\((g,\va)\) is reused for all query inputs. Then, for any coefficient law
\(\nu(\va)\), the finite-dimensional distributions induced by
\[
    f(\cdot)
    =
    g(\cdot)+\Psib_{\Z}(\cdot)\bigl[\vmu+\Dmat\va-g(\Z)\bigr],
    \qquad
    g\sim p_0,\quad \va\sim\nu,
\]
satisfy Kolmogorov consistency. Hence both the \GMVIP surrogate prior and
variational posterior define valid stochastic processes.
\end{proposition}

This result ensures that \GMVIP is not merely a collection of unrelated
finite-dimensional approximations. The same sampled path \(g\) and the same
coefficient \(\va\) define all evaluations of a function, so restrictions and
permutations of finite query sets are compatible.

\begin{proposition}[Coefficient KL reduction]
\label{prop:coefficient_kl}
Let \(p_{\Z}\) and \(q_{\Z,\phi}\) be defined by
Equations~\eqref{eq:surrogate_prior_posterior_bold} and \eqref{eq:surrogate_prior_posterior_bold2}. Then
\begin{equation}
    \KL \left(q_{\Z,\phi}(g,\va,f)\,\|\,p_{\Z}(g,\va,f)\right)
    =
    \KL \left(q_\phi(\va)\,\|\,p(\va)\right).
    \label{eq:kl_augmented_bold}
\end{equation}
Moreover, if \(q_{\Z,\phi}^{f}\) and \(p_{\Z}^{f}\) denote the corresponding
marginal laws over functions, then
\begin{equation}
    \KL(q_{\Z,\phi}^{f}\,\|\,p_{\Z}^{f})
    \leq
    \KL \left(q_\phi(\va)\,\|\,p(\va)\right).
    \label{eq:kl_marginal_bound_bold}
\end{equation}
If \(\Psib_{\Z}(\Z)=\Imat_{d_a}\) and \(\Dmat\) is invertible, the inequality in
\eqref{eq:kl_marginal_bound_bold} is an equality.
\end{proposition}

Since \(p_{\Z}\) and \(q_{\Z,\phi}\) share the same pathwise map
\((g,\va)\mapsto f\), their augmented KL is exactly
\(\KL(q_\phi(\va)\|p(\va))\), avoiding measurement-set estimates of a
functional KL. The marginal function-space KL is upper bounded by this quantity,
and the bound is tight whenever \(\va\) is recoverable from \(f\), for example
when \(\Psib_{\Z}(\Z)=\Imat_{d_a}\) and \(\Dmat\) is invertible, which satisfies for the empirical operator.

\begin{proposition}[GP and SVGP recovery]
\label{prop:gp_svgp_recovery}
For a scalar-valued process, let \(p_0=\GP(\mu,\kappa)\), assume
\(\kappa(\Z,\Z)\succ0\), and choose
\[
    \Psib_{\Z}(\X)=\kappa(\X,\Z)\kappa(\Z,\Z)^{-1},
    \qquad
    \vmu=\mu(\Z),
    \qquad
    \Dmat\Dmat^\top=\kappa(\Z,\Z).
\]
If \(\va\sim\N(\mathbf{0},\Imat_M)\), then \(T(g,\va)\sim\GP(\mu,\kappa)\).
Moreover, for any coefficient law \(q_\phi(\va)\), with
\(\vu=\vmu+\Dmat\va\), the induced posterior is
\[
    q(f)=\int p_0(f | \vu)\,q_\phi(\vu)\, \mathrm d\vu .
\]
Thus \GMVIP recovers both the original GP prior and the standard
inducing-variable GP family. In particular, any Gaussian SVGP posterior
\(q(\vu)=\N(\vm_u,\Smat_u)\) is recovered by taking
\(q_\phi(\va)=\N(\Dmat^{-1}(\vm_u-\vmu),
\Dmat^{-1}\Smat_u\Dmat^{-\top})\).
\end{proposition}

This proposition verifies that the construction is not an ad-hoc interpolation
scheme: with the matching GP operator, \GMVIP reduces exactly to classical
Matheron sampling on sparse variational GPs.

\begin{proposition}[Second-order prior preservation]
\label{prop:second_order_preservation}
For a scalar-valued process, assume \(g\sim p_0\) has finite second moments,
with mean \(\mu\) and covariance
function \(\kappa\). If
\(\Psib_{\Z}(\cdot)=\kappa(\cdot,\Z)\kappa(\Z,\Z)^{-1}\),
\(\vmu=\mu(\Z)\), \(\Dmat\Dmat^\top=\kappa(\Z,\Z)\), and
\(\va\sim\N(\mathbf{0},\Imat_M)\) is independent of \(g\), then the \GMVIP surrogate prior has
the same mean and covariance function as \(p_0\).
\end{proposition}

For a Gaussian prior, Proposition~\ref{prop:gp_svgp_recovery} gives equality of
the full process law. For a general implicit prior, Proposition~\ref{prop:second_order_preservation}
shows what the empirical Matheron operator preserves the first two moments of the original implicit process when the prior is not
Gaussian. Higher-order
differences are controlled by the quality of the replacement operator and the
choice of inducing set; Appendix~\ref{app:theory} makes this dependence explicit. In these propositions, \(\kappa(\X,\Z)\) denotes the corresponding Gram
matrix, whereas \(\kappa(\vx,\vz)\) denotes a scalar kernel value.

\paragraph{Computational complexity}
For scalar outputs, let \(N\) be the number of input locations at which the function is sampled,
\(M=|\Z|\) the number of inducing locations, and \(B\) the number of prior
samples used to estimate the empirical covariance operator. Let
\(C_{\mathrm{bank}}(N+M,B)\) denote the cost of evaluating these \(B\) prior
samples on \(\X\cup\Z\), and let \(C_g(N+M)\) denote the cost of drawing and
evaluating the coherent prior sample \(g(\X),g(\Z)\) used in the Matheron
update. Constructing the empirical operator costs \(
    O\left(
        C_{\mathrm{bank}}(N+M,B)
        + BNM
        + BM^2
        + NM^2
        + M^3
    \right)\). 
The terms \(BNM\) and \(BM^2\) form the empirical cross-covariance and
inducing covariance from the prior bank, \(M^3\) is the inducing Cholesky
factorization, and \(NM^2\) applies the inverse inducing covariance to the
\(N\times M\) cross-covariance matrix. After this setup is amortized, one
posterior sample costs \(
    O\left(C_g(N+M)+M^2+NM\right)\), 
where \(M^2\) accounts for a dense coefficient transformation and \(NM\)
for the pathwise correction. For a joint \(C\)-output model, replace the matrix
dimension \(M\) by \(d_a=MC\) and let \(N\) count scalar output evaluations.

\begin{remark}[Tunable inducing locations and operators]
\label{rem:tunable_z}
All identities above hold conditionally on \(\Z\), \(\Psib_{\Z}\), \(\vmu\), and
\(\Dmat\). If these quantities are fixed or learned from prior samples only,
\GMVIP is variational inference under a fixed surrogate prior. If they are
optimized using the observed targets, the surrogate prior itself becomes
data-adaptive. This can be useful as empirical-Bayes approximation tuning, and is used in our experiments when suitable.
\end{remark}

\section{Experiments}

In this section we showcase the experimental results obtained in a range of regression and forecasting experiments. Refer to Appendix~\ref{app:details} for further experimental details and Appendix~\ref{app:ablation-studies} for ablation studies on inducing locations and prior samples. Additional classification LeNet experiments are included in
Appendix~\ref{app:classification}. 

The codebase for the experiments can be found in \url{https://github.com/Ludvins/ImplicitProcessZoo}.

\subsection{Synthetic function-space diagnostics}
We adapt the one-dimensional synthetic regression benchmark of
\citet{izmailov2020subspace} to the implicit-process setting. The fixed
dataset is
\(\mathcal{D}=\{(x_i,y_i)\}_{i=1}^{400}\) with scalar inputs and outputs.
The implicit process prior is induced by a BNN. Figure~\ref{fig:synthetic} (Appendix~\ref{app:synthetic-valla}) shows the predictive distributions produced by the different methods on the same input grid. The experiment tests whether each approximation can represent plausible uncertainty in regions with sparse observations and extrapolation.

\subsection{Small and medium regression benchmarks}
\input{tables/uci_rmse}

We evaluate regression methods on the standard UCI suite: Boston, Concrete,
Energy, Kin8nm, Naval, Power, Protein, Wine Red, and Yacht; using a two-layer with 10 units BNN as prior. Each task is scalar
regression with Gaussian likelihood, and uses random 90\%/10\% train/test splits. Table~\ref{tab:uci-all-metrics-gmvip-wandb}
reports RMSE and NLL averaged across random seeds (CRPS and CQM~\citep{ortega2024variational} in Appendix~\ref{app:full_uci}). GMVIP is highly
competitive across the experiment: it obtains the best or second-best mean on several
datasets and metrics, including strong RMSE/CRPS results on Boston, Concrete,
Energy, Kin8nm, Naval, and Protein.  For the BNN regression benchmarks, VIP, FTIP, and GMVIP jointly optimize
the BNN prior parameters and therefore implement empirical-Bayes
approximation tuning. Results with fixed prior in Appendix~\ref{app:ablation-studies}.

GMVIP is particularly strong in RMSE, NLL, and CRPS, while centered
interval calibration is mixed and sometimes favors simpler alternatives. This
suggests that the Matheron posterior better captures predictive
uncertainty while retaining the implicit-prior structure. On some datasets,
notably Power and Yacht, simpler or alternative approximations remain competitive,
indicating that no single posterior family dominates uniformly.
Appendix~\ref{app:ablation-studies} shows that the residual path prevents
severe overconfidence on several small-data tasks, although an
inducing-only representation remains competitive on some datasets; it
also quantifies the role of empirical-Bayes prior adaptation.

\subsection{Forecasting With Implicit and Empirical Priors}

We evaluate GMVIP in two complementary forecasting settings with
sample-defined priors: a simulator-induced Lotka--Volterra prior and a
retrieval-conditioned empirical prior for electricity load. Full protocols
are provided in Appendices~\ref{app:lv-implicit-prior}
and~\ref{app:eld-forecasting}. The simulator prior remain fixed, and the
electricity prior is fixed after target-specific retrieval.

\paragraph{Lotka--Volterra.}
The implicit prior is induced by the predator--prey system
\[
\dot u=\alpha u-\beta uv,
\qquad
\dot v=\delta uv-\gamma v,
\]
by sampling physical parameters and initial conditions and solving the ODE.
Each method observes noisy values for \(t\leq15\), receives no data from the
gap \(15<t\leq20\), and is evaluated only on the forecast interval
\(20<t\leq30\). GMVIP uses the joint multioutput construction.

\begin{figure}[t]
    \centering
    \includegraphics[width=0.98\linewidth]{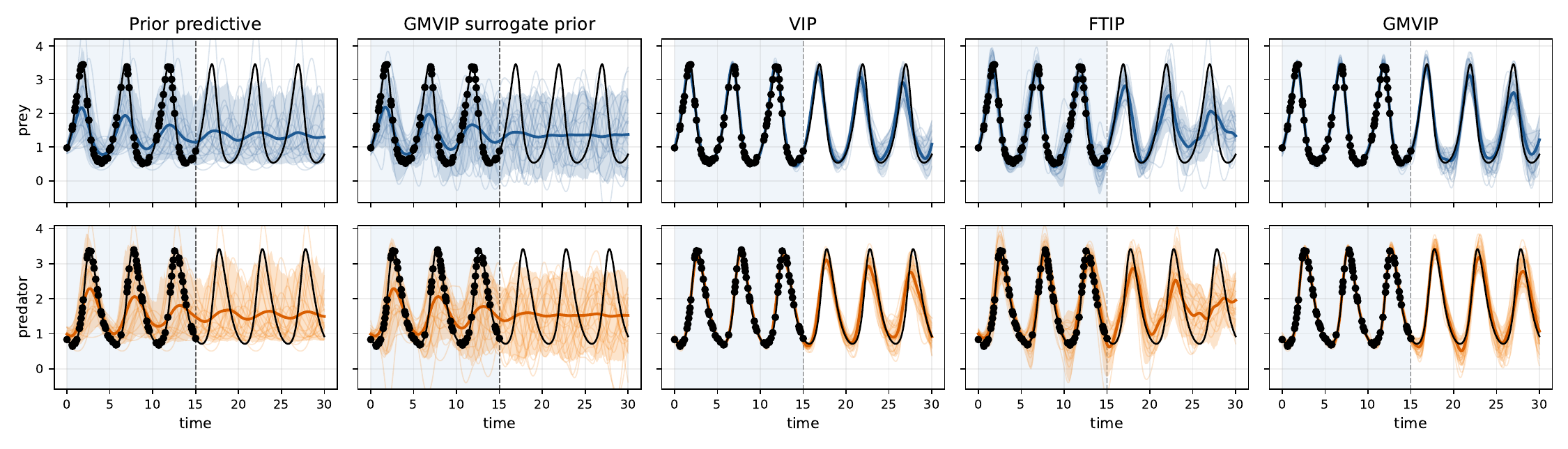}
    \caption{Lotka--Volterra forecasting with an implicit ODE prior.
    Black points are noisy observations. Curves and bands show posterior predictions from
    the final scheduled optimization iterate.}
    \label{fig:volterra}
\end{figure}

\begin{table*}[t]
\centering
\caption{Mean \(\pm\) standard error Lotka--Volterra results on the
held-out test interval \(t\in(20,30]\) across 20 target trajectories.
VIP and FTIP use \(S=20\). VIP is trained for 800 iterations; FTIP uses a
400-iteration VIP warm start followed by 400 fine-tuning iterations. Best results are shown in
\textcolor{purple}{\textbf{purple}} and second-best results in
\textcolor{teal}{\textbf{teal}}. Coverage is ranked by proximity to
the nominal \(0.90\) level.}
\label{tab:lotka_volterra_mean_se_results}
\scalebox{0.9}{
\begin{tabular}{lccccc}
\toprule
Method
& RMSE $\downarrow$
& NLL $\downarrow$
& CRPS $\downarrow$
& Cov$_{90}$
& ODE residual $\downarrow$ \\
\midrule
Prior predictive
& $0.63 \pm 0.06$
& $\second{0.90 \pm 0.11}$
& $0.36 \pm 0.04$
& $\best{0.90 \pm 0.03}$
& $\best{\approx 10^{-6}}$\\

GMVIP sur.\ prior
& $0.63 \pm 0.06$
& $1.95 \pm 0.93$
& $0.37 \pm 0.03$
& $0.93 \pm 0.02$
& $0.31 \pm 0.00$ \\

VIP
& $\second{0.50 \pm 0.08}$
& $11.32 \pm 5.01$
& $\second{0.31 \pm 0.05}$
& $0.72 \pm 0.06$
& $0.65 \pm 0.03$ \\

FTIP
& $0.57 \pm 0.08$
& $1.04 \pm 0.22$
& $0.34 \pm 0.05$
& $0.82 \pm 0.04$
& $0.62 \pm 0.03$ \\

GMVIP
& $\best{0.32 \pm 0.06}$
& $\best{0.51 \pm 0.29}$
& $\best{0.20 \pm 0.04}$
& $\second{0.90 \pm 0.04}$
& $\second{0.31 \pm 0.02}$ \\
\bottomrule
\end{tabular}}
\end{table*}

Figure~\ref{fig:volterra} qualitatively illustrates the differences between
the methods. The prior predictive and its GMVIP surrogate capture the
oscillatory structure of the ODE prior but are not adapted to the observed
trajectory. VIP and FTIP condition on the observations, although their
forecasts exhibit increasing phase and amplitude errors over the held-out
interval. In contrast, GMVIP more consistently preserves the coupled
oscillatory dynamics of both prey and predator over the forecasting horizon.

Quantitatively, Table~\ref{tab:lotka_volterra_mean_se_results} shows that
GMVIP achieves the lowest mean RMSE, NLL, and CRPS at \(S=20\),
while the prior predictive has coverage closest to \(0.90\). GMVIP's mean
RMSE of \(0.32\) represents reductions of \(35.5\%\) relative to VIP and
\(43.6\%\) relative to FTIP. The basis-size analysis in Appendix~\ref{app:lv-basis-sweep} shows that
increasing \(S\) improves VIP and FTIP's point forecasts but does not
overtake GMVIP on the aggregate metrics. GMVIP has the lowest mean RMSE, NLL,
CRPS, and ODE residual among the learned configurations, as well as coverage
closest to \(0.90\).

\paragraph{Electricity load.}
We next consider day-ahead forecasting on the Electricity Load
dataset~\citep{trindade2015electricityloaddiagrams}. Each target contains
48 hours at 15-minute resolution: the first 24 hours are observed and the
next 24 hours are held out. For each target, we retrieve \(2048\)
historical trajectories from 2011--2013 using calendar compatibility and
nearest-neighbor similarity over the observed context; test targets are
drawn from 2014. Because the observed prefix is used both to retrieve the historical bank
and to fit the posterior, this experiment is a two-stage empirical-Bayes
procedure rather than inference under a fixed, target-independent prior.
Retrieval never accesses the held-out forecast day, and every method
receives the same retrieved bank. GMVIP uses \(M=96\) inducing locations,
with \(\Z=\X_{\mathrm{obs}}\): one at each observed context point and none
in the forecast interval.

\begin{figure}[t]
    \centering
    \includegraphics[width=0.98\linewidth]{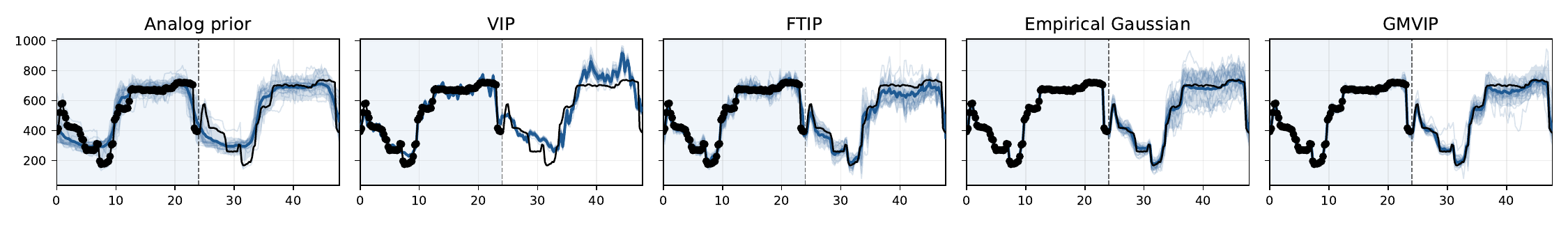}
    \caption{Electricity-load forecasting with a retrieval-conditioned
    empirical prior. Methods observe the first 24 hours and forecast the
    following 24 hours. The panels show VIP, FTIP, and GMVIP on the same
    target window. Black points are observations, the black curve is the
    held-out target, and blue curves and bands show posterior predictions.}
    \label{fig:forecast}
\end{figure}

\begin{table*}[t]
\centering
\caption{Electricity-load results across \(75\) seed--target windows.
Entries are means \(\pm\) standard errors across windows.
Coverage and its standard error are reported in percentage points.
Best results are shown in
\textcolor{purple}{\textbf{purple}} and second-best results in
\textcolor{teal}{\textbf{teal}}. Coverage is ranked by proximity to its
nominal level; lower is better otherwise.}
\label{tab:eld-results}
\small
\resizebox{\textwidth}{!}{
\begin{tabular}{lcccccc}
\toprule
Method
& RMSE \(\downarrow\)
& NLL \(\downarrow\)
& CRPS \(\downarrow\)
& CQM \(\downarrow\)
& Cov.\ 80\%
& Cov.\ 90\% \\
\midrule
Analog prior
& \(55.45 \pm 15.91\)
& \(\second{5.31 \pm 0.44}\)
& \(30.73 \pm 9.04\)
& \(\second{0.12 \pm 0.01}\)
& \(74.18 \pm 2.14\)
& \(85.96 \pm 1.70\) \\

VIP
& \(60.93 \pm 18.61\)
& \(99.19 \pm 55.22\)
& \(40.60 \pm 12.83\)
& \(0.38 \pm 0.01\)
& \(20.94 \pm 1.16\)
& \(26.26 \pm 1.42\) \\

FTIP
& \(68.02 \pm 22.10\)
& \(10.03 \pm 4.17\)
& \(41.28 \pm 13.99\)
& \(0.17 \pm 0.01\)
& \(57.25 \pm 2.03\)
& \(68.06 \pm 1.99\) \\

Empirical Gaussian
& \(\best{48.44 \pm 15.27}\)
& \(6.63 \pm 0.99\)
& \(\second{28.37 \pm 9.06}\)
& \(0.13 \pm 0.01\)
& \(\best{79.32 \pm 2.16}\)
& \(\best{86.62 \pm 1.78}\) \\

GMVIP
& \(\second{49.06 \pm 15.06}\)
& \(\best{4.35 \pm 0.14}\)
& \(\best{27.48 \pm 8.74}\)
& \(\best{0.11 \pm 0.01}\)
& \(\second{76.22 \pm 2.11}\)
& \(\second{86.53 \pm 1.73}\) \\
\bottomrule
\end{tabular}}
\end{table*}

Figure~\ref{fig:forecast} provides a representative qualitative comparison.
VIP and FTIP follow the broad daily pattern but show larger forecast spread
and noticeable errors around the load minimum and subsequent recovery. GMVIP
more closely reproduces the sharp decline after the observed context and the
later return to the higher-load regime, while maintaining coherent
uncertainty across the forecast horizon.

Quantitatively, Table~\ref{tab:eld-results} reports means and sample standard
errors across the \(75\) windows; the large standard errors show
substantial target heterogeneity. The Empirical Gaussian has the lowest mean
RMSE, with GMVIP close behind (\(48.44\) versus \(49.06\)). GMVIP obtains the
lowest mean mixture NLL, CRPS, and CQM. The Empirical Gaussian has coverage
closest to the nominal \(80\%\) and \(90\%\) levels, with GMVIP second on both.
VIP is strongly under-dispersed, whereas FTIP recovers part of the missing
uncertainty. Overall, GMVIP provides the strongest distributional scores,
while analytic Gaussian conditioning has a marginal advantage in mean point
error.

\subsection{Large Scale Regression}

We evaluate on the
YearPredictionMSD~\citep{bertinmahieux2011yearpredictionmsd} and Airline
benchmark derived from the U.S. Bureau of Transportation Statistics on-time
performance data~\citep{bts_airline_ontime} and used in prior large-scale GP
work~\citep{dutordoir2020sparse}.
These tasks were also considered by
\citet{ortega2024variational}, but we use our own shared protocol. All
methods use the same preprocessing and fixed train/test splits, with inputs
standardized using training-set statistics and targets normalized during
training. We train for a matched budget of \(60\,000\) stochastic iterations
using two-hidden-layer MLPs with \(50\) units for Year and \(100\) for
Airlin.

Table~\ref{tab:large-regression-all-metrics} shows that no single method dominates all metrics. On YearPredictionMSD, FBNN gives the best predictive accuracy and distributional scores, with the lowest RMSE, NLL, and CRPS. GMVIP remains competitive, ranking close to the best methods in RMSE and CRPS, but does not improve over FBNN on this dataset. On Airline, the picture changes: GMVIP obtains the best RMSE and ties for the best rounded NLL, while also achieving the second-best CRPS. This indicates that the empirical Matheron correction is particularly effective in the larger Airline setting, where adapting function-space prior samples through learned inducing corrections improves predictive accuracy.

Calibration, as measured by CQM, favors VIP and FTIP on both datasets. These methods obtain the lowest CQM values, while GMVIP is second-best on Airline and tied for second-best on Year. Overall, the results suggest a tradeoff: VIP/FTIP provide the strongest centered interval calibration, FBNN is strongest on Year, and GMVIP gives the best or near-best predictive performance on the larger Airline benchmark while retaining competitive calibration.
\input{tables/regression}

\section{Discussion and Conclusion}
\label{sec:discussion_conclusion}

We introduced GMVIP, a pathwise variational family for implicit-process priors.
Instead of forming posterior samples as linear combinations of prior functions,
GMVIP draws a coherent prior path and replaces its values at inducing locations
through an empirical Matheron-style correction. Since the surrogate prior and
posterior share the same pathwise map and differ only in the law of the inducing
coefficient, the augmented KL reduces to a finite-dimensional coefficient KL.

The construction recovers the standard inducing-variable variational GP family
when the prior is Gaussian and the matching covariance operator is used. For
general implicit priors, empirical GMVIP should be understood as inference under
a Matheron surrogate prior: it preserves the prior mean and covariance in the
population limit, but not necessarily the full process law. Empirically, GMVIP is competitive across regression and
classification. In Lotka-Volterra, GMVIP has the lowest mean RMSE and CRPS
and the lowest mean mixture NLL among the learned methods in the matched
\(S=20\) comparison. It also retains the best mean RMSE, NLL, CRPS,
calibration, and ODE residual when compared with the larger VIP and FTIP
bases. These comparisons vary substantially by target, and mean NLL remains
sensitive to one difficult high-amplitude trajectory. Thus the simulator
results support the value of coherent prior residual paths while showing that
larger coefficient bases narrow the point-forecasting gap. In electricity,
GMVIP has the lowest mean NLL, CRPS, and CQM, while the Empirical Gaussian has
the lowest mean RMSE and coverage closest to nominal.

\paragraph{Limitations} GMVIP performs inference under the surrogate prior induced by the replacement
map, which may differ from the original implicit-process prior for non-Gaussian
or highly structured processes. The empirical operator also depends on the
quality of the prior bank and on the conditioning of the inducing covariance, so
poor inducing locations or too few prior samples can degrade performance. Finally,
empirical GMVIP is more expensive than coefficient-only VIP-style methods because
it requires prior-bank evaluations, covariance estimation, and a Matheron
correction. The present experiments use a Gaussian posterior over inducing
coefficients; richer posterior families may be needed for strongly multimodal
posteriors.

\subsection*{AI use statement}

In this work, we used generative AI tools to provide feedback on research methodology and experimental design and to assist with the interpretation and presentation of experimental results. We did not use generative AI tools to generate datasets, formulate mathematical claims, develop theoretical models, provide critical ingredients for or assist in writing proofs, propose hypotheses, implement the proposed methods, clean or reformat datasets, or perform qualitative or thematic data analysis. Translation and research tasks involving surveys or interviews were not applicable to this work.

Additionally, we used generative AI tools for brainstorming, searching for and summarizing potentially relevant literature, suggesting improvements to the structure and presentation of the paper, and drafting or editing text for clarity and readability. All methodological suggestions and interpretations produced with AI assistance were independently evaluated by the authors. All references were manually verified against their original sources, all experimental results and numerical claims were checked against the outputs of experiments conducted by the authors, and all AI-assisted text was reviewed and revised by the authors. We take responsibility for the final content of this work, including text, claims, and artifacts produced with the aid of generative AI.

\section*{Reproducibility Statement}

We provide the details required to reproduce both the theoretical and empirical
results of this work. Complete proofs of the main theoretical claims, together
with additional properties of the proposed construction, are given in
Appendix~\ref{app:theory}. The \GMVIP training and evaluation procedures are
specified in Algorithms~\ref{alg:gmvip-training}
and~\ref{alg:gmvip-evaluation} in
Appendix~\ref{app:training-evaluation}. Experimental protocols are provided in
Appendix~\ref{app:details}, including the UCI regression experiments
(Appendix~\ref{app:uci-regression}), the Lotka--Volterra implicit-prior
experiment (Appendix~\ref{app:lv-implicit-prior}), the electricity-load
forecasting experiment (Appendix~\ref{app:eld-forecasting}), the large-scale
regression experiments (Appendix~\ref{app:large-regression}), and the LeNet
classification experiments (Appendix~\ref{app:multiclass}). These sections
specify dataset preprocessing and splits, model architectures, optimization
settings, random seeds, Monte Carlo sample counts, inducing-point
configurations, prior-bank sizes, likelihoods, and evaluation metrics.
Additional reproducibility analyses are reported for surrogate-prior fidelity
in Appendix~\ref{app:surrogate_prior}, classification experiments in
Appendix~\ref{app:classification}, full UCI results in
Appendix~\ref{app:full_uci}, ablation studies in
Appendix~\ref{app:ablation-studies}, and Lotka--Volterra basis-size sensitivity
in Appendix~\ref{app:lv-basis-sweep}. All reported predictive results are
evaluated on held-out data according to the protocols specified in these
sections.


\bibliography{references}
\bibliographystyle{plainnat}

\newpage
\appendix

\section{Mathematical Derivations}
\label{app:theory}

This appendix contains the proofs of the main results in
Section~\ref{sec:methodology}, together with additional properties of the
replacement operator used in \GMVIP.

\subsection{Proofs of the main results}

\begin{proof}[Proof of Proposition~\ref{prop:process_validity}]
Let \(\X=(\vx_1,\ldots,\vx_n)\) be any finite ordered tuple. The induced
finite-dimensional random vector is
\[
    f(\X)
    =
    g(\X)+\Psib_{\Z}(\X)\bigl[\vmu+\Dmat\va-g(\Z)\bigr].
\]
It is a measurable function of \((g(\X),g(\Z),\va)\), and therefore defines a
probability law.

It remains to verify projective consistency. Let \(\X_1\) be obtained from
\(\X_2\) by a selection or permutation matrix \(\boldsymbol{\Pi}\). For
multioutput functions, the same symbol denotes the block operator
\(\boldsymbol{\Pi}\otimes\Imat_C\) acting on stacked evaluations. The replacement
operators used in this work are row-consistent, meaning
\(\Psib_{\Z}(\X_1)=\boldsymbol{\Pi}\Psib_{\Z}(\X_2)\). Hence
\[
\begin{aligned}
    f(\X_1)
    &=
    g(\X_1)+\Psib_{\Z}(\X_1)\bigl[\vmu+\Dmat\va-g(\Z)\bigr] \\
    &=
    \boldsymbol{\Pi}g(\X_2)+\boldsymbol{\Pi}\Psib_{\Z}(\X_2)
    \bigl[\vmu+\Dmat\va-g(\Z)\bigr]
    =
    \boldsymbol{\Pi}f(\X_2).
\end{aligned}
\]
Since \(g\sim p_0\) is a valid stochastic process, its finite-dimensional laws
are Kolmogorov-consistent. Because the same global pair \((g,\va)\) is reused
for all query sets, the induced laws satisfy
\(P^{\nu}_{\X_1}=\boldsymbol{\Pi}_{\#}P^{\nu}_{\X_2}\), where
\(\boldsymbol{\Pi}_{\#}\) denotes
pushforward by the restriction/permutation map. Thus the family
\(\{P^\nu_{\X}\}_{\X}\) satisfies Kolmogorov consistency. By the Kolmogorov
extension theorem, it defines a valid stochastic process. Taking
\(\nu=p(\va)\) gives the surrogate prior and taking \(\nu=q_\phi(\va)\) gives
the variational posterior.
\end{proof}

\begin{proof}[Proof of Proposition~\ref{prop:coefficient_kl}]
By construction,
\[
    p_{\Z}(g,\va,f)
    =
    p_0(g)p(\va)\delta[f=T(g,\va)],
    \qquad
    q_{\Z,\phi}(g,\va,f)
    =
    p_0(g)q_\phi(\va)\delta[f=T(g,\va)].
\]
Thus the two measures have the same conditional law of \((g,f)\) given \(\va\).
Applying the chain rule for KL divergences,
\[
\begin{aligned}
    \KL(q_{\Z,\phi}(g,\va,f)\|p_{\Z}(g,\va,f))
    &=
    \KL(q_\phi(\va)\|p(\va))  \\
    &\quad+
    \E_{q_\phi(\va)}
    \KL(q_{\Z,\phi}(g,f\mid\va)\|p_{\Z}(g,f\mid\va)).
\end{aligned}
\]
The second term is zero because both conditionals equal
\(p_0(g)\delta[f=T(g,\va)]\). This proves
\eqref{eq:kl_augmented_bold}.

The marginal function-space bound follows by data processing under the projection
\((g,\va,f)\mapsto f\):
\[
    \KL(q_{\Z,\phi}^{f}\|p_{\Z}^{f})
    \leq
    \KL(q_{\Z,\phi}(g,\va,f)\|p_{\Z}(g,\va,f))
    =
    \KL(q_\phi(\va)\|p(\va)).
\]
Finally, if \(\Psib_{\Z}(\Z)=\Imat_{d_a}\) and \(\Dmat\) is invertible, then evaluating the
replacement map at the inducing inputs gives
\[
    f(\Z)
    =
    g(\Z)+\vmu+\Dmat\va-g(\Z)
    =
    \vmu+\Dmat\va,
\]
so \(\va=\Dmat^{-1}(f(\Z)-\vmu)\) is a measurable function of \(f\). Hence the
density ratio between the augmented posterior and prior depends only on \(f\),
and the projection to \(f\) loses no KL information. Therefore the above
inequality is an equality in this case.
\end{proof}

Proposition~\ref{prop:coefficient_kl} applies exactly to the unregularized cardinal operator. In
the numerical implementation, diagonal jitter makes cardinality and the
corresponding marginal-KL equality approximate up to numerical
regularization.

\begin{proof}[Proof of Proposition~\ref{prop:gp_svgp_recovery}]
Fix a finite input set \(\X\). The GP conditional decomposition gives
\[
    g(\X)
    =
    \mu(\X)
    +
    \kappa(\X,\Z)\kappa(\Z,\Z)^{-1}
    \bigl(g(\Z)-\vmu\bigr)
    +
    \vr_{\X},
\]
where \(\vr_{\X}\) is independent of \(g(\Z)\) and has covariance
\[
    \kappa(\X,\X)
    -
    \kappa(\X,\Z)\kappa(\Z,\Z)^{-1}\kappa(\Z,\X).
\]
With \(\vu=\vmu+\Dmat\va\), the transformed sample is
\[
    T(g,\va)(\X)
    =
    \mu(\X)
    +
    \kappa(\X,\Z)\kappa(\Z,\Z)^{-1}
    \bigl(\vu-\vmu\bigr)
    +
    \vr_{\X}.
\]
If \(\va\sim\N(\mathbf{0},\Imat_M)\), then
\(\vu\sim\N(\vmu,\kappa(\Z,\Z))\), independently of \(\vr_{\X}\). Therefore
\(T(g,\va)(\X)\) has mean \(\mu(\X)\) and covariance \(\kappa(\X,\X)\). Since
this holds for every finite \(\X\), the transformed process is exactly
\(\GP(\mu,\kappa)\).

For arbitrary \(q_\phi(\va)\), the same expression is the GP conditional
\(p_0(f(\X)\mid \vu)\). Integrating over the pushforward law \(q_\phi(\vu)\)
gives \(q(f)=\int p_0(f\mid\vu)q_\phi(\vu)d\vu\). Finally, the correspondence
with a Gaussian SVGP posterior \(q(\vu)=\N(\vm_u,\Smat_u)\) follows from the
affine whitening relation \(\vu=\vmu+\Dmat\va\), which gives
\(q_\phi(\va)=\N(\Dmat^{-1}(\vm_u-\vmu),
\Dmat^{-1}\Smat_u\Dmat^{-\top})\).
\end{proof}

\begin{proof}[Proof of Proposition~\ref{prop:second_order_preservation}]
Fix finite input sets \(\X\) and \(\X'\), and write
\(\Bmat_{\X}=\kappa(\X,\Z)\kappa(\Z,\Z)^{-1}\). The transformed sample can be written as
\[
    f(\X)=g(\X)-\Bmat_{\X}g(\Z)+\Bmat_{\X}\vu,
    \qquad
    \vu=\vmu+\Dmat\va.
\]
Since \(\vmu=\mu(\Z)\), \(\Dmat\Dmat^\top=\kappa(\Z,\Z)\), and
\(\va\sim\N(\mathbf{0},\Imat_M)\), we have
\(\E[\vu]=\mu(\Z)\) and \(\Cov(\vu)=\kappa(\Z,\Z)\). Moreover, \(\vu\) is
independent of \(g\). Hence
\[
    \E[f(\X)]
    =
    \mu(\X)-\Bmat_{\X}\mu(\Z)+\Bmat_{\X}\mu(\Z)
    =
    \mu(\X).
\]
For the covariance, independence gives
\[
\begin{aligned}
    \Cov(f(\X),f(\X'))
    &=
    \Cov(g(\X)-\Bmat_{\X}g(\Z),\,g(\X')-\Bmat_{\X'}g(\Z)) \\
    &\quad+
    \Bmat_{\X}\kappa(\Z,\Z)\Bmat_{\X'}^\top .
\end{aligned}
\]
Expanding the first term and using
\(\Bmat_{\X}=\kappa(\X,\Z)\kappa(\Z,\Z)^{-1}\) yields
\[
    \Cov(g(\X)-\Bmat_{\X}g(\Z),\,g(\X')-\Bmat_{\X'}g(\Z))
    =
    \kappa(\X,\X')-\Bmat_{\X}\kappa(\Z,\X').
\]
The second term satisfies
\(\Bmat_{\X}\kappa(\Z,\Z)\Bmat_{\X'}^\top
=\Bmat_{\X}\kappa(\Z,\X')\). Therefore
\(\Cov(f(\X),f(\X'))=\kappa(\X,\X')\). Since this holds for arbitrary finite
\(\X,\X'\), the surrogate prior has the same mean and covariance function as
\(p_0\).
\end{proof}

\subsection{Additional theoretical properties}

\begin{proposition}[Finite-dimensional prior distortion]
\label{prop:prior_distortion}
Fix a finite input set \(\X\), write \(\Bmat_{\X}=\Psib_{\Z}(\X)\), and define the
prior residual \(\vr_{\X}=g(\X)-\Bmat_{\X}g(\Z)\). Let the surrogate inducing variable
\(\vu\) be independent of \(g\), with law \(P_{\vu}\). Then, whenever the KLs
are well-defined,
\[
    \KL \left(p_0(g(\X))\,\|\,p_{\Z}(f(\X))\right)
    \leq
    I(\vr_{\X};g(\Z))
    +
    \KL \left(p_0(g(\Z))\,\|\,P_{\vu}\right).
\]
In particular, if \(\vu\overset{d}{=}g(\Z)\), the distortion is bounded by
\(I(\vr_{\X};g(\Z))\), and it is zero whenever \(\vr_{\X}\) is independent of
\(g(\Z)\).
\end{proposition}

\begin{proof}
Let \(H(\vr,\vw)=\vr+\Bmat_{\X}\vw\). Under the original prior,
\(g(\X)=H(\vr_{\X},g(\Z))\), where \((\vr_{\X},g(\Z))\) has its joint prior law.
Under the surrogate prior, \(f(\X)=H(\vr_{\X},\vu)\), where the joint law is
\(P_{\vr_{\X}}\otimes P_{\vu}\), because \(\vu\) is independent of \(g\). By data
processing for KL divergences,
\[
    \KL(p_0(g(\X))\|p_{\Z}(f(\X)))
    \leq
    \KL(P_{\vr_{\X},g(\Z)}\|P_{\vr_{\X}}\otimes P_{\vu}).
\]
The right-hand side decomposes as
\[
    \KL(P_{\vr_{\X},g(\Z)}\|P_{\vr_{\X}}\otimes P_{g(\Z)})
    +
    \KL(P_{g(\Z)}\|P_{\vu})
    =
    I(\vr_{\X};g(\Z))
    +
    \KL(p_0(g(\Z))\|P_{\vu}).
\]
If \(P_{\vu}=P_{g(\Z)}\), the second term vanishes. If, in addition,
\(\vr_{\X}\) and \(g(\Z)\) are independent, the mutual information term also
vanishes.
\end{proof}

The finite-dimensional KL distortion bound may be vacuous for singular or
empirical implicit priors; guarantees based on Wasserstein or other integral
probability metrics are an important direction for future work.

\begin{proposition}[Prior reconstruction error]
\label{prop:prior_reconstruction_error}
Assume \(g\sim p_0\) has mean \(\mu\), covariance function \(\kappa\), and
\(\kappa(\Z,\Z)\succ0\). With
\(\Psib_{\Z}(\U)=\kappa(\U,\Z)\kappa(\Z,\Z)^{-1}\), the mean-squared residual
reconstruction error is
\[
    \E \left[
        \left\|
        g(\U)-\mu(\U)
        -
        \Psib_{\Z}(\U)\bigl(g(\Z)-\mu(\Z)\bigr)
        \right\|_2^2
    \right]
    =
    \tr \left(
        \kappa(\U,\U)
        -
        \kappa(\U,\Z)\kappa(\Z,\Z)^{-1}\kappa(\Z,\U)
    \right).
\]
\end{proposition}

\begin{proof}
Let \(\tilde g(\U)=g(\U)-\mu(\U)\), \(\tilde g(\Z)=g(\Z)-\mu(\Z)\), and
\(\Bmat_{\U}=\kappa(\U,\Z)\kappa(\Z,\Z)^{-1}\). The residual is
\(\mathbf{e}_{\U}=\tilde g(\U)-\Bmat_{\U}\tilde g(\Z)\). Since \(\E[\mathbf{e}_{\U}]=\mathbf{0}\),
\(\E\|\mathbf{e}_{\U}\|_2^2=\tr(\Cov(\mathbf{e}_{\U}))\). Expanding the covariance gives
\[
\begin{aligned}
    \Cov(\mathbf{e}_{\U})
    &=
    \kappa(\U,\U)
    -
    \Bmat_{\U}\kappa(\Z,\U)
    -
    \kappa(\U,\Z)\Bmat_{\U}^{\top}
    +
    \Bmat_{\U}\kappa(\Z,\Z)\Bmat_{\U}^{\top}  \\
    &=
    \kappa(\U,\U)
    -
    \kappa(\U,\Z)\kappa(\Z,\Z)^{-1}\kappa(\Z,\U).
\end{aligned}
\]
Taking the trace proves the result.
\end{proof}

\begin{proposition}[Path-regularity inheritance]
\label{prop:path_regular_inheritance}
Assume \(g\) has almost surely continuous sample paths and
\(\vx\mapsto\Psib_{\Z}(\vx)\) is continuous. Then every \GMVIP sample
\[
    f(\vx)=g(\vx)+\Psib_{\Z}(\vx)\bigl[\vmu+\Dmat\va-g(\Z)\bigr]
\]
is almost surely continuous. The same statement holds for differentiability or
higher-order smoothness whenever \(g\) and \(\vx\mapsto\Psib_{\Z}(\vx)\) have the
corresponding regularity.
\end{proposition}

\begin{proof}
For a fixed draw of \(\va\) and \(g(\Z)\), the vector
\(\mathbf{c}=\vmu+\Dmat\va-g(\Z)\) is finite-dimensional and constant in \(\vx\). Hence
\[
    f(\vx)=g(\vx)+\sum_{j=1}^{d_a} \Psib_{\Z,j}(\vx)c_j .
\]
A finite linear combination of continuous functions is continuous, and adding it
to the continuous path \(g\) preserves continuity. Here
\(\Psib_{\Z,j}(\vx)\) is the \(j\)-th column of the operator (a scalar in the
scalar-output case and a \(C\)-vector in the joint case). The same argument
applies to derivatives of any order for which both \(g\) and these columns of
\(\Psib_{\Z}\) possess the required regularity.
\end{proof}

\section{Optimization and Evaluation Procedures}
\label{app:training-evaluation}

Let \(\mathcal G=\{g^{(b)}\}_{b=1}^{B}\) denote the prior trajectory bank.
From this bank, we estimate the empirical mean \(\hat\vmu\) and covariance
\(\hat\K\). Assuming \(\hat\K(\Z,\Z)\succ0\), we set
\[
\Dmat\Dmat^\top
=
\hat\K(\Z,\Z),
\qquad
\Psib_{\Z}(\X)
=
\hat\K(\X,\Z)
\hat\K(\Z,\Z)^{-1}.
\]
The inducing values and posterior path are
\[
\vmu=\hat\vmu(\Z),
\qquad
\vu(\va)=\vmu+\Dmat\va,
\qquad
f(\X)
=
g(\X)+\Psib_{\Z}(\X)\bigl[\vu(\va)-g(\Z)\bigr].
\]
For vector-valued outputs, inducing evaluations and coefficients are
vectorized so that \(\Dmat\) and \(q_\phi(\va)\) retain the joint
cross-output covariance. In the electricity experiment,
\(\Z=\X_{\mathrm{obs}}\) contains the \(96\) observed locations and does
not extend into the held-out interval.

Algorithm~\ref{alg:gmvip-training} optimizes the coefficient posterior and
likelihood parameters using only the observed partition. After training, Algorithm~\ref{alg:gmvip-evaluation} draws coherent
posterior paths and computes metrics exclusively on the held-out interval.

When the prior parameters or inducing locations are optimized, we keep the
bank's base random draws fixed, reevaluate the bank under the current
parameters, and recompute the empirical moments and replacement operator
at each optimization step. Gradients are propagated through these
operations.

\begin{algorithm}[t]
\caption{\GMVIP training loop}
\label{alg:gmvip-training}
\begin{algorithmic}[1]
\Require Training data
\(\D=(\X,\Y)\), prior sampler \(p_0(g)\), operator
\((\Z,\vmu,\Dmat,\Psib_{\Z})\), iterations \(T_{\mathrm{opt}}\), samples
\(R_{\mathrm{tr}}\).
\Ensure Variational posterior \(q_\phi(\va)\) and likelihood parameters
\(\vpsi\)

\State Initialize
\(q_\phi(\va)=\N(\vm_a,\Smat_a)\) and \(\vpsi\)
\For{\(t=1,\ldots,T_{\mathrm{opt}}\)}
    \State Draw a minibatch
    \((\X_{\mathcal B},\Y_{\mathcal B})\)
    \For{\(r=1,\ldots,R_{\mathrm{tr}}\)}
        \State Draw \(\va^{(r)}\sim q_\phi(\va)\)
        \State Draw one coherent path \(g^{(r)}\sim p_0(g)\) at
        \(\X_{\mathcal B}\cup\Z\)
        \State
        \(\vu^{(r)}\gets\vmu+\Dmat\va^{(r)}\)
        \State
        \(\displaystyle
        \vf_{\mathcal B}^{(r)}
        \gets
        g^{(r)}(\X_{\mathcal B})
        +
        \Psib_{\Z}(\X_{\mathcal B})
        \bigl[\vu^{(r)}-g^{(r)}(\Z)\bigr]
        \)
    \EndFor
    \State
    \(\displaystyle
    \widehat{\mathcal L}_t
    \gets
    \frac{N}{|\mathcal B|R_{\mathrm{tr}}}
    \sum_{r=1}^{R_{\mathrm{tr}}}
    \log p_{\vpsi}
        \bigl(\Y_{\mathcal B}\mid \vf_{\mathcal B}^{(r)}\bigr)
    -
    \KL\!\left(
    q_\phi(\va)\,\|\,\N(\mathbf{0},\Imat_{d_a})
    \right)
    \)
    \State Update \((\phi,\vpsi)\) with an Adam step on
    \(-\widehat{\mathcal L}_t\)
\EndFor
\State \Return \(q_\phi(\va)\) and \(\vpsi\)
\end{algorithmic}
\end{algorithm}

\begin{algorithm}[t]
\caption{\GMVIP evaluation loop}
\label{alg:gmvip-evaluation}
\begin{algorithmic}[1]
\Require Fitted \(q_\phi(\va)\), prior sampler \(p_0(g)\), operator
\((\Z,\vmu,\Dmat,\Psib_{\Z})\), test inputs \(\X_*\), targets \(\Y_*\),
likelihood parameters \(\vpsi\), and samples \(R_{\mathrm{ev}}\)
\Ensure Posterior samples and target-level metrics

\State Disable gradient computation
\For{\(r=1,\ldots,R_{\mathrm{ev}}\)}
    \State Draw \(\va^{(r)}\sim q_\phi(\va)\)
    \State Draw one coherent path \(g^{(r)}\sim p_0(g)\) at
    \(\X_*\cup\Z\)
    \State \(\vu^{(r)}\gets\vmu+\Dmat\va^{(r)}\)
    \State
    \(\displaystyle
    \vf_*^{(r)}
    \gets
    g^{(r)}(\X_*)
    +
    \Psib_{\Z}(\X_*)
    \bigl[\vu^{(r)}-g^{(r)}(\Z)\bigr]
    \)
\EndFor
\State Compute
\(\bar{\vf}_*=
R_{\mathrm{ev}}^{-1}\sum_{r=1}^{R_{\mathrm{ev}}} \vf_*^{(r)}\)
and the empirical predictive quantiles
\State Compute RMSE, NLL, CRPS, CQM, and coverage on
\((\X_*,\Y_*)\)
\State Store the posterior samples and target-level metrics
\State \Return
\(\{\vf_*^{(r)}\}_{r=1}^{R_{\mathrm{ev}}}\) and the metrics
\end{algorithmic}
\end{algorithm}

For sample-based \(C\)-output regression with a diagonal Gaussian likelihood,
the electricity and Lotka--Volterra NLL is computed from the
equal-weight Gaussian mixture
\[
-\frac{1}{N_*C}
\sum_{i=1}^{N_*}
\sum_{c=1}^{C}
\log\left[
\frac{1}{R_{\mathrm{ev}}}
\sum_{r=1}^{R_{\mathrm{ev}}}
\N\!\left(
y_{*,ic};
\bigl[\vf_*^{(r)}\bigr]_{ic},
\sigma_{y,c}^2
\right)
\right].
\]
The mixture is evaluated with a stable log-sum-exp calculation. In the
Lotka--Volterra and electricity experiments,
\(R_{\mathrm{ev}}=1024\). For VIP, FTIP, and GMVIP,
\(\sigma_{y,c}\) is learned for each output channel and initialized at the
nominal observation noise. Training-free prior baselines and analytic Gaussian
baselines retain the fixed nominal likelihood.
The UCI and large-regression experiments evaluate each method's predictive
density; for latent-sample methods this has the same mixture form above.
CRPS is computed from the empirical posterior samples, and coverage uses
their central predictive intervals. For the scalar benchmarks, let \(u_i\)
denote the PIT values and define
\(\gamma(\ell)=N_*^{-1}\sum_i
\mathbf{1}\{|2u_i-1|\leq\ell\}\). The regression benchmarks use
\[
\operatorname{CQM}
=
\int_0^1 |\gamma(\ell)-\ell|\,d\ell,
\]
evaluated by trapezoidal quadrature on \(200\) equally spaced levels. The
electricity experiment instead uses its predefined nine-level quadrature,
\[
\operatorname{CQM}
=
\frac{1}{9}
\sum_{\ell\in\{0.1,\ldots,0.9\}}
\left|\hat c_\ell-\ell\right|.
\]
Here \(\hat c_\ell\) is the empirical coverage of the central predictive
interval with nominal level \(\ell\).

Electricity metrics are computed independently for each of the \(75\)
held-out windows and summarized by the mean and sample standard errors. Lotka--Volterra metrics are evaluated only on the
held-out test interval and summarized by the mean and standard error across
the \(20\) target trajectories.

\section{Experimental Details}\label{app:details}

\subsection{Synthetic Regression Experiment}
\label{app:synthetic-valla}

\begin{figure}[t]
    \centering
    \includegraphics[width=0.97\linewidth]{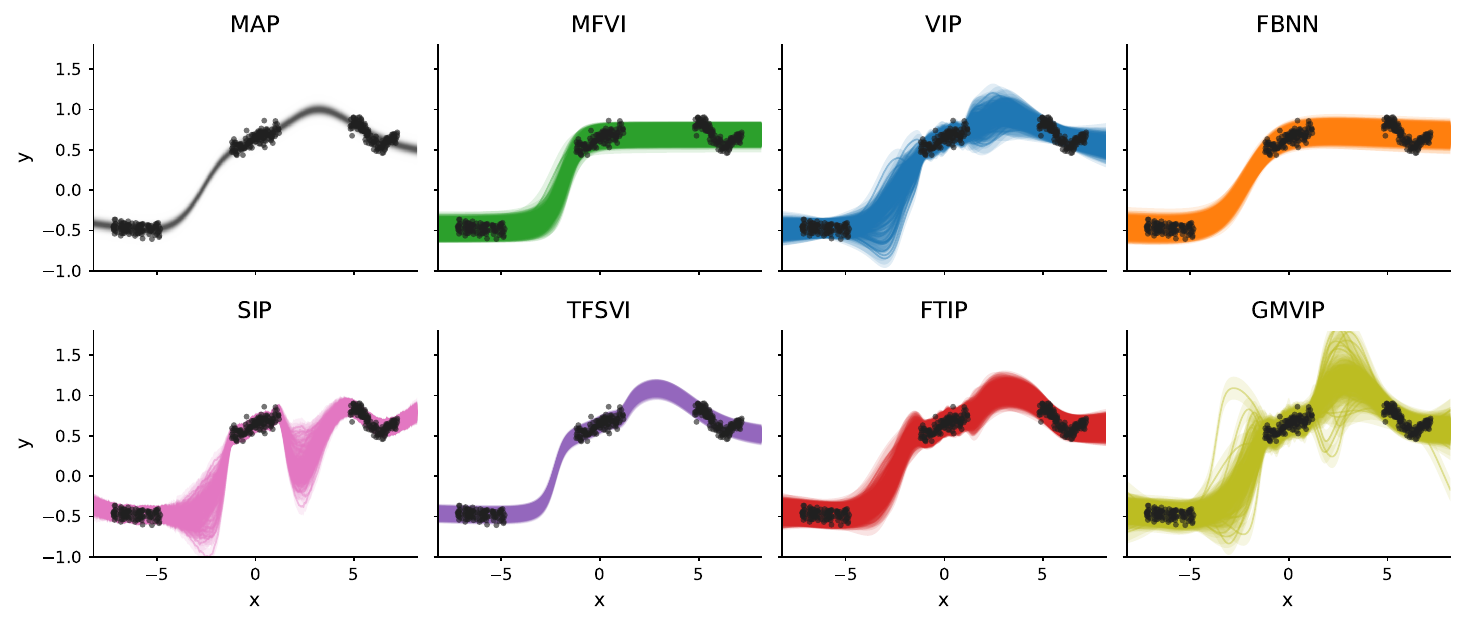}
    \caption{Synthetic one-dimensional regression.
    Black points denote observed data. Each panel shows the predictive behavior
    of one method on the same input grid. MAP gives a deterministic fit, while
    the Bayesian and function-space methods represent posterior uncertainty
    through predictive samples with learned Gaussian noise.}
    \label{fig:synthetic}
\end{figure}

We adapt the synthetic regression dataset introduced by
\citet{izmailov2020subspace} and subsequently used by
\citet{ortega2024variational}. The fixed realization contains \(N=400\)
scalar observations
\[
\mathcal{D}=\{(x_i,y_i)\}_{i=1}^{N},
\qquad x_i,y_i\in\R.
\]
The inputs lie approximately in \([-7.20,7.18]\), and the outputs in
\([-0.61,0.90]\). All observed points are used for training, and posterior
predictions are evaluated on a dense one-dimensional grid.

\paragraph{Implicit Process Prior.}
The implicit process prior is induced by a Bayesian neural network. Sampling
weights and biases \(\vw\sim p(\vw)\) and evaluating the network gives
\[
f_{\vw}(x)=\operatorname{BNN}_{\vw}(x), \qquad f\sim p(f).
\]
The BNN has two hidden layers of width \((10,10)\), \(\tanh\) activations, and a
scalar output. In the frozen-prior setting, Bayesian linear layers use zero-mean
Gaussian weights and biases with unit standard deviation,
\[
\vw_\ell \sim \N\!\left(\mathbf{0},\Imat_{\dim(\vw_\ell)}\right),
\qquad
\vb_\ell \sim \N\!\left(\mathbf{0},\Imat_{\dim(\vb_\ell)}\right).
\]

\paragraph{Learning Problem.}
We use Gaussian regression,
\[
y_i = f(x_i)+\varepsilon_i,
\qquad
\varepsilon_i\sim\N(0,\sigma_y^2),
\]
and infer the posterior predictive distribution
\[
p(f(x_\star)\mid\mathcal{D})
\]
on the plotting grid.

\paragraph{Methods and Visualization.}
Figure~\ref{fig:synthetic} compares MAP, MFVI, VIP, FBNN, SIP, TFSVI,
FTIP, and GMVIP on the same grid. Black points are observations. MAP gives a
deterministic curve, while the other methods show posterior uncertainty through
predictive samples or bands. VIP learns a variational distribution over finite
regression coefficients using samples from the BNN-induced prior; GMVIP uses the
same function prior but constructs a generalized Matheron posterior through
inducing evaluations.

\subsection{UCI Regression Experiments}
\label{app:uci-regression}

We evaluate on nine scalar UCI regression datasets:
\[
\{\text{Boston},\text{Concrete},\text{Energy},\text{Kin8nm},\text{Naval},
\text{Power},\text{Protein},\text{Wine Red},\text{Yacht}\}.
\]
For each seed, we use a random \(90/10\) train/test split. Inputs are
standardized using training-split statistics. Targets are normalized during
training and all reported metrics are transformed back to the original target
scale. Naval RMSE and CRPS are reported in units of \(10^{-4}\).

\paragraph{Sweep design.}
We run five seeds, \(\{0,1,2,3,4\}\), for each dataset and method variant. The reported variants are MAP, MFVI, FBNN, TFSVI, VIP, FTIP, GMVIP, and SIP. We use \(T_{\mathrm{opt}}=30\,000\) iterations for Boston, Concrete,
Energy, and Protein, and \(T_{\mathrm{opt}}=60\,000\) iterations for Kin8nm, Naval, Power, Wine
Red, and Yacht.

\paragraph{Shared training configuration.}
All runs use batch size \(100\) and Adam with initial learning rate \(10^{-3}\).
A cosine schedule anneals the learning rate to \(10^{-5}\) over the training
budget. The shared network architecture is a Bayesian
MLP with two hidden layers of width \((10,10)\), \(\tanh\) activations, scalar
output.

\paragraph{Likelihood and predictive distribution.}
All probabilistic methods use the Gaussian regression likelihood
\[
y_i = f(x_i) + \varepsilon_i,\qquad
\varepsilon_i\sim\N(0,\sigma_y^2),
\]
with a learned observation-noise parameter where applicable. Predictive
distributions are evaluated from the method-specific posterior predictive
samples or Gaussian predictive moments.

\paragraph{Method-specific settings.}
MAP is trained as a deterministic neural-network baseline with an \(L_2\)
penalty. MFVI uses \(R_{\mathrm{tr}}=512\) Monte Carlo samples during training
and \(R_{\mathrm{ev}}=512\) evaluation samples. FBNN uses a BNN functional
prior, freezes the prior parameters, uses \(512\) regression coefficients, and
evaluates with \(R_{\mathrm{ev}}=512\) samples. TFSVI uses
\(R_{\mathrm{tr}}=512\) training samples and \(R_{\mathrm{ev}}=512\)
evaluation samples.

VIP uses \(S=20\) sampled-function basis coefficients. FTIP uses the same
\(S=20\) basis, a \texttt{spline\_1x1} normalizing flow,
\(R_{\mathrm{tr}}=512\) flow samples during training,
\(R_{\mathrm{ev}}=512\) evaluation samples, and no VIP warm start, so the optimizer
budget is not inflated by pretraining.

SIP uses \(M=100\) inducing inputs initialized by \(k\)-means, learns the inducing
locations, and uses \(B=512\) prior samples, \(R_{\mathrm{tr}}=512\) training
samples, and \(R_{\mathrm{ev}}=512\) evaluation samples.
GMVIP uses the empirical Matheron operator, a Gaussian inducing-coefficient
posterior, a tunable BNN prior, \(M=100\) inducing inputs initialized by
\(k\)-means, learned inducing locations, \(B=512\) operator-bank samples,
\(R_{\mathrm{tr}}=512\) training samples, \(R_{\mathrm{ev}}=512\) evaluation
samples, prior-sample mean mode, and prior-Cholesky inducing scale.

\paragraph{Metrics.}
Tables~\ref{tab:uci-all-metrics-gmvip-wandb}and\ref{tab:uci-all-metrics-gmvip-wandb2} reports RMSE, negative log-likelihood (NLL),
continuous ranked probability score (CRPS), and centered quantile metric (CQM).
All metrics are computed on the held-out test split after undoing target
normalization. NLL evaluates the predictive density at the observed targets. CRPS measures the
quality of the full one-dimensional predictive distribution using posterior
predictive samples. CQM is computed from probability integral transform values
and measures centered predictive-interval calibration error. Lower is better for
all reported metrics.

\subsection{Lotka--Volterra Implicit-Prior Experiment}
\label{app:lv-implicit-prior}

We define an implicit process prior through the Lotka--Volterra predator--prey system. For parameters
\[
\veta = (\alpha,\beta,\delta,\gamma,u_0,v_0),
\]
the physical-time state
\[
\vs_{\veta}(t) = \bigl(u_{\veta}(t), v_{\veta}(t)\bigr)
\]
is the solution of
\[
\frac{du}{dt} = \alpha u - \beta u v,
\]
\[
\frac{dv}{dt} = \delta u v - \gamma v,
\]
with initial condition
\[
u(0)=u_0,\qquad v(0)=v_0.
\]
Here \(u(t)\) denotes prey population and \(v(t)\) denotes predator population.
We use the normalized regression coordinate \(x=2t/30-1\), with inverse
\(t(x)=15(x+1)\), and define
\(f_{\veta}(x)=\vs_{\veta}(t(x))\). The implicit prior is the push-forward of
a parameter distribution through the ODE solver:
\[
\veta \sim p(\veta),\qquad
\vs_{\veta}=\operatorname{ODESolve}(\veta),\qquad
f_{\veta}(x)=\vs_{\veta}(t(x)),
\]
and therefore
\[
f_{\veta} \sim p(f).
\]
In the experiment we use
\[
\alpha \sim \operatorname{LogNormal}(\log 1.5, 0.15^2),
\qquad
\beta \sim \operatorname{LogNormal}(\log 1.0, 0.15^2),
\]
\[
\delta \sim \operatorname{LogNormal}(\log 0.75, 0.15^2),
\qquad
\gamma \sim \operatorname{LogNormal}(\log 1.0, 0.15^2),
\]
and
\[
u_0,v_0 \sim \operatorname{Uniform}(0.8,1.2).
\]

The inference methods do not receive an explicit variational distribution over
\(\veta\). The prior is used as a black-box function sampler: one may draw
\[
f^{(s)} \sim p(f)
\]
and evaluate \(f^{(s)}(x)\) at requested inputs, but the simulator parameters are not optimized.

\paragraph{Learning Problem}

Each target trajectory is generated by sampling a held-out parameter vector
\(
\veta^\star \sim p(\veta)
\)
and solving the ODE on
\(
t \in [0,30]
\).
The regression input is the normalized time defined above, so that
\(x \in [-1,1]\). The output is vector-valued:
\(
f(x) = \bigl(u(t), v(t)\bigr) \in \R^2
\).

Training observations are noisy samples from the target trajectory:
\[
\vy_i = f_{\veta^\star}(x_i) + \veps_i,
\qquad
\veps_i \sim \N(\mathbf{0},\boldsymbol{\Sigma}_y),
\]
using \(N_{\mathrm{obs}}=80\) training times restricted to \(t_i\le 15\).
For output channel \(c\), the physical observation-noise standard deviation is
\(0.03\) times the standard deviation of its clean training values; the
likelihood uses this value after target normalization. The interval
\(15<t\leq20\) is deliberately unused: it supplies neither observations,
validation scores, nor checkpoint selection. Test points are clean trajectory
values on \(t\in(20,30]\). Thus the task is sparse forecasting from early
predator--prey observations rather than interpolation over the full
trajectory. The reported aggregate uses 20 deterministic held-out targets.
Outputs are normalized separately for each target using the noisy
training observations. For each trainable method, the normalized Gaussian
observation-noise scale is initialized from the known simulator noise and
learned separately for prey and predator. The training-free prior predictive
and GMVIP surrogate-prior controls use the known simulator noise.
\paragraph{Methods}

This experiment is designed for implicit-process inference. The object we want to use as a prior is the simulator-induced distribution
\(
p(f)
\),
not a neural-network generator with trainable weights. Standard MAP or mean-field variational inference baselines require introducing an explicit parametric generator and fitting a distribution over its parameters. That changes the problem: inference is then performed over the surrogate generator rather than directly with the ODE-induced prior.

VIP, FTIP, and GMVIP are appropriate methods in this setting because they require only samples from
\(
f \sim p(f)
\).
VIP constructs its approximation from a finite collection of sampled prior functions. FTIP starts from a VIP approximation and then applies a learnable transport map in function-sample space. GMVIP also uses simulator samples, but forms a generalized Matheron posterior representation through inducing evaluations.

These are the headline configurations; Appendix~\ref{app:lv-basis-sweep}
reports sensitivity to the sampled-function basis size \(S\) for VIP and
FTIP.

In the reported setting, the standalone VIP baseline uses
\(
S = 20
\)
sampled prior functions and is trained for \(T_{\mathrm{opt}}=800\)
optimization steps. FTIP is warm-started from a VIP model using
\(
S = 20
\)
sampled prior functions, then fine-tuned with a coupling affine flow using
\(R_{\mathrm{tr}}=8\) Monte Carlo samples per step for
\(T_{\mathrm{opt}}=400\) additional optimization steps. Appendix~\ref{app:lv-basis-sweep} reports sensitivity to
\(S\in\{20,64,128,256\}\) under the same forecasting protocol. GMVIP estimates its empirical operator from
\(B=512\) simulator trajectories, uses
\(M=96\) shared inducing times, and is trained for
\(T_{\mathrm{opt}}=800\) optimization steps with
\(R_{\mathrm{tr}}=4\). All trained models are evaluated at the exact final
scheduled iterate: no validation loss, early stopping, or checkpoint
selection is used. All methods use \(R_{\mathrm{ev}}=1024\) posterior samples
for evaluation, including the equal-weight Gaussian-mixture NLL and the
reported trajectory plot.

\paragraph{Joint multioutput parameterization.}
GMVIP uses \(M=96\) shared inducing times, uniformly spaced over
\(t\in[0,30]\). These are input locations only. Prey and predator
inducing values are stacked into a vector of dimension
\(2M=192\). The empirical inducing covariance
\(\hat\K(\Z,\Z)\) is a full \(192\times192\) matrix containing
prey--prey, predator--predator, and prey--predator covariance blocks.
The variational posterior over the corresponding whitened inducing
coefficients also uses a full \(192\times192\) covariance. Thus GMVIP
explicitly models posterior dependence across both time and the two state
variables in this experiment.

\paragraph{ODE residual.}
The ODE residual is evaluated in physical state and time units on the held-out
interval \(t\in(20,30]\). It measures whether each posterior trajectory is
consistent with \emph{some} Lotka--Volterra parameter vector. It does not
use the parameters that generated the target trajectory. For each of up
to \(64\) posterior trajectories
\[
f^{(s)}(t_i)=\bigl(u_i^{(s)},v_i^{(s)}\bigr),
\]
we estimate a separate parameter vector
\[
\widehat{\boldsymbol{\lambda}}^{(s)}
=
\left(
\widehat\alpha^{(s)},
\widehat\beta^{(s)},
\widehat\delta^{(s)},
\widehat\gamma^{(s)}
\right)
\]
by unconstrained least squares. Derivatives are approximated at the interior test-grid points using
centered finite differences:
\[
\dot u_i^{(s)}
=
\frac{u_{i+1}^{(s)}-u_{i-1}^{(s)}}
     {t_{i+1}-t_{i-1}},
\qquad
\dot v_i^{(s)}
=
\frac{v_{i+1}^{(s)}-v_{i-1}^{(s)}}
     {t_{i+1}-t_{i-1}}.
\]
For the regression and right-hand-side evaluation, states are clipped as
\[
u_i^+=\max(u_i^{(s)},10^{-8}),
\qquad
v_i^+=\max(v_i^{(s)},10^{-8}).
\]
The prey and predator parameters are then fitted independently:
\[
(\widehat\alpha^{(s)},\widehat\beta^{(s)})
=
\arg\min_{\alpha,\beta}
\sum_{i=2}^{N_t-1}
\left[
\dot u_i^{(s)}
-\left(\alpha u_i^+-\beta u_i^+v_i^+\right)
\right]^2,
\]
\[
(\widehat\delta^{(s)},\widehat\gamma^{(s)})
=
\arg\min_{\delta,\gamma}
\sum_{i=2}^{N_t-1}
\left[
\dot v_i^{(s)}
-\left(\delta u_i^+v_i^+-\gamma v_i^+\right)
\right]^2.
\]
No positivity constraints are imposed on the fitted parameters. For posterior trajectory \(s\), define
\[
\mathbf{F}_i^{(s)}
=
\begin{bmatrix}
\widehat\alpha^{(s)}u_i^+
-\widehat\beta^{(s)}u_i^+v_i^+\\
\widehat\delta^{(s)}u_i^+v_i^+
-\widehat\gamma^{(s)}v_i^+
\end{bmatrix},
\qquad
\dot{\mathbf{f}}_i^{(s)}
=
\begin{bmatrix}
\dot u_i^{(s)}\\
\dot v_i^{(s)}
\end{bmatrix}.
\]
Its normalized residual is
\[
\mathcal R^{(s)}
=
\frac{
\displaystyle
\frac{1}{2(N_t-2)}
\sum_{i=2}^{N_t-1}
\left\|
\dot{\mathbf{f}}_i^{(s)}-\mathbf{F}_i^{(s)}
\right\|_2^2
}{
\displaystyle
\max\left\{
\frac{1}{2(N_t-2)}
\sum_{i=2}^{N_t-1}
\left\|\dot{\mathbf{f}}_i^{(s)}\right\|_2^2,
10^{-12}
\right\}
}.
\]
The numerator and denominator therefore average over both prey and
predator and over all interior test-grid points. The normalization makes
the reported residual dimensionless. Finally, the per-target ODE
residual is
\[
E_{\mathrm{ODE}}
=
\frac{1}{R_{\mathrm{ev}}'}
\sum_{s=1}^{R_{\mathrm{ev}}'}\mathcal R^{(s)},
\qquad
R_{\mathrm{ev}}'=\min(R_{\mathrm{ev}},64).
\]
Although the input trajectories and derivatives are evaluated in
physical units, only \(E_{\mathrm{ODE}}\), the normalized residual
ratio, is reported.

\subsection{Electricity-Load Empirical-Prior Forecasting Experiment}
\label{app:eld-forecasting}

We define a real-data empirical process prior using the
Electricity Load  dataset~\citep{trindade2015electricityloaddiagrams}.
The dataset contains electricity-consumption trajectories sampled every
\(15\) minutes. Let
\[
y_{c,\tau}(t_j), \qquad j=0,\ldots,191,
\]
denote the \(48\)-hour load window for customer \(c\), starting at calendar time
\(\tau\), with
\[
t_j = 0.25j \quad \text{hours}.
\]
For each target window, the first \(24\) hours are observed and the next
\(24\) hours are held out for forecasting. Thus each window has
\[
N_t=192,\qquad N_{\mathrm{obs}}=96,\qquad N_{\mathrm{test}}=96
\]
quarter-hourly points.

Unlike the ODE experiments, the prior is not generated by a known simulator.
Instead, the prior is an empirical distribution over historical trajectories.
For a target window \(y^\star\), we construct a prior bank
\[
\mathcal{G}(y^\star)
=
\{f^{(1)},\ldots,f^{(B)}\},
\qquad B=2048,
\]
from historical windows in the years \(2011\)--\(2013\). Test targets are sampled
from \(2014\), so the prior bank and test targets are separated by calendar year.

All windows are normalized using statistics from their observed prefix. For a
window \(y\), define
\[
\mu_y = \frac{1}{N_{\mathrm{obs}}}\sum_{j=0}^{N_{\mathrm{obs}}-1} y(t_j),
\qquad
s_y =
\max\left\{
\operatorname{std}\bigl(y(t_0),\ldots,y(t_{N_{\mathrm{obs}}-1})\bigr),
10^{-3}
\right\}.
\]
The normalized trajectory is
\[
\tilde y(t_j)=\frac{y(t_j)-\mu_y}{s_y}.
\]
Each historical prior trajectory is normalized by its own observed prefix
statistics before being added to the bank. The bank is selected by a calendar-and-prefix nearest-neighbor rule. Candidate
historical windows must come from prior years, match the target's
weekday/weekend status, and have month in
\[
\{m^\star-1,m^\star,m^\star+1\},
\]
with months interpreted cyclically. Among these candidates, we compute the
normalized prefix distance
\[
d(y,y^\star)
=
\frac{1}{N_{\mathrm{obs}}}
\sum_{j=0}^{N_{\mathrm{obs}}-1}
\left(\tilde y(t_j)-\tilde y^\star(t_j)\right)^2,
\]
and select the \(B=2048\) nearest historical windows. The empirical prior is
therefore
\[
f \sim p_{\mathrm{emp}}(f\mid y^\star),
\qquad
p_{\mathrm{emp}}(f\mid y^\star)
=
\frac{1}{B}\sum_{b=1}^B \delta_{f^{(b)}}(f).
\]
This is an implicit prior: the methods can draw trajectories from the historical
bank and evaluate them at requested time points, but no tractable density over
functions or latent calendar factors is provided.

\paragraph{Learning Problem}

For each target, the regression input is normalized time
\[
x_j = 2t_j/48 - 1,
\qquad x_j\in[-1,1].
\]
The output is scalar normalized load,
\[
f(x_j)=\tilde y^\star(t_j)\in\R.
\]
Training observations are the first \(24\) hours of the target window:
\[
\mathcal{D}_{\mathrm{obs}}
=
\{(x_j,\tilde y^\star(t_j)):j=0,\ldots,95\}.
\]
The held-out forecasting grid is
\[
t_j\in[24,48),\qquad j=96,\ldots,191.
\]
Metrics are computed after transforming predictions back to the original load
scale using the target-window prefix statistics \((\mu_{y^\star},s_{y^\star})\).

We use a Gaussian observation model in normalized units,
\[
y_i = f(x_i)+\varepsilon_i,
\qquad
\varepsilon_i\sim\N(0,\sigma_y^2),
\]
with
\[
\sigma_y^{(0)} = 0.05.
\]
For VIP, FTIP, and GMVIP, the scalar normalized observation-noise scale is
initialized at \(0.05\) and optimized jointly with the variational
parameters. The Analog prior and Empirical Gaussian baselines retain the
fixed nominal value \(0.05\). In physical units, a normalized scale
\(\sigma_y\) corresponds to \(\sigma_y s_{y^\star}\) for each target window.

We run three target-selection seeds,
\(
\{0,1,2\}
\),
with \(25\) held-out target windows per seed, for \(75\) total target windows.
All targets are sampled from \(2014\), while all empirical-prior trajectories are
drawn from \(2011\)--\(2013\).

\paragraph{Methods}

This experiment evaluates whether the same function-space inference mechanisms
used for simulator priors also apply to a real empirical trajectory prior. The
prior is the historical empirical process \(p_{\mathrm{emp}}(f\mid y^\star)\);
the methods are not given a parametric model for electricity demand and do not
optimize customer-specific latent calendar variables.

The analog baseline directly uses samples from the empirical prior bank without
posterior adaptation. VIP constructs a coefficient posterior from
\(S=20\) sampled prior functions. FTIP uses the same \(S=20\) basis and
adds a learned transport map in coefficient space. GMVIP uses the empirical
Matheron operator with a Gaussian inducing-coefficient posterior.

In the reported setting, VIP and FTIP use \(S=20\) sampled-function basis
coefficients. GMVIP uses \(B=2048\) operator-bank trajectories and
\(M=96\) inducing locations, with \(\Z=\X_{\mathrm{obs}}\) corresponding to
the observed \(24\)-hour input grid. All learned methods are trained for
\(T_{\mathrm{opt}}=500\) optimization steps, with \(R_{\mathrm{tr}}=8\)
Monte Carlo samples per training step. Observation noise is learned in every
trainable method. Runs are evaluated at the final scheduled iterate, with no
validation loss, early stopping, or checkpoint selection. All methods use
\(R_{\mathrm{ev}}=1024\) posterior samples for final metrics and plots.

\paragraph{Empirical Gaussian baseline.}
For each target window, the Empirical Gaussian baseline uses the same
\(B=2048\) historical trajectories retrieved for GMVIP. Each historical trajectory is
standardized using the mean and standard deviation of its own 96-point
observed prefix. The target trajectory is analogously standardized using
its observed prefix. All conditioning is performed in these normalized
coordinates.

Let \(\vh^{(b)}\in\R^{192}\), \(b=1,\ldots,B\), denote the retrieved
historical trajectories. We estimate the full-grid mean and unbiased
empirical covariance as
\[
\hat\vmu
=
\frac{1}{B}\sum_{b=1}^{B}\vh^{(b)},
\qquad
\hat\K
=
\frac{1}{B-1}
\sum_{b=1}^{B}
\left(\vh^{(b)}-\hat\vmu\right)
\left(\vh^{(b)}-\hat\vmu\right)^\top.
\]
Thus, \(\hat\K\in\R^{192\times192}\) contains correlations
between all observed and forecast time points.

Partition the grid into the observed prefix \(o=\{1,\ldots,96\}\) and
the held-out forecast interval \(*=\{97,\ldots,192\}\). Under the
Gaussian observation model
\[
\vy_o=\vf_o+\veps,
\qquad
\veps\sim\N(\mathbf{0},\sigma_y^2\Imat_{96}),
\]
the conditional latent distribution is
\[
p(\vf_*\mid\vy_o)
=
\N(\vm_*,\Cmat_*),
\]
where
\[
\vm_*
=
\hat\vmu_*
+
\hat\K_{*o}
\left(
\hat\K_{oo}
+
\sigma_y^2\Imat_{96}
\right)^{-1}
\left(\vy_o-\hat\vmu_o\right),
\]
and
\[
\Cmat_*
=
\hat\K_{**}
-
\hat\K_{*o}
\left(
\hat\K_{oo}
+
\sigma_y^2\Imat_{96}
\right)^{-1}
\hat\K_{o*}.
\]
Here \(\Cmat_*\) is the conditional covariance of the latent forecast. The baseline is analytic and uses no optimization
steps.

\paragraph{Metrics}

We report RMSE, negative log-likelihood (NLL), continuous ranked probability
score (CRPS), and centered quantile metric (CQM) on the held-out forecasting
region \(t\in[24,48)\). All metrics are computed in the original load units after
undoing the target normalization. The main paper reports the mean and sample
standard errors across the \(75\) seed--target
windows.

\paragraph{Results Summary}

As shown in Table~\ref{tab:eld-results-median}, we report medians because
RMSE exhibits substantial variation across forecasting windows. Electricity
clients have markedly different load scales, and a few high-load or unusually
difficult windows produce extreme absolute errors that disproportionately
influence the mean. The median therefore better represents performance on a
typical task, while the bootstrap confidence intervals quantify uncertainty
across windows. GMVIP provides the strongest typical point and distributional
forecasting performance, although its confidence intervals overlap those of
the strongest baselines. Calibration remains dependent on the metric and
nominal coverage level.

\begin{table*}[t]
\centering
\caption{Electricity-load results across \(75\) seed--target windows.
Entries are medians with \(95\%\) percentile-bootstrap confidence
intervals from \(100{,}000\) resamples. Coverage is reported in
percentage points and ranked by proximity to its nominal level.
Lower is better otherwise.}
\label{tab:eld-results-median}
\small
\resizebox{\textwidth}{!}{
\begin{tabular}{lcccccc}
\toprule
Method
& RMSE \(\downarrow\)
& NLL \(\downarrow\)
& CRPS \(\downarrow\)
& CQM \(\downarrow\)
& Cov.\ 80\%
& Cov.\ 90\% \\
\midrule
Analog prior
& \(17.46\ [12.80,25.41]\)
& \(\second{4.30\ [4.06,4.70]}\)
& \(8.92\ [6.49,13.78]\)
& \(\best{0.093\ [0.077,0.112]}\)
& \(\second{78.13\ [73.96,84.38]}\)
& \(\best{91.67\ [88.54,92.71]}\) \\

VIP
& \(22.94\ [14.13,27.18]\)
& \(10.05\ [8.60,11.64]\)
& \(15.41\ [9.70,18.26]\)
& \(0.384\ [0.365,0.396]\)
& \(20.83\ [17.71,22.92]\)
& \(25.00\ [21.88,29.17]\) \\

FTIP
& \(21.20\ [15.75,28.02]\)
& \(4.88\ [4.27,5.37]\)
& \(12.35\ [8.81,17.57]\)
& \(0.153\ [0.132,0.174]\)
& \(57.29\ [54.17,61.46]\)
& \(68.75\ [66.67,73.96]\) \\

Empirical Gaussian
& \(\second{16.99\ [11.14,25.10]}\)
& \(4.36\ [4.02,5.06]\)
& \(\second{8.39\ [6.06,14.05]}\)
& \(0.108\ [0.070,0.151]\)
& \(84.38\ [80.21,87.50]\)
& \(\best{91.67\ [89.58,93.75]}\) \\

GMVIP
& \(\best{16.20\ [11.14,24.85]}\)
& \(\best{4.08\ [3.79,4.54]}\)
& \(\best{8.14\ [5.93,13.52]}\)
& \(\second{0.102\ [0.072,0.112]}\)
& \(\best{80.21\ [77.08,84.38]}\)
& \(\second{92.71\ [89.58,93.75]}\) \\
\bottomrule
\end{tabular}}
\end{table*}

\subsection{Large Regression Experiments}
\label{app:large-regression}

We evaluate on two large scalar regression datasets:
\[
\{\text{Year},\text{Airline}\}.
\]
The Year dataset uses the standard fixed split with the first \(463\, 715\)
examples for training and the remaining examples for testing. The Airline
dataset uses the first \(800\, 000\) examples, with the first \(700\, 000\) for
training and the remaining \(100\, 000\) for testing. Inputs are standardized
using training-split statistics. Targets are normalized during training and all
reported metrics are transformed back to the original target scale.

\paragraph{Sweep design.}
We run five seeds, \(\{0,1,2,3,4\}\), for each dataset and method variant. The
reported variants are MFVI, FBNN, TFSVI, VIP, FTIP, GMVIP, and SIP. MAP and the
fixed-prior variants of VIP, FTIP, GMVIP, and SIP are omitted from
Table~\ref{tab:large-regression-all-metrics}. We use
\(T_{\mathrm{opt}}=60\,000\) optimization iterations for both Year and Airline.

\paragraph{Shared training configuration.}
All runs use batch size \(100\) and Adam with initial learning rate \(10^{-3}\).
A cosine schedule anneals the learning rate to \(10^{-5}\) over the training
budget. The shared model class is a Bayesian MLP with \(\tanh\)
activations, scalar output. Year uses two
hidden layers of width \((50,50)\), while Airline uses two hidden layers of
width \((100,100)\).

\paragraph{Likelihood and predictive distribution.}
All probabilistic methods use the Gaussian regression likelihood
\[
y_i = f(x_i) + \varepsilon_i,\qquad
\varepsilon_i\sim\N(0,\sigma_y^2),
\]
with a learned observation-noise parameter where applicable. Predictive
distributions are evaluated from the method-specific posterior predictive
samples or Gaussian predictive moments.

\paragraph{Method-specific settings.}
MFVI uses \(R_{\mathrm{tr}}=512\) Monte Carlo samples during training and
\(R_{\mathrm{ev}}=512\) evaluation samples. FBNN uses a BNN functional prior,
freezes the prior parameters and
evaluates with \(R_{\mathrm{ev}}=512\) samples. TFSVI uses
\(R_{\mathrm{tr}}=512\) training samples and \(R_{\mathrm{ev}}=512\)
evaluation samples.

VIP uses \(S=20\) sampled-function basis coefficients and learns the prior
parameters. FTIP uses the same \(S=20\) basis, a
\texttt{spline\_1x1} normalizing flow, \(R_{\mathrm{tr}}=512\) flow samples
during training, \(R_{\mathrm{ev}}=512\) evaluation samples.

SIP uses \(M=100\) inducing inputs initialized by \(k\)-means, learns the inducing
locations, and uses \(B=512\) prior samples, \(R_{\mathrm{tr}}=512\) training
samples, and \(R_{\mathrm{ev}}=512\) evaluation samples. GMVIP uses the empirical
Matheron operator, a Gaussian inducing-coefficient posterior, a tunable BNN
prior, \(M=100\) inducing inputs initialized by \(k\)-means, learned inducing
locations, \(B=512\) operator-bank samples, \(R_{\mathrm{tr}}=512\) training
samples, \(R_{\mathrm{ev}}=512\) evaluation samples, prior-sample mean mode,
and prior-Cholesky inducing scale.

\paragraph{Metrics.}
Table~\ref{tab:large-regression-all-metrics} reports RMSE, negative
log-likelihood (NLL), continuous ranked probability score (CRPS), and centered
quantile metric (CQM). All metrics are computed on the held-out test split after
undoing target normalization. NLL evaluates the predictive density at the
observed targets. CRPS measures the quality of the full one-dimensional
predictive distribution using posterior predictive samples. CQM is computed from
probability integral transform values and measures centered predictive-interval
calibration error. Lower is better for all reported metrics.

\subsection{LeNet Classification Experiments}
\label{app:multiclass}

We evaluate image classification on two standard multiclass datasets,
\[
\{\text{FashionMNIST},\text{CIFAR-10}\}.
\]
For both datasets we use the standard train/test split provided by the dataset:
FashionMNIST has \(60\,000\) training examples and \(10\,000\) test examples,
and CIFAR-10 has \(50\,000\) training examples and \(10\,000\) test examples.
Images are flattened by the dataset wrapper and rescaled to \([0,1]\); the
LeNet model reshapes them back to image tensors internally. No input
standardization or target normalization is applied. Targets are integer class
labels and all methods use the same ten-class classification task.

\paragraph{Sweep design.}
We run five seeds, \(\{0,1,2,3,4\}\), for each dataset and method variant.
The reported variants are MAP, MFVI, FBNN, VIP, FTIP, GMVIP, and SIP. The main
classification table reports the LeNet setting, where the
convolutional feature extractor is deterministic and the classifier head is
Bayesian.

\paragraph{Shared training configuration.}
All runs use \(T_{\mathrm{opt}}=10\,000\) optimizer iterations, batch size \(128\), Adam with
learning rate \(10^{-3}\), no cosine annealing, and evaluation batch size
\(512\). The LeNet backbone consists of two convolutional blocks followed by a
classifier head:
\begin{align}
&\text{Conv}(C,6,5) \rightarrow \text{ReLU} \rightarrow \text{AvgPool}
\rightarrow
\text{Conv}(6,16,5) \rightarrow \text{ReLU} \rightarrow \text{AvgPool}
\rightarrow
\\
&\rightarrow \text{FC}(120) \rightarrow \text{ReLU}
\rightarrow
\text{FC}(84) \rightarrow \text{ReLU}
\rightarrow
\text{FC}(10).
\end{align}
In the Bayesian setting, the convolutional layers are deterministic
and the three fully connected classifier layers are Bayesian
layers for probabilistic methods. Bayesian weight log-standard deviations are
initialized at \(-3.0\), and the prior weight log-standard deviation is also
initialized at \(-3.0\).

\paragraph{Likelihood and predictive distribution.}
All probabilistic methods use a categorical multiclass likelihood. Predictive
class probabilities are obtained by averaging softmax probabilities over
posterior predictive samples. MAP is evaluated from a single deterministic
network. The remaining methods use their method-specific posterior predictive
samplers with \(R_{\mathrm{ev}}=512\) evaluation samples unless otherwise
stated.

\paragraph{Method-specific settings.}
MAP is trained as a deterministic LeNet baseline with cross-entropy and an
\(L_2\) penalty of \(10^{-4}\). MFVI uses \(R_{\mathrm{tr}}=512\) Monte Carlo
samples during training and \(R_{\mathrm{ev}}=512\) samples at evaluation.
FBNN uses a frozen BNN functional prior, \(R_{\mathrm{ev}}=512\) posterior
samples, and \(B=512\) prior samples.

VIP uses \(S=20\) sampled-function basis coefficients and learns the Bayesian
classifier-head prior parameters. FTIP uses the same \(S=20\) basis, a
\texttt{spline\_1x1} normalizing flow, \(R_{\mathrm{tr}}=512\) flow samples
during training, \(R_{\mathrm{ev}}=512\) evaluation samples, and learns the
prior parameters.

SIP uses \(M=100\) inducing inputs initialized by \(k\)-means, learns the inducing
locations, and uses \(B=512\) prior samples, \(R_{\mathrm{tr}}=512\) training
samples, and \(R_{\mathrm{ev}}=512\) evaluation samples. GMVIP uses the empirical
Matheron operator, a Gaussian inducing-coefficient posterior, \(M=100\) inducing
inputs initialized by \(k\)-means, learned inducing locations, \(B=512\)
operator-bank samples, \(R_{\mathrm{tr}}=512\) training samples, and
\(R_{\mathrm{ev}}=512\) evaluation samples.

For classification, GMVIP uses a channel-wise multioutput approximation.
Each of the ten latent logits has its own inducing-coefficient block, and
the variational covariance is block-diagonal across logits. With
\(M=100\), a joint formulation would require a full
\(1000\times1000\) coefficient covariance, so the block-diagonal
parameterization is used for scalability. The softmax likelihood jointly
normalizes the sampled logits, but cross-logit covariance is not explicitly
parameterized in \(q_\phi(\va)\).

\paragraph{Metrics.}
Table~\ref{tab:lenet-last-layer-classification-all-metrics} reports negative
log-likelihood (NLL), classification error, expected calibration error (ECE),
and Brier score on the held-out test split. NLL evaluates the predictive
categorical density assigned to the observed class. Error is the fraction of
incorrect predictions under the posterior mean class probabilities. ECE measures
confidence calibration using binned predicted probabilities. The Brier score
measures the squared error between predictive class probabilities and one-hot
labels. Lower is better for all reported metrics.

\section{Further Experiments}

\subsection{Finite-Dimensional Surrogate-Prior Fidelity}
\label{app:surrogate_prior}
\label{app:prior-fidelity}

The variational objective of \GMVIP is defined relative to its Matheron
surrogate prior rather than directly relative to the original implicit process.
Proposition~\ref{prop:prior_distortion} characterizes two sources of
finite-dimensional distortion: replacing the inducing-value law and removing
the dependence between the residual path and the replacement inducing values.
Proposition~\ref{prop:second_order_preservation} shows that the population
Matheron construction nevertheless preserves the original mean and covariance.
This section measures the resulting finite-sample discrepancy directly and
compares it with the coefficient-space surrogate used by VIP.

\paragraph{Label-free prior and surrogate constructions.}
We use the frozen scalar BNN prior from the synthetic experiment, but do not
load observations, evaluate a likelihood, or perform optimization. The
architecture is \(1\)-\(10\)-\(10\)-\(1\), with \(\tanh\) hidden activations
and independent zero-mean, unit-variance Gaussian weights and biases. We
evaluate every sampled function on \(N_x=301\) uniformly spaced inputs over
\([-5,5]\).

For a VIP basis of \(S\) independently sampled prior functions
\(\mathcal G_{\mathrm{VIP}}=\{g^{(s)}\}_{s=1}^S\), let
\[
    \hat\mu_S(\vx)
    =
    \frac{1}{S}\sum_{s=1}^{S}g^{(s)}(\vx),
    \qquad
    \phi_s(\vx)
    =
    \frac{g^{(s)}(\vx)-\hat\mu_S(\vx)}{\sqrt{S-1}}.
\]
The VIP surrogate-prior sample is
\[
    f_{\mathrm{VIP}}(\vx)
    =
    \hat\mu_S(\vx)+\sum_{s=1}^{S}\phi_s(\vx)a_s,
    \qquad
    \va\sim\N(\mathbf{0},\Imat_S).
\]
Conditional on the sampled basis, this is a Gaussian finite-rank process with
the empirical basis mean and covariance.

The empirical \GMVIP surrogate uses \(M\) uniformly spaced inducing inputs,
the paper's prior-bank notation
\(\mathcal G=\{g^{(b)}\}_{b=1}^B\), and
\[
\begin{aligned}
    f_{\mathrm{GMVIP}}(\cdot)
    &=T(g,\va)(\cdot)\\
    &=
    g(\cdot)
    +
    \Psib^{\mathrm{emp}}_{\Z}(\cdot)
    \left[
        \hat\vmu(\Z)
        +
        \Lmat\va
        -
        g(\Z)
    \right],
    \\
    &\hspace{8em}
    \va\sim\N(\mathbf{0},\Imat_M).
\end{aligned}
\]
Here \(g\) is a fresh coherent BNN-prior path,
\(\Lmat\Lmat^{\top}=\hat\K(\Z,\Z)\) is the stabilized empirical inducing
covariance, and
\(\Psib^{\mathrm{emp}}_{\Z}(\cdot)
=\hat\K(\cdot,\Z)\hat\K(\Z,\Z)^{-1}\), as in
Equation~\ref{eq:empirical_operator_bold}. We use the same numerical settings
as the synthetic \GMVIP experiment.
Both surrogates therefore use \(p(\va)=\N(\mathbf{0},\Imat_{d_a})\), with
\(d_a=S\) for VIP and \(d_a=M\) for scalar-output \GMVIP. No variational
parameters are optimized in this experiment.

\paragraph{Sampling protocol and finite-sample floor.}
We run seeds \(0\)--\(4\). For every seed, disjoint random streams generate
an operator or basis bank, surrogate coefficients and residual paths, and
three independent sets of \(R_{\mathrm{ev}}=2\, 048\) original-prior
functions. The first
original-prior set estimates a pointwise mean \(\widehat\mu_j\) and standard
deviation \(\widehat\sigma_j\). All functions are standardized as
\[
    \widetilde f^{(r)}_j
    =
    \frac{f^{(r)}(x_j)-\widehat\mu_j}{\widehat\sigma_j}.
\]
This common coordinatewise affine map makes the distances dimensionless and
prevents high-variance input regions from dominating the comparison. Its
parameters are estimated only from the independent original-prior calibration
set; no surrogate sample influences the normalization.
The second original-prior set is the reference distribution. The third is an
independent \emph{true-prior split}, and its distance from the reference
quantifies the nonzero floor caused by finite-sample estimation. Operator and
basis samples are never reused in this reference comparison.

\paragraph{Metrics.}
Let
\(\widetilde{\mathbf F},\widetilde{\mathbf H}
\in\R^{R_{\mathrm{ev}}\times N_x}\)
denote standardized reference and candidate samples. We report:
\begin{enumerate}[leftmargin=*,itemsep=2pt,topsep=2pt]
    \item the average pointwise empirical Wasserstein distance
    \[
        \overline W_1
        =
        \frac{1}{N_x}
        \sum_{j=1}^{N_x}
        \frac{1}{R_{\mathrm{ev}}}
        \sum_{r=1}^{R_{\mathrm{ev}}}
        \left|
            \widetilde F_{(r)j}-\widetilde H_{(r)j}
        \right|,
    \]
    together with its maximum over \(x_j\);
    \item the joint sliced Wasserstein distance
    \[
        \operatorname{SW}_2
        =
        \left[
        \frac{1}{P}
        \sum_{p=1}^{P}
        W_2^2
        \left(
            \widetilde{\mathbf F}\vtheta_p,
            \widetilde{\mathbf H}\vtheta_p
        \right)
        \right]^{1/2},
        \qquad P=512,
    \]
    where the \(\vtheta_p\) are random unit directions over the
    complete discretized function;
    \item the RMSE between sample means,
    \(e_\mu=N_x^{-1/2}\|\widehat\vmu_F-\widehat\vmu_H\|_2\), and relative
    covariance error
    \(e_{\K}=\|\hat\K_F-\hat\K_H\|_F/\|\hat\K_F\|_F\);
    \item the energy distance and a three-bandwidth RBF
    \(\operatorname{MMD}_{\mathrm{RBF}}^2\), computed on the first \(512\)
    evaluation samples to bound their quadratic cost.
\end{enumerate}
All metrics are lower-is-closer. Tables report means and standard errors across
the five experiment seeds.
Figure~\ref{fig:prior-fidelity-samples} visualizes representative paths at the
published settings, and Table~\ref{tab:prior-fidelity-defaults} reports all
default-setting discrepancies.

\begin{figure*}[t]
    \centering
    \includegraphics[width=0.98\textwidth]{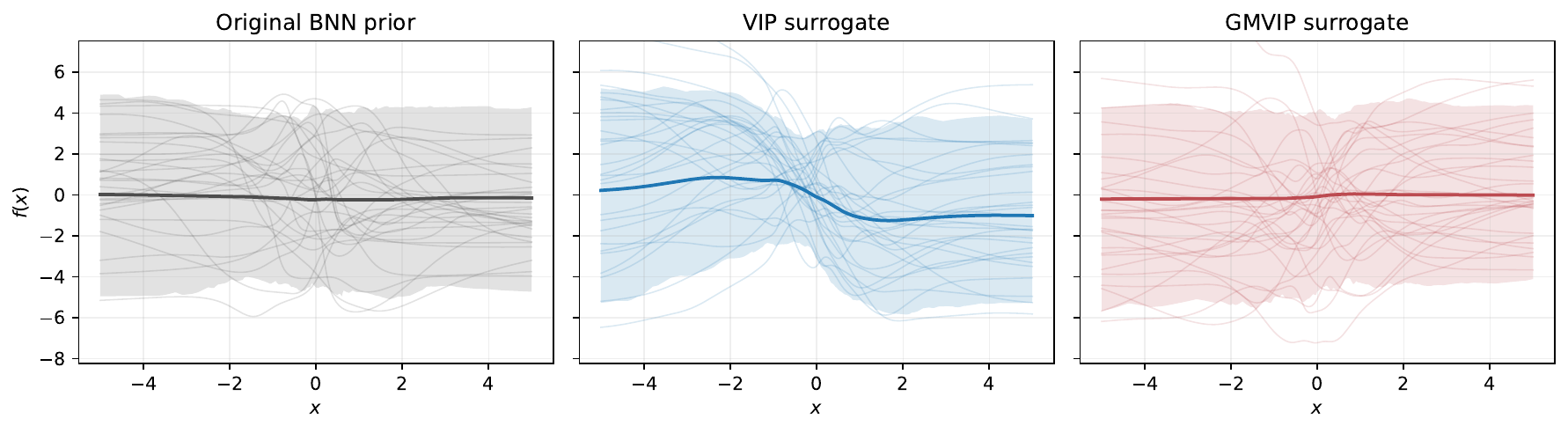}
    \caption{Original and surrogate prior paths for the published synthetic
    configurations. Thin curves are sampled functions, thick curves are sample
    means, and shaded regions are pointwise \(90\%\) intervals. VIP uses
    \(S=20\); \GMVIP uses \(M=256\) inducing inputs and \(B=1024\) operator-bank
    paths. The VIP mean and path distribution visibly depend on the particular
    finite basis, whereas the \GMVIP residual paths retain the qualitative
    variability of the original BNN prior.}
    \label{fig:prior-fidelity-samples}
\end{figure*}

\begin{table*}[t]
\centering
\scriptsize
\setlength{\tabcolsep}{3.5pt}
\caption{Surrogate-prior fidelity at the published synthetic configurations
(\(\downarrow\)). Entries are mean \(\pm\) standard error across five seeds.
The true-prior split is the finite-sample noise floor rather than a competing
method. Bold marks the closer of the two nontrivial surrogates.}
\label{tab:prior-fidelity-defaults}
\resizebox{\textwidth}{!}{%
\begin{tabular}{@{}llccccccc@{}}
\toprule
Distribution & Configuration
& \(\operatorname{SW}_2\)
& \(\overline W_1\)
& \(\max_x W_1\)
& \(e_\mu\)
& \(e_{\K}\)
& Energy
& \(\operatorname{MMD}_{\mathrm{RBF}}^2\) \\
\midrule
True-prior split & \(R_{\mathrm{ev}}=2048\)
& \(0.064 \pm 0.004\)
& \(0.045 \pm 0.005\)
& \(0.066 \pm 0.007\)
& \(0.034 \pm 0.008\)
& \(0.056 \pm 0.010\)
& \(0.0032 \pm 0.0003\)
& \(0.0011 \pm 0.0001\) \\
VIP surrogate & \(S=20\)
& \(0.253 \pm 0.032\)
& \(0.214 \pm 0.030\)
& \(0.312 \pm 0.029\)
& \(0.195 \pm 0.030\)
& \(0.369 \pm 0.073\)
& \(0.0317 \pm 0.0088\)
& \(0.0117 \pm 0.0028\) \\
\GMVIP surrogate & \(M=256,\ B=1024\)
& \(\bm{0.072 \pm 0.004}\)
& \(\bm{0.049 \pm 0.003}\)
& \(\bm{0.069 \pm 0.006}\)
& \(\bm{0.040 \pm 0.005}\)
& \(\bm{0.073 \pm 0.009}\)
& \(\bm{0.0056 \pm 0.0001}\)
& \(\bm{0.0022 \pm 0.0001}\) \\
\bottomrule
\end{tabular}}
\end{table*}

At the published settings, \GMVIP reduces \(\operatorname{SW}_2\) by
\(71.4\%\) and average marginal \(W_1\) by \(76.9\%\) relative to VIP.
Its joint and marginal distances are close to the corresponding true-split
floors: \(0.072\) versus \(0.064\) for \(\operatorname{SW}_2\), and \(0.049\)
versus \(0.045\) for \(\overline W_1\). The maximum pointwise discrepancy is
also near the floor (\(0.069\) versus \(0.066\)), whereas VIP reaches
\(0.312\). Figure~\ref{fig:prior-fidelity-pointwise} shows that this is not an
average hiding a localized failure: \GMVIP tracks the split baseline across
the full input interval, while VIP remains substantially farther away,
especially away from the center.

\begin{figure}[t]
    \centering
    \includegraphics[width=0.96\linewidth]{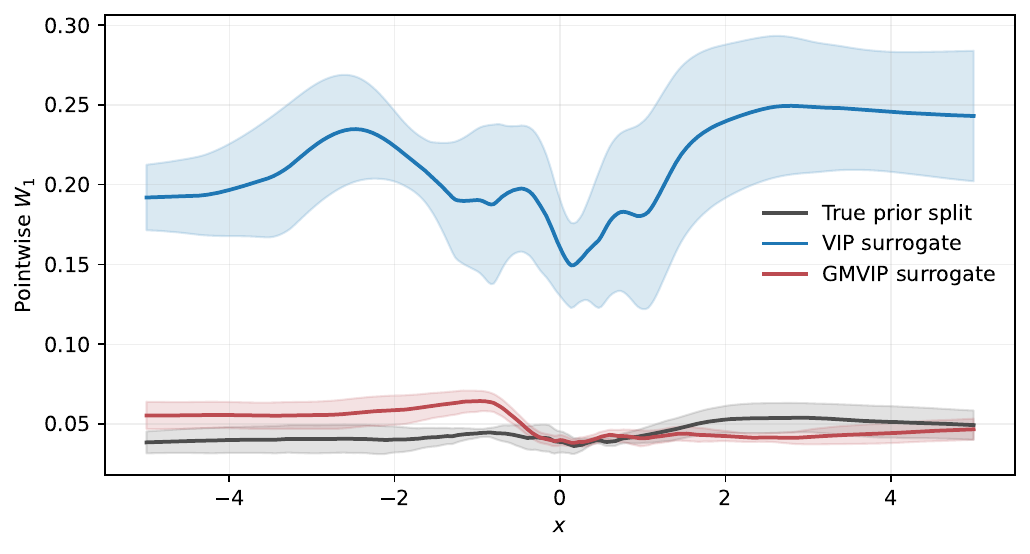}
    \caption{Pointwise marginal \(W_1\) for the published VIP and \GMVIP
    configurations. Curves and bands show mean \(\pm\) standard error across
    five seeds. The \GMVIP surrogate follows the true-prior split floor across
    \([-5,5]\); the VIP finite-basis surrogate has a much larger discrepancy
    throughout the interval.}
    \label{fig:prior-fidelity-pointwise}
\end{figure}

\paragraph{Matched coefficient dimension.}
To separate the published configuration choice from the pathwise
construction, we match the finite coefficient dimension at
\(S=M\in\{8,20,32,64,128,256\}\), holding the \GMVIP operator bank fixed at
\(B=1024\). This matches the dimension of the standard-normal coefficient,
not total sampling cost: \GMVIP additionally retains a fresh residual prior
path by design.
Table~\ref{tab:prior-fidelity-matched} reports the two core distances, while
Figure~\ref{fig:prior-fidelity-dimension} shows their dimension dependence.

\begin{table*}[t]
\centering
\small
\caption{Core distances under matched coefficient dimension
(\(\downarrow\)). \GMVIP uses \(B=1024\). Entries are mean \(\pm\) standard
error across five seeds; bold marks the smaller distance in each pair.}
\label{tab:prior-fidelity-matched}
\begin{tabular}{@{}ccccc@{}}
\toprule
Coefficient dim.
& \multicolumn{2}{c}{\(\operatorname{SW}_2\)}
& \multicolumn{2}{c}{\(\overline W_1\)} \\
\cmidrule(lr){2-3}\cmidrule(lr){4-5}
\(S=M\) & VIP & \GMVIP & VIP & \GMVIP \\
\midrule
\(8\)   & \(0.367 \pm 0.043\) & \(\bm{0.083 \pm 0.005}\) & \(0.303 \pm 0.039\) & \(\bm{0.059 \pm 0.007}\) \\
\(20\)  & \(0.253 \pm 0.032\) & \(\bm{0.078 \pm 0.004}\) & \(0.214 \pm 0.030\) & \(\bm{0.055 \pm 0.007}\) \\
\(32\)  & \(0.189 \pm 0.037\) & \(\bm{0.077 \pm 0.003}\) & \(0.144 \pm 0.031\) & \(\bm{0.049 \pm 0.001}\) \\
\(64\)  & \(0.144 \pm 0.018\) & \(\bm{0.081 \pm 0.005}\) & \(0.103 \pm 0.014\) & \(\bm{0.054 \pm 0.006}\) \\
\(128\) & \(0.119 \pm 0.022\) & \(\bm{0.074 \pm 0.006}\) & \(0.086 \pm 0.024\) & \(\bm{0.049 \pm 0.005}\) \\
\(256\) & \(0.100 \pm 0.005\) & \(\bm{0.072 \pm 0.004}\) & \(0.074 \pm 0.007\) & \(\bm{0.049 \pm 0.003}\) \\
\bottomrule
\end{tabular}
\end{table*}

\begin{figure*}[t]
    \centering
    \includegraphics[width=0.98\textwidth]{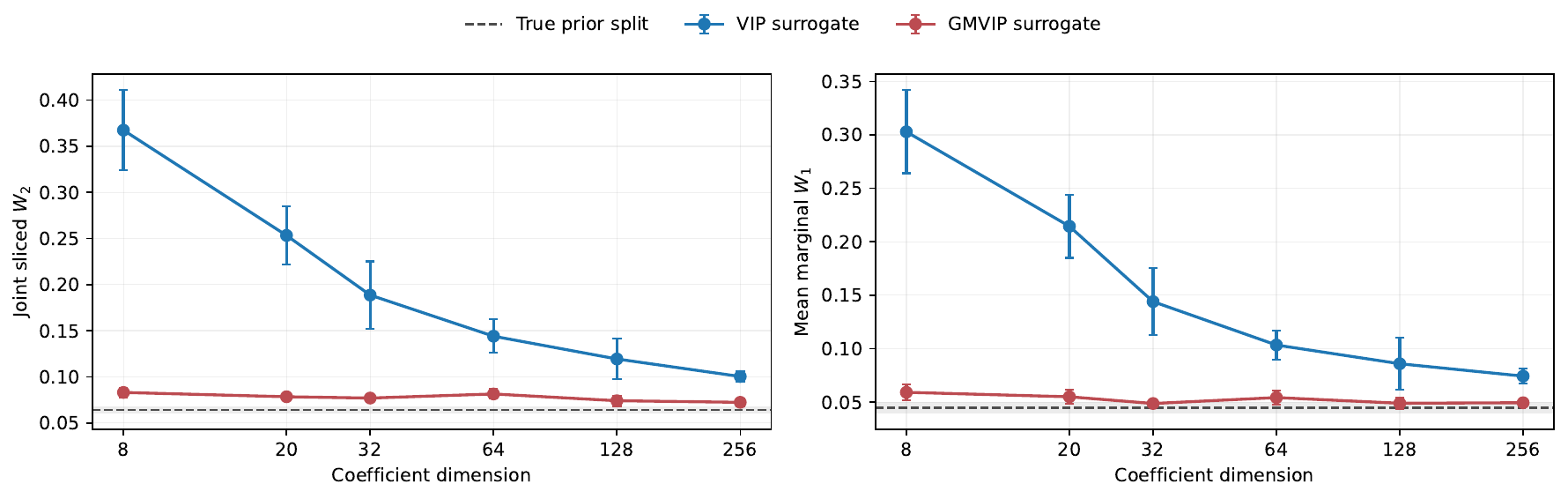}
    \caption{Joint sliced \(W_2\) and average marginal \(W_1\) against matched
    coefficient dimension. Error bars are standard errors across five seeds;
    the dashed line and gray band show the true-prior split mean and standard
    error. VIP approaches the original prior as its sampled-function basis
    grows, whereas \GMVIP is near the finite-sample floor throughout the
    sweep.}
    \label{fig:prior-fidelity-dimension}
\end{figure*}

\GMVIP is closer on both core metrics at every matched dimension. Its
\(\operatorname{SW}_2\) advantage ranges from \(77.4\%\) at dimension \(8\)
to \(27.9\%\) at dimension \(256\); the corresponding
\(\overline W_1\) reductions range from \(80.5\%\) to \(33.5\%\).
Notably, \GMVIP with only \(M=8\) has lower joint and marginal distances than
VIP with \(S=256\). VIP improves consistently with basis size, as expected
from its empirical moment approximation, but its Gaussian finite-basis law
remains detectably farther from this non-Gaussian BNN prior.
Figure~\ref{fig:prior-fidelity-moments} compares the corresponding mean and
covariance errors.

\begin{figure*}[t]
    \centering
    \includegraphics[width=0.98\textwidth]{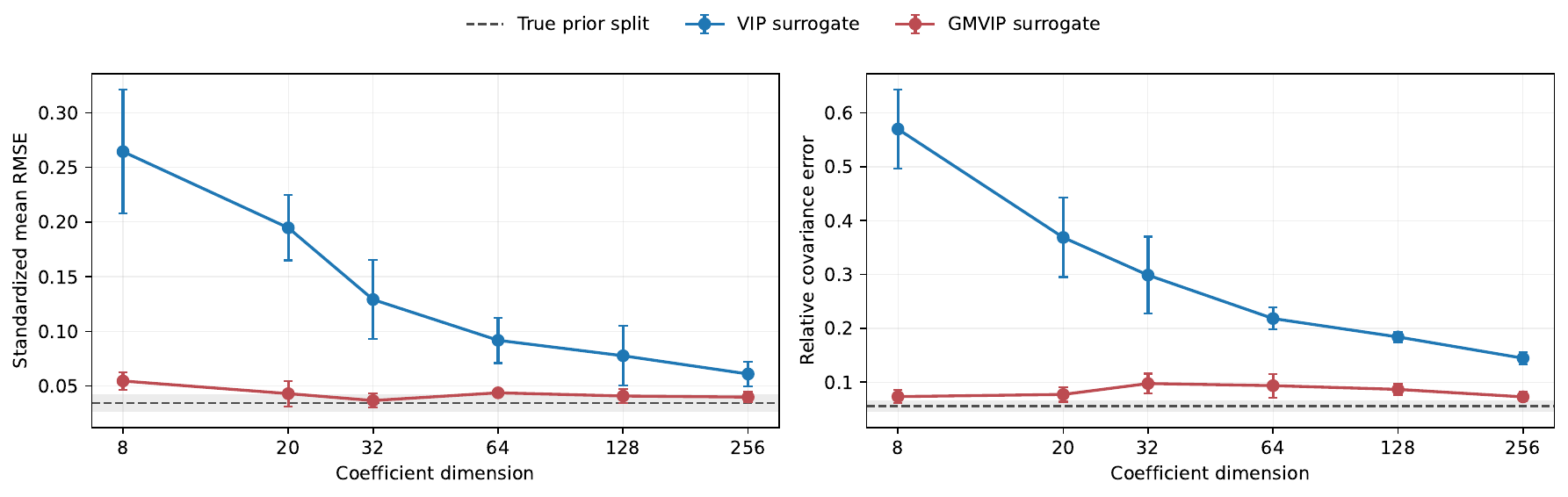}
    \caption{Standardized mean RMSE and relative covariance error under matched
    coefficient dimension. \GMVIP remains close to the finite-sample floor
    because it combines an empirical prior-covariance correction with a fresh
    residual prior path. VIP moment errors decrease with basis size but remain
    larger over the tested range.}
    \label{fig:prior-fidelity-moments}
\end{figure*}

The moment diagnostics agree with the distributional distances.
At the published settings, the \GMVIP mean error is \(0.040\), compared with
\(0.034\) for the true split and \(0.195\) for VIP. Its relative covariance
error is \(0.073\), compared with \(0.056\) for the split and \(0.369\) for
VIP. Thus the empirical results support the population second-order
preservation result while also quantifying finite-bank estimation error.
Figure~\ref{fig:prior-fidelity-robustness} provides the energy-distance and
kernel-MMD robustness checks.

\begin{figure*}[t]
    \centering
    \includegraphics[width=0.98\textwidth]{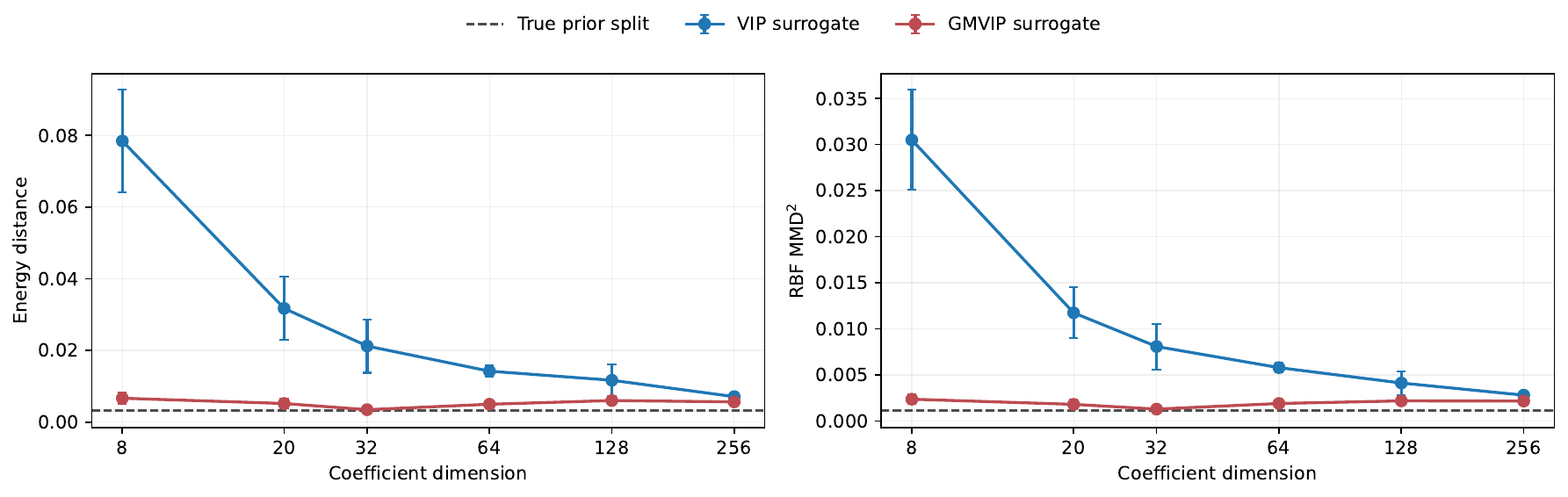}
    \caption{Energy distance and multi-bandwidth RBF
    \(\operatorname{MMD}_{\mathrm{RBF}}^2\) under matched coefficient
    dimension. These
    metrics use a deterministic \(512\)-sample prefix. They reproduce the
    Wasserstein conclusion: \GMVIP stays close to the true-split floor, while
    VIP requires a much larger basis to reduce the discrepancy.}
    \label{fig:prior-fidelity-robustness}
\end{figure*}

The independent robustness metrics lead to the same ordering. At the
published settings, the energy distance is \(0.0056\) for \GMVIP,
\(0.0317\) for VIP, and \(0.0032\) for the true split. The corresponding
\(\operatorname{MMD}_{\mathrm{RBF}}^2\) values are \(0.0022\), \(0.0117\),
and \(0.0011\).
The agreement across marginal, projected-joint, moment, energy, and kernel
metrics makes the conclusion insensitive to any single discrepancy measure.

\paragraph{Operator-bank sensitivity.}
Finally, we fix \(M=256\) and vary the number of prior paths used to estimate
the empirical operator. Table~\ref{tab:prior-fidelity-bank} and
Figure~\ref{fig:prior-fidelity-bank} show that the discrepancy is not strictly
monotone for the smaller banks: \(B=512\) is worse than \(B=256\) on this
five-seed experiment. This is consistent with finite-bank and covariance
stabilization variability rather than a deterministic monotonicity guarantee.
The larger banks improve sharply. At \(B=2048\),
\(\operatorname{SW}_2=0.066\) and \(\overline W_1=0.043\), which are at the
scale of the true-split floors \(0.064\) and \(0.045\).

\begin{table*}[t]
\centering
\scriptsize
\setlength{\tabcolsep}{4pt}
\caption{\GMVIP operator-bank sensitivity at \(M=256\) (\(\downarrow\)).
Entries are mean \(\pm\) standard error across five seeds.}
\label{tab:prior-fidelity-bank}
\resizebox{\textwidth}{!}{%
\begin{tabular}{@{}ccccccc@{}}
\toprule
Bank \(B\)
& \(\operatorname{SW}_2\)
& \(\overline W_1\)
& \(e_\mu\)
& \(e_{\K}\)
& Energy
& \(\operatorname{MMD}_{\mathrm{RBF}}^2\) \\
\midrule
\(256\)
& \(0.093 \pm 0.013\)
& \(0.065 \pm 0.011\)
& \(0.058 \pm 0.017\)
& \(0.105 \pm 0.020\)
& \(0.0083 \pm 0.0018\)
& \(0.0028 \pm 0.0006\) \\
\(512\)
& \(0.106 \pm 0.008\)
& \(0.083 \pm 0.009\)
& \(0.071 \pm 0.011\)
& \(0.116 \pm 0.022\)
& \(0.0081 \pm 0.0018\)
& \(0.0029 \pm 0.0006\) \\
\(1024\)
& \(0.072 \pm 0.004\)
& \(0.049 \pm 0.003\)
& \(0.040 \pm 0.005\)
& \(0.073 \pm 0.009\)
& \(0.0056 \pm 0.0001\)
& \(0.0022 \pm 0.0001\) \\
\(2048\)
& \(\bm{0.066 \pm 0.005}\)
& \(\bm{0.043 \pm 0.005}\)
& \(\bm{0.035 \pm 0.007}\)
& \(\bm{0.059 \pm 0.011}\)
& \(\bm{0.0044 \pm 0.0011}\)
& \(\bm{0.0015 \pm 0.0005}\) \\
\bottomrule
\end{tabular}}
\end{table*}

\begin{figure*}[t]
    \centering
    \includegraphics[width=0.98\textwidth]{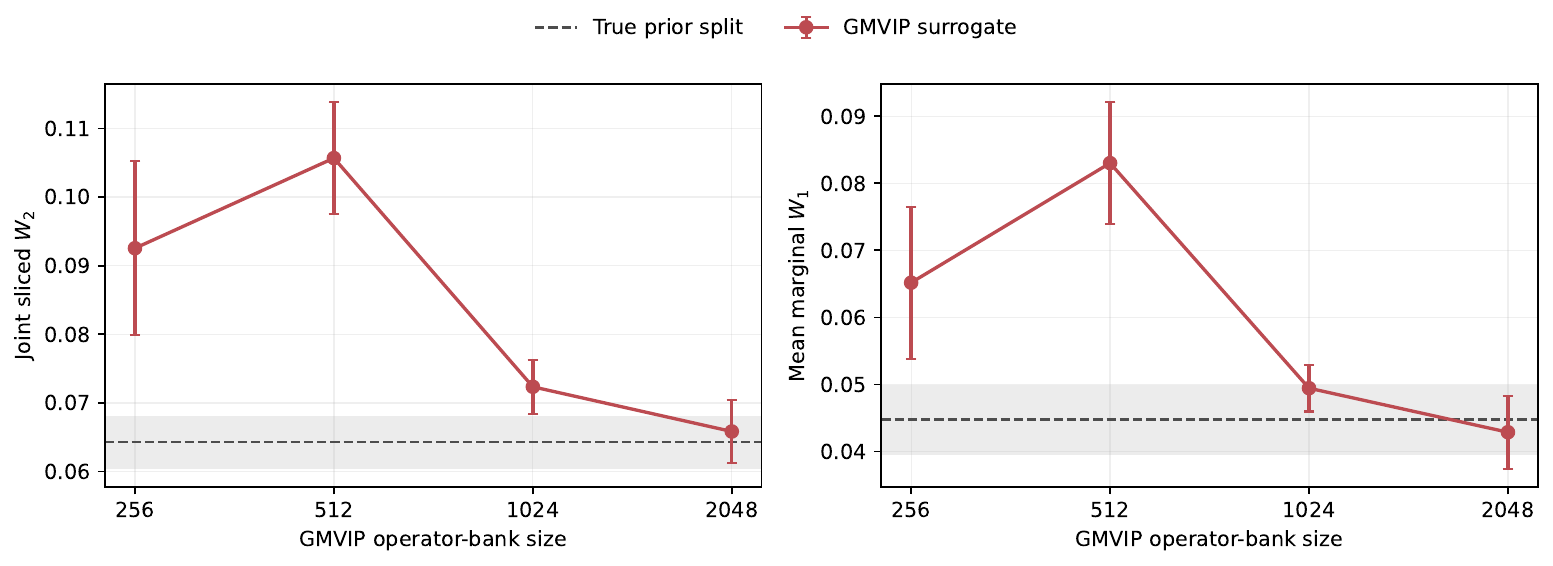}
    \caption{\GMVIP joint and marginal prior discrepancy as a function of
    operator-bank size at \(M=256\). The dashed line and gray band denote the
    true-prior split floor. At \(B=2048\), the error bars overlap the floor at
    the resolution of these sample-based diagnostics, although the
    smaller-bank trend is non-monotone.}
    \label{fig:prior-fidelity-bank}
\end{figure*}

\paragraph{Interpretation and limitations.}
The experiment isolates prior fidelity from posterior optimization and
predictive performance. It shows that the residual-path construction makes
the empirical \GMVIP surrogate substantially closer to the original BNN prior
than the VIP finite-basis surrogate, and that most of the remaining detectable
gap at the published setting is reduced by enlarging the operator bank.
However, the true-split floor prevents interpreting a distance numerically
close to that floor as proof of equality. The measurements concern one
non-Gaussian BNN prior on one \(301\)-point grid; they do not establish equality
of infinite-dimensional process laws or uniform behavior for every implicit
prior. They instead provide a direct finite-dimensional complement to
Proposition~\ref{prop:prior_distortion} and empirical support for the
second-order statement in Proposition~\ref{prop:second_order_preservation}.

\subsection{Classification Experiments}\label{app:classification}

\paragraph{LeNet classification.}
Table~\ref{tab:lenet-last-layer-classification-all-metrics} summarizes
this first setting.
We first evaluate image classification on FashionMNIST and CIFAR-10 using a
LeNet-style architecture with deterministic convolutional features and
Bayesian classifier layers. We report NLL, classification error, ECE, and
Brier score over five random seeds. This setting tests whether the proposed
function-space posterior improves predictive uncertainty and accuracy while
keeping the architecture lightweight enough for repeated Bayesian
comparisons. GMVIP is consistently competitive in this Bayesian-head LeNet
regime: on FashionMNIST it attains the best classification error and Brier
score, and is tied for the best rounded NLL with FBNN. On CIFAR-10, SIP obtains
the strongest NLL, error, and Brier score, but GMVIP remains among the best
methods, with the second-best error and performance close to FBNN on NLL and
Brier score. MFVI gives the lowest ECE, although this comes with worse
predictive accuracy and likelihood on several metrics. These results indicate
that GMVIP provides a favorable accuracy--uncertainty trade-off in the
Bayesian-head LeNet setting, particularly on FashionMNIST, whereas CIFAR-10 is
more challenging and favors SIP under this representation.

\begin{table*}[t]
\centering
\scriptsize
\setlength{\tabcolsep}{3.0pt}
\renewcommand{\arraystretch}{0.92}
\caption{Bayesian-head LeNet classification metrics ($\downarrow$). Entries are mean $\pm$ standard deviation across five seeds. Best results are shown in \textcolor{purple}{\textbf{purple}}; second-best are shown in \textcolor{teal}{\textbf{teal}}.}
\label{tab:lenet-last-layer-classification-all-metrics}
\resizebox{\textwidth}{!}{%
\begin{tabular}{clccccccc}
\toprule
\textbf{Metric} & \textbf{Dataset} & MAP & MFVI & VIP & FBNN & SIP & FTIP & GMVIP \\
\midrule
\multirow{2}{*}{\textbf{NLL}}
& FashionMNIST & $0.287 \pm 0.012$ & $0.296 \pm 0.014$ & $0.282 \pm 0.015$ & $\best{0.261 \pm 0.016}$ & $0.277 \pm 0.009$ & $\second{0.274 \pm 0.017}$ & $\best{0.261 \pm 0.015}$ \\
& CIFAR10 & $1.099 \pm 0.054$ & $1.189 \pm 0.023$ & $1.190 \pm 0.081$ & $\second{1.093 \pm 0.065}$ & $\best{1.039 \pm 0.065}$ & $1.344 \pm 0.109$ & $\second{1.093 \pm 0.025}$ \\
\midrule
\multirow{2}{*}{\textbf{Error}}
& FashionMNIST & $0.102 \pm 0.004$ & $0.108 \pm 0.006$ & $0.099 \pm 0.006$ & $\second{0.093 \pm 0.005}$ & $0.094 \pm 0.005$ & $0.096 \pm 0.005$ & $\best{0.092 \pm 0.005}$ \\
& CIFAR10 & $0.374 \pm 0.018$ & $0.424 \pm 0.010$ & $0.372 \pm 0.012$ & $0.350 \pm 0.009$ & $\best{0.342 \pm 0.010}$ & $0.374 \pm 0.016$ & $\second{0.348 \pm 0.010}$ \\
\midrule
\multirow{2}{*}{\textbf{ECE}}
& FashionMNIST & $0.020 \pm 0.006$ & $\best{0.016 \pm 0.003}$ & $\second{0.019 \pm 0.004}$ & $0.020 \pm 0.005$ & $0.024 \pm 0.009$ & $0.021 \pm 0.007$ & $0.021 \pm 0.005$ \\
& CIFAR10 & $\second{0.069 \pm 0.015}$ & $\best{0.045 \pm 0.011}$ & $0.121 \pm 0.023$ & $0.105 \pm 0.017$ & $0.071 \pm 0.033$ & $0.165 \pm 0.019$ & $0.102 \pm 0.004$ \\
\midrule
\multirow{2}{*}{\textbf{Brier}}
& FashionMNIST & $0.148 \pm 0.006$ & $0.154 \pm 0.008$ & $0.144 \pm 0.008$ & $\second{0.134 \pm 0.007}$ & $0.138 \pm 0.007$ & $0.140 \pm 0.008$ & $\best{0.133 \pm 0.007}$ \\
& CIFAR10 & $0.505 \pm 0.021$ & $0.553 \pm 0.009$ & $0.517 \pm 0.022$ & $\second{0.485 \pm 0.018}$ & $\best{0.472 \pm 0.017}$ & $0.543 \pm 0.026$ & $0.486 \pm 0.010$ \\
\bottomrule
\end{tabular}%
}
\end{table*}

\paragraph{Frozen-CLIP protocol.}
We additionally evaluate all eight methods with linear classifier heads over
frozen, normalized \(512\)-dimensional CLIP ViT-B/32 embeddings. These
experiments separate uncertainty in the classifier from representation
learning: the encoder is never updated, and each method differs only in its
linear head and posterior construction. The shared benchmark fixes the
learning rate at \(10^{-3}\) for every method. Every method is otherwise
tuned independently, the selected head is retrained with \(3\) seeds, and Bayesian predictions average \(100\) posterior samples.
A positive scalar temperature is fitted on a held-out calibration split after
training. Temperature scaling does not change the predicted class, so the
reported accuracy is unchanged, while NLL and ECE are temperature calibrated.

\paragraph{Frozen-CLIP CIFAR-10.}
For CIFAR-10, the \(50\,000\) training images are divided into \(40\,000\)
tuning examples, \(5\,000\) model-selection examples, and \(5\,000\)
temperature-calibration examples. After model selection, each winning head is
retrained on the combined \(45\,000\) fitting examples and evaluated once on
the untouched \(10\,000\)-image test set. Table~\ref{tab:clip-cifar10-linear}
shows that GMVIP is strongest overall. It obtains \(94.45\%\) accuracy, NLL
\(0.1667\), and ECE \(0.0058\), giving the best accuracy and NLL and the
second-best ECE. VIP ranks second in accuracy (\(94.42\%\)) and NLL
(\(0.1700\)), while SIP has the lowest ECE (\(0.0057\)). GMVIP gives the strongest
accuracy--uncertainty trade-off on frozen-CLIP CIFAR-10.

\begin{table}[t]
\centering
\scriptsize
\setlength{\tabcolsep}{4.5pt}
\renewcommand{\arraystretch}{0.96}
\caption{Frozen-CLIP CIFAR-10 test results with linear classifier heads.
Entries are mean \(\pm\) standard deviation over three seeds. Accuracy is in
percentage points; NLL and ECE are temperature calibrated. Best mean is shown
in \textcolor{purple}{\textbf{purple}} and second-best mean in
\textcolor{teal}{\textbf{teal}}.}
\label{tab:clip-cifar10-linear}
\resizebox{0.7\linewidth}{!}{%
\begin{tabular}{lccc}
\toprule
Method & Accuracy (\%) \(\uparrow\) & NLL \(\downarrow\) & ECE \(\downarrow\) \\
\midrule
MAP   & $93.93 \pm 0.02$ & $0.1855 \pm 0.0001$ & $0.0092 \pm 0.0001$ \\
MFVI  & $93.80 \pm 0.06$ & $0.1872 \pm 0.0001$ & $0.0088 \pm 0.0005$ \\
VIP   & $\second{94.42 \pm 0.05}$ & $\second{0.1700 \pm 0.0008}$ & $0.0093 \pm 0.0021$ \\
FBNN  & $93.94 \pm 0.03$ & $0.1851 \pm 0.0001$ & $0.0092 \pm 0.0003$ \\
SIP   & $93.05 \pm 0.03$ & $0.2107 \pm 0.0021$ & $\best{0.0057 \pm 0.0011}$ \\
TFSVI & $93.94 \pm 0.02$ & $0.1851 \pm 0.0001$ & $0.0105 \pm 0.0005$ \\
FTIP  & $94.19 \pm 0.03$ & $0.1741 \pm 0.0016$ & $0.0068 \pm 0.0015$ \\
GMVIP & $\best{94.45 \pm 0.02}$ & $\best{0.1667 \pm 0.0014}$ & $\second{0.0058 \pm 0.0008}$ \\
\bottomrule
\end{tabular}}
\end{table}

\paragraph{Frozen-CLIP DynaSent.}
We next consider three-class sentiment classification on DynaSent. CLIP text
embeddings are fitted with linear heads using \(80\,488\) Round-1 training
sentences for tuning and one stratified half of the Round-1 development set
(\(1\,800\) examples) for model selection. The selected head is retrained on
the resulting \(82\,288\) examples, while the other \(1\,800\) development
examples are reserved for temperature calibration. We evaluate on the
\(3\,600\)-example Round-1 test set (in distribution) and the \(720\)-example
Round-2 Dynabench test set (shifted). In addition to accuracy, NLL, and ECE, we
report area under the risk--coverage curve (AURC), for which lower values mean
that predictive confidence ranks correct examples ahead of errors more
effectively.

Table~\ref{tab:clip-dynasent} shows that GMVIP performs best on the Round-1
test set in accuracy, NLL, and AURC, attaining \(59.48\%\), \(0.8769\), and
\(0.2503\), respectively. Under the Round-2 shift, GMVIP retains the best NLL
(\(0.9163\)) and AURC (\(0.3111\)), and has the second-best accuracy
(\(57.41\%\)) behind VIP (\(57.96\%\)). SIP has the lowest shifted ECE, but
only two of its three final seeds completed.

\begin{table*}[t]
\centering
\scriptsize
\setlength{\tabcolsep}{3.5pt}
\renewcommand{\arraystretch}{0.96}
\caption{Frozen-CLIP DynaSent results with linear three-class heads. Round 1
is the in-distribution test set and Round 2 is the adversarially collected
shifted test set. Entries are mean \(\pm\) standard deviation; SIP has two
completed seeds and all other methods have three. Accuracy is in percentage
points; NLL and ECE are temperature calibrated. Lower is better except for
accuracy. AURC is the area under the risk--coverage curve. Best and second-best
means are highlighted as above.}
\label{tab:clip-dynasent}
\resizebox{\textwidth}{!}{%
\begin{tabular}{lcccccccc}
\toprule
& \multicolumn{4}{c}{Round 1 (in distribution)}
& \multicolumn{4}{c}{Round 2 (shifted)} \\
\cmidrule(lr){2-5}\cmidrule(lr){6-9}
Method
& Acc. (\%) \(\uparrow\) & NLL \(\downarrow\) & ECE \(\downarrow\) & AURC \(\downarrow\)
& Acc. (\%) \(\uparrow\) & NLL \(\downarrow\) & ECE \(\downarrow\) & AURC \(\downarrow\) \\
\midrule
MAP
& $57.61 \pm 0.03$ & $0.8985 \pm 0.0001$ & $0.0437 \pm 0.0003$ & $0.2706 \pm 0.0002$
& $55.74 \pm 0.08$ & $0.9335 \pm 0.0001$ & $0.0351 \pm 0.0037$ & $0.3262 \pm 0.0002$ \\
MFVI
& $57.23 \pm 0.13$ & $0.9031 \pm 0.0003$ & $0.0433 \pm 0.0004$ & $0.2749 \pm 0.0005$
& $55.19 \pm 0.08$ & $0.9374 \pm 0.0004$ & $\second{0.0319 \pm 0.0004}$ & $0.3284 \pm 0.0005$ \\
VIP
& $\second{59.30 \pm 0.21}$ & $\second{0.8796 \pm 0.0005}$ & $\best{0.0385 \pm 0.0006}$ & $\second{0.2530 \pm 0.0006}$
& $\best{57.96 \pm 0.29}$ & $\second{0.9209 \pm 0.0023}$ & $0.0372 \pm 0.0035$ & $\second{0.3168 \pm 0.0005}$ \\
FBNN
& $58.48 \pm 0.40$ & $0.8888 \pm 0.0027$ & $\second{0.0388 \pm 0.0007}$ & $0.2615 \pm 0.0030$
& $56.81 \pm 0.28$ & $0.9261 \pm 0.0011$ & $0.0417 \pm 0.0045$ & $0.3205 \pm 0.0005$ \\
SIP
& $57.21 \pm 0.29$ & $0.9073 \pm 0.0017$ & $0.0418 \pm 0.0019$ & $0.2779 \pm 0.0020$
& $55.83 \pm 0.20$ & $0.9385 \pm 0.0020$ & $\best{0.0296 \pm 0.0011}$ & $0.3274 \pm 0.0009$ \\
TFSVI
& $57.60 \pm 0.06$ & $0.8985 \pm 0.0001$ & $0.0441 \pm 0.0003$ & $0.2707 \pm 0.0002$
& $55.69 \pm 0.00$ & $0.9335 \pm 0.0000$ & $0.0349 \pm 0.0009$ & $0.3264 \pm 0.0002$ \\
FTIP
& $56.61 \pm 0.93$ & $0.9095 \pm 0.0077$ & $0.0432 \pm 0.0056$ & $0.2800 \pm 0.0073$
& $56.76 \pm 0.70$ & $0.9453 \pm 0.0054$ & $0.0345 \pm 0.0036$ & $0.3319 \pm 0.0011$ \\
GMVIP
& $\best{59.48 \pm 0.37}$ & $\best{0.8769 \pm 0.0032}$ & $0.0403 \pm 0.0029$ & $\best{0.2503 \pm 0.0032}$
& $\second{57.41 \pm 0.21}$ & $\best{0.9163 \pm 0.0012}$ & $0.0365 \pm 0.0042$ & $\best{0.3111 \pm 0.0000}$ \\
\bottomrule
\end{tabular}}
\end{table*}

Taken together, Tables~\ref{tab:lenet-last-layer-classification-all-metrics},
\ref{tab:clip-cifar10-linear}, and \ref{tab:clip-dynasent} show that the value
of a Bayesian classifier depends on the representation and data regime. GMVIP
is strongest in the FashionMNIST LeNet comparison and is the strongest method
under the common fixed-learning-rate frozen-CLIP CIFAR-10 protocol. On
DynaSent, GMVIP provides the strongest overall in-distribution and shifted
predictive performance. These results indicate that its function-space
posterior is particularly useful when the representation retains label
ambiguity, while still remaining effective with highly separable frozen image
features.

\subsection{Full UCI regression results}\label{app:full_uci}

Table~\ref{tab:uci-all-metrics-gmvip-wandb2} completes the results of Table~\ref{tab:uci-all-metrics-gmvip-wandb} by including CRPS and CQM metrics of the experiments.

\begin{table*}[t]
\centering
\scriptsize
\setlength{\tabcolsep}{3.5pt}
\renewcommand{\arraystretch}{0.92}
\caption{UCI regression metrics ($\downarrow$). Entries are mean $\pm$ standard deviation across seeds. Frozen-prior variants are omitted. Best mean per dataset is shown in \textcolor{purple}{\textbf{purple}}; second-best mean is shown in \textcolor{teal}{\textbf{teal}}. Naval RMSE and CRPS are reported in units of $10^{-4}$.}
\label{tab:uci-all-metrics-gmvip-wandb2}
\resizebox{\textwidth}{!}{%
\begin{tabular}{lccccccccc}
\toprule
\textbf{RMSE} & Boston & Concrete & Energy & Kin8nm & Naval ($\times 10^{-4}$) & Power & Protein & Wine Red & Yacht \\
\midrule
MAP & $4.34 \pm 1.41$ & $5.45 \pm 0.57$ & $0.52 \pm 0.05$ & $0.082 \pm 0.001$ & $4.04 \pm 1.97$ & $3.94 \pm 0.13$ & $4.62 \pm 0.07$ & $0.69 \pm 0.01$ & $0.52 \pm 0.20$ \\
MFVI & $3.93 \pm 0.70$ & $6.23 \pm 0.51$ & $0.62 \pm 0.09$ & $0.091 \pm 0.002$ & $147.61 \pm 0.90$ & $4.12 \pm 0.17$ & $4.94 \pm 0.08$ & $\second{0.62 \pm 0.03}$ & $0.68 \pm 0.10$ \\
VIP & $6.49 \pm 1.27$ & $5.76 \pm 0.97$ & $0.58 \pm 0.09$ & $0.080 \pm 0.003$ & $2.59 \pm 0.96$ & $\best{3.76 \pm 0.07}$ & $4.37 \pm 0.06$ & $0.79 \pm 0.05$ & $0.84 \pm 0.30$ \\
FBNN & $4.17 \pm 1.27$ & $\second{5.21 \pm 0.52}$ & $\second{0.51 \pm 0.05}$ & $0.086 \pm 0.001$ & $3.95 \pm 0.47$ & $4.05 \pm 0.16$ & $4.81 \pm 0.05$ & $0.65 \pm 0.04$ & $\best{0.23 \pm 0.06}$ \\
SIP & $\best{3.22 \pm 0.99}$ & $5.30 \pm 0.27$ & $0.52 \pm 0.04$ & $\second{0.071 \pm 0.001}$ & $\second{1.48 \pm 0.23}$ & $3.84 \pm 0.11$ & $4.62 \pm 0.05$ & $0.66 \pm 0.03$ & $0.41 \pm 0.15$ \\
TFSVI & $6.39 \pm 0.63$ & $5.97 \pm 0.85$ & $0.59 \pm 0.15$ & $0.081 \pm 0.002$ & $3.95 \pm 0.81$ & $3.90 \pm 0.13$ & $4.61 \pm 0.05$ & $0.73 \pm 0.02$ & $0.83 \pm 0.38$ \\
FTIP & $6.68 \pm 1.58$ & $6.02 \pm 1.08$ & $0.55 \pm 0.07$ & $0.079 \pm 0.002$ & $3.36 \pm 1.58$ & $\second{3.77 \pm 0.11}$ & $\second{4.36 \pm 0.05}$ & $0.82 \pm 0.03$ & $1.17 \pm 0.52$ \\
GMVIP & $\second{3.32 \pm 0.97}$ & $\best{4.24 \pm 0.74}$ & $\best{0.51 \pm 0.05}$ & $\best{0.069 \pm 0.001}$ & $\best{1.36 \pm 0.19}$ & $3.78 \pm 0.12$ & $\best{4.32 \pm 0.04}$ & $\best{0.60 \pm 0.03}$ & $\second{0.32 \pm 0.11}$ \\
\midrule
\textbf{NLL} & Boston & Concrete & Energy & Kin8nm & Naval & Power & Protein & Wine Red & Yacht \\
\midrule
MAP & $7.62 \pm 3.48$ & $3.24 \pm 0.17$ & $\second{0.81 \pm 0.14}$ & $-1.08 \pm 0.02$ & $-6.14 \pm 1.03$ & $2.79 \pm 0.03$ & $2.95 \pm 0.01$ & $1.27 \pm 0.04$ & $16.18 \pm 15.06$ \\
MFVI & $\best{2.76 \pm 0.14}$ & $3.25 \pm 0.08$ & $0.94 \pm 0.14$ & $-0.98 \pm 0.03$ & $-2.80 \pm 0.01$ & $2.84 \pm 0.04$ & $3.02 \pm 0.02$ & $\best{0.94 \pm 0.04}$ & $1.08 \pm 0.05$ \\
VIP & $4245.41 \pm 2526.37$ & $4.80 \pm 1.11$ & $2.48 \pm 0.60$ & $-1.09 \pm 0.04$ & $-6.78 \pm 0.49$ & $\best{2.74 \pm 0.02}$ & $2.89 \pm 0.01$ & $3.16 \pm 0.39$ & $96.61 \pm 53.56$ \\
FBNN & $5.24 \pm 2.45$ & $\second{3.13 \pm 0.16}$ & $0.84 \pm 0.22$ & $-1.04 \pm 0.01$ & $-6.42 \pm 0.11$ & $2.82 \pm 0.04$ & $2.99 \pm 0.01$ & $1.08 \pm 0.08$ & $\second{0.30 \pm 0.84}$ \\
SIP & $3.27 \pm 1.22$ & $3.18 \pm 0.09$ & $0.84 \pm 0.14$ & $\second{-1.22 \pm 0.02}$ & $\second{-7.04 \pm 0.29}$ & $2.76 \pm 0.03$ & $2.95 \pm 0.01$ & $1.09 \pm 0.07$ & $5.65 \pm 5.09$ \\
TFSVI & $11.32 \pm 2.79$ & $3.38 \pm 0.23$ & $1.05 \pm 0.44$ & $-1.09 \pm 0.03$ & $-6.41 \pm 0.22$ & $2.78 \pm 0.03$ & $2.95 \pm 0.01$ & $1.29 \pm 0.07$ & $14.30 \pm 10.07$ \\
FTIP & $1789.45 \pm 735.75$ & $4.82 \pm 1.21$ & $2.13 \pm 0.36$ & $-1.11 \pm 0.02$ & $-6.47 \pm 0.59$ & $\second{2.75 \pm 0.03}$ & $\second{2.89 \pm 0.01}$ & $3.52 \pm 0.45$ & $56.66 \pm 56.15$ \\
GMVIP & $\second{3.11 \pm 1.03}$ & $\best{2.82 \pm 0.16}$ & $\best{0.75 \pm 0.09}$ & $\best{-1.26 \pm 0.02}$ & $\best{-7.12 \pm 0.04}$ & $2.75 \pm 0.03$ & $\best{2.88 \pm 0.01}$ & $\second{0.94 \pm 0.08}$ & $\best{0.18 \pm 0.22}$ \\
\midrule
\textbf{CRPS} & Boston & Concrete & Energy & Kin8nm & Naval ($\times 10^{-4}$) & Power & Protein & Wine Red & Yacht \\
\midrule
MAP & $2.33 \pm 0.43$ & $3.00 \pm 0.28$ & $0.28 \pm 0.02$ & $0.046 \pm 0.001$ & $2.46 \pm 1.52$ & $2.17 \pm 0.06$ & $2.60 \pm 0.05$ & $0.39 \pm 0.01$ & $0.20 \pm 0.05$ \\
MFVI & $2.06 \pm 0.22$ & $3.48 \pm 0.27$ & $0.34 \pm 0.06$ & $0.051 \pm 0.001$ & $86.17 \pm 0.58$ & $2.29 \pm 0.08$ & $2.79 \pm 0.04$ & $\second{0.34 \pm 0.02}$ & $0.37 \pm 0.04$ \\
VIP & $4.70 \pm 0.88$ & $3.26 \pm 0.48$ & $0.33 \pm 0.05$ & $0.045 \pm 0.002$ & $1.50 \pm 0.64$ & $\best{2.06 \pm 0.05}$ & $2.43 \pm 0.03$ & $0.48 \pm 0.02$ & $0.40 \pm 0.12$ \\
FBNN & $2.31 \pm 0.47$ & $\second{2.86 \pm 0.24}$ & $\second{0.28 \pm 0.02}$ & $0.048 \pm 0.001$ & $2.09 \pm 0.25$ & $2.25 \pm 0.08$ & $2.71 \pm 0.04$ & $0.37 \pm 0.02$ & $\best{0.12 \pm 0.03}$ \\
SIP & $\best{1.69 \pm 0.39}$ & $2.86 \pm 0.15$ & $0.28 \pm 0.02$ & $\second{0.040 \pm 0.001}$ & $\second{1.04 \pm 0.23}$ & $2.10 \pm 0.05$ & $2.60 \pm 0.03$ & $0.37 \pm 0.02$ & $0.16 \pm 0.04$ \\
TFSVI & $3.83 \pm 0.37$ & $3.26 \pm 0.44$ & $0.33 \pm 0.07$ & $0.045 \pm 0.001$ & $2.14 \pm 0.42$ & $2.14 \pm 0.06$ & $2.58 \pm 0.03$ & $0.40 \pm 0.02$ & $0.35 \pm 0.11$ \\
FTIP & $4.92 \pm 1.25$ & $3.40 \pm 0.65$ & $0.32 \pm 0.04$ & $0.044 \pm 0.001$ & $2.03 \pm 1.05$ & $\second{2.06 \pm 0.05}$ & $\second{2.43 \pm 0.03}$ & $0.48 \pm 0.02$ & $0.56 \pm 0.18$ \\
GMVIP & $\second{1.79 \pm 0.37}$ & $\best{2.26 \pm 0.35}$ & $\best{0.28 \pm 0.02}$ & $\best{0.038 \pm 0.001}$ & $\best{0.92 \pm 0.07}$ & $2.07 \pm 0.06$ & $\best{2.41 \pm 0.02}$ & $\best{0.33 \pm 0.01}$ & $\second{0.16 \pm 0.04}$ \\
\midrule
\textbf{CQM} & Boston & Concrete & Energy & Kin8nm & Naval & Power & Protein & Wine Red & Yacht \\
\midrule
MAP & $0.235 \pm 0.053$ & $0.051 \pm 0.009$ & $0.049 \pm 0.021$ & $0.013 \pm 0.009$ & $0.136 \pm 0.142$ & $0.021 \pm 0.010$ & $0.010 \pm 0.003$ & $0.088 \pm 0.015$ & $0.099 \pm 0.028$ \\
MFVI & $\best{0.047 \pm 0.013}$ & $\best{0.029 \pm 0.004}$ & $\best{0.042 \pm 0.013}$ & $\best{0.011 \pm 0.003}$ & $\best{0.072 \pm 0.005}$ & $\second{0.017 \pm 0.007}$ & $\second{0.008 \pm 0.003}$ & $\best{0.037 \pm 0.016}$ & $0.229 \pm 0.050$ \\
VIP & $0.495 \pm 0.003$ & $0.165 \pm 0.036$ & $0.172 \pm 0.033$ & $0.020 \pm 0.013$ & $0.130 \pm 0.062$ & $0.023 \pm 0.010$ & $0.021 \pm 0.001$ & $0.217 \pm 0.009$ & $0.342 \pm 0.052$ \\
FBNN & $0.211 \pm 0.063$ & $\second{0.030 \pm 0.012}$ & $0.051 \pm 0.017$ & $\second{0.011 \pm 0.004}$ & $0.095 \pm 0.013$ & $\best{0.016 \pm 0.007}$ & $0.008 \pm 0.004$ & $0.055 \pm 0.019$ & $\best{0.048 \pm 0.017}$ \\
SIP & $\second{0.089 \pm 0.022}$ & $0.049 \pm 0.021$ & $\second{0.046 \pm 0.012}$ & $0.016 \pm 0.005$ & $0.214 \pm 0.085$ & $0.025 \pm 0.008$ & $\best{0.007 \pm 0.004}$ & $0.057 \pm 0.026$ & $\second{0.067 \pm 0.018}$ \\
TFSVI & $0.292 \pm 0.023$ & $0.069 \pm 0.013$ & $0.063 \pm 0.013$ & $0.013 \pm 0.008$ & $\second{0.088 \pm 0.079}$ & $0.023 \pm 0.005$ & $0.015 \pm 0.004$ & $0.075 \pm 0.017$ & $0.118 \pm 0.072$ \\
FTIP & $0.483 \pm 0.008$ & $0.159 \pm 0.043$ & $0.213 \pm 0.043$ & $0.016 \pm 0.006$ & $0.171 \pm 0.070$ & $0.022 \pm 0.008$ & $0.017 \pm 0.004$ & $0.209 \pm 0.021$ & $0.331 \pm 0.055$ \\
GMVIP & $0.119 \pm 0.065$ & $0.056 \pm 0.031$ & $0.077 \pm 0.026$ & $0.034 \pm 0.010$ & $0.244 \pm 0.044$ & $0.037 \pm 0.010$ & $0.031 \pm 0.005$ & $\second{0.046 \pm 0.012}$ & $0.198 \pm 0.053$ \\
\bottomrule
\end{tabular}%
}
\end{table*}

\subsection{Ablation Studies}\label{app:ablation-studies}

We study the sensitivity of GMVIP to the initialization of the inducing
locations, the number of inducing locations, and the number of prior functions
used to estimate the operator moments.  We use the nine UCI regression
datasets and seeds $0$--$4$, giving 45 runs per setting and 540 runs in total.
Unless it is the quantity being varied, we use $M=100$, a prior bank of
$B=512$, and $k$-means initialization.  The inducing locations $\Z$ and the
prior are subsequently optimized in every run; consequently, the geometry
ablation measures sensitivity to the \emph{initial} placement of $\Z$.

Tables~\ref{tab:gmvip-ablation-initialization}--\ref{tab:gmvip-ablation-B}
report the full per-dataset results as mean $\pm$ standard deviation over the
five seeds.  Lower values are better for RMSE, NLL, CRPS, and the centered
quantile metric (CQM).  Within each table and dataset, the best mean is shown
in purple and the second-best mean in teal.

\paragraph{Inducing geometry.} Table~\ref{tab:gmvip-ablation-initialization}.
The $k$-means initialization is the strongest choice for predictive quality.
It obtains the best mean RMSE on eight of nine datasets and the best NLL and
CRPS on seven.
The grid initialization exhibits a different trade-off: although its
predictive metrics are generally weaker, it achieves the best CQM on six datasets.  Thus,
$k$-means is preferable when predictive accuracy is primary, whereas the grid
initialization can improve marginal calibration.

\paragraph{Number of inducing locations.} Table~\ref{tab:gmvip-ablation-M}.
Increasing the inducing set improves predictive performance over the tested
range.  The largest setting, $M=128$, wins on eight datasets for RMSE and CRPS
and on seven for the NLL metric.
The improvement is already visible at $M=64$, while $M\in\{16,32\}$ generally
underfits relative to the larger sets.  Calibration does not follow the same
trend: $M=16$ obtains the best CQM on five datasets.  The results therefore
show a clear trade-off between accuracy and calibration rather than uniform
dominance as $M$ grows.

\paragraph{Prior bank size.} Table~\ref{tab:gmvip-ablation-B}.
Performance is not monotone in the bank size.  A moderate bank of $B=256$
provides the best aggregate predictive results, with five RMSE wins and four
wins for each of the NLL and CRPS metrics.  Increasing the bank to $B=512$
does not improve consistently across datasets.
Smaller banks occasionally improve CQM---$B=32$ and $B=64$ obtain four and
three calibration wins, respectively---but have substantially poorer
predictive metrics.  Overall, $B=256$ gives the most consistent predictive
performance among the evaluated bank sizes.

\definecolor{ablationBest}{HTML}{D81B60}
\definecolor{ablationSecond}{HTML}{008C95}

\begin{table*}[t]
\centering
\caption{Inducing-location initialization ablation on the nine UCI regression datasets ($\downarrow$). Entries are mean $\pm$ standard deviation across five seeds. The best mean per dataset is shown in \textcolor{ablationBest}{purple}; the second-best is shown in \textcolor{ablationSecond}{teal}. Naval RMSE and CRPS are reported in units of $10^{-4}$.}\label{tab:gmvip-ablation-initialization}
\resizebox{\textwidth}{!}{%
\begin{tabular}{@{}l*{9}{c}@{}}
\toprule
\textbf{RMSE} & \textbf{Boston} & \textbf{Concrete} & \textbf{Energy} & \textbf{Kin8nm} & \textbf{Naval ($\times 10^{-4}$)} & \textbf{Power} & \textbf{Protein} & \textbf{Wine red} & \textbf{Yacht} \\
Grid & $\textcolor{ablationSecond}{\mathbf{3.26 \pm 0.55}}$ & $4.75 \pm 0.55$ & $0.52 \pm 0.06$ & $0.071 \pm 0.001$ & $2.26 \pm 0.53$ & $3.82 \pm 0.12$ & $4.49 \pm 0.03$ & $\textcolor{ablationSecond}{\mathbf{0.62 \pm 0.03}}$ & $0.49 \pm 0.27$ \\
Random subset & $\textcolor{ablationBest}{\mathbf{3.22 \pm 0.85}}$ & $\textcolor{ablationSecond}{\mathbf{4.27 \pm 0.54}}$ & $\textcolor{ablationSecond}{\mathbf{0.51 \pm 0.05}}$ & $\textcolor{ablationSecond}{\mathbf{0.069 \pm 0.001}}$ & $\textcolor{ablationSecond}{\mathbf{1.45 \pm 0.34}}$ & $\textcolor{ablationSecond}{\mathbf{3.79 \pm 0.12}}$ & $\textcolor{ablationSecond}{\mathbf{4.32 \pm 0.04}}$ & $0.62 \pm 0.03$ & $\textcolor{ablationSecond}{\mathbf{0.41 \pm 0.15}}$ \\
$k$-means & $3.32 \pm 0.97$ & $\textcolor{ablationBest}{\mathbf{4.24 \pm 0.74}}$ & $\textcolor{ablationBest}{\mathbf{0.51 \pm 0.05}}$ & $\textcolor{ablationBest}{\mathbf{0.069 \pm 0.001}}$ & $\textcolor{ablationBest}{\mathbf{1.36 \pm 0.19}}$ & $\textcolor{ablationBest}{\mathbf{3.78 \pm 0.12}}$ & $\textcolor{ablationBest}{\mathbf{4.32 \pm 0.04}}$ & $\textcolor{ablationBest}{\mathbf{0.60 \pm 0.03}}$ & $\textcolor{ablationBest}{\mathbf{0.32 \pm 0.11}}$ \\
\midrule
\textbf{NLL} & \textbf{Boston} & \textbf{Concrete} & \textbf{Energy} & \textbf{Kin8nm} & \textbf{Naval} & \textbf{Power} & \textbf{Protein} & \textbf{Wine red} & \textbf{Yacht} \\
Grid & $\textcolor{ablationBest}{\mathbf{2.83 \pm 0.41}}$ & $2.96 \pm 0.12$ & $0.76 \pm 0.06$ & $-1.23 \pm 0.01$ & $-6.36 \pm 0.10$ & $2.76 \pm 0.03$ & $2.92 \pm 0.01$ & $\textcolor{ablationSecond}{\mathbf{0.98 \pm 0.07}}$ & $0.96 \pm 1.48$ \\
Random subset & $\textcolor{ablationSecond}{\mathbf{2.87 \pm 0.58}}$ & $\textcolor{ablationSecond}{\mathbf{2.84 \pm 0.12}}$ & $\textcolor{ablationBest}{\mathbf{0.74 \pm 0.07}}$ & $\textcolor{ablationSecond}{\mathbf{-1.25 \pm 0.02}}$ & $\textcolor{ablationSecond}{\mathbf{-7.08 \pm 0.06}}$ & $\textcolor{ablationSecond}{\mathbf{2.75 \pm 0.03}}$ & $\textcolor{ablationSecond}{\mathbf{2.88 \pm 0.01}}$ & $1.01 \pm 0.10$ & $\textcolor{ablationSecond}{\mathbf{0.35 \pm 0.32}}$ \\
$k$-means & $3.11 \pm 1.03$ & $\textcolor{ablationBest}{\mathbf{2.82 \pm 0.16}}$ & $\textcolor{ablationSecond}{\mathbf{0.75 \pm 0.09}}$ & $\textcolor{ablationBest}{\mathbf{-1.26 \pm 0.02}}$ & $\textcolor{ablationBest}{\mathbf{-7.12 \pm 0.04}}$ & $\textcolor{ablationBest}{\mathbf{2.75 \pm 0.03}}$ & $\textcolor{ablationBest}{\mathbf{2.88 \pm 0.01}}$ & $\textcolor{ablationBest}{\mathbf{0.94 \pm 0.08}}$ & $\textcolor{ablationBest}{\mathbf{0.18 \pm 0.22}}$ \\
\midrule
\textbf{CRPS} & \textbf{Boston} & \textbf{Concrete} & \textbf{Energy} & \textbf{Kin8nm} & \textbf{Naval ($\times 10^{-4}$)} & \textbf{Power} & \textbf{Protein} & \textbf{Wine red} & \textbf{Yacht} \\
Grid & $1.82 \pm 0.23$ & $2.56 \pm 0.26$ & $\textcolor{ablationBest}{\mathbf{0.28 \pm 0.02}}$ & $0.040 \pm 0.001$ & $1.84 \pm 0.25$ & $2.09 \pm 0.06$ & $2.52 \pm 0.02$ & $0.34 \pm 0.02$ & $0.21 \pm 0.08$ \\
Random subset & $\textcolor{ablationBest}{\mathbf{1.73 \pm 0.35}}$ & $\textcolor{ablationSecond}{\mathbf{2.29 \pm 0.27}}$ & $\textcolor{ablationBest}{\mathbf{0.28 \pm 0.02}}$ & $\textcolor{ablationSecond}{\mathbf{0.039 \pm 0.001}}$ & $\textcolor{ablationSecond}{\mathbf{0.96 \pm 0.13}}$ & $\textcolor{ablationSecond}{\mathbf{2.08 \pm 0.06}}$ & $\textcolor{ablationSecond}{\mathbf{2.41 \pm 0.02}}$ & $\textcolor{ablationSecond}{\mathbf{0.34 \pm 0.02}}$ & $\textcolor{ablationSecond}{\mathbf{0.18 \pm 0.05}}$ \\
$k$-means & $\textcolor{ablationSecond}{\mathbf{1.79 \pm 0.37}}$ & $\textcolor{ablationBest}{\mathbf{2.26 \pm 0.35}}$ & $\textcolor{ablationBest}{\mathbf{0.28 \pm 0.02}}$ & $\textcolor{ablationBest}{\mathbf{0.038 \pm 0.001}}$ & $\textcolor{ablationBest}{\mathbf{0.92 \pm 0.07}}$ & $\textcolor{ablationBest}{\mathbf{2.07 \pm 0.06}}$ & $\textcolor{ablationBest}{\mathbf{2.41 \pm 0.02}}$ & $\textcolor{ablationBest}{\mathbf{0.33 \pm 0.01}}$ & $\textcolor{ablationBest}{\mathbf{0.16 \pm 0.04}}$ \\
\midrule
\textbf{CQM} & \textbf{Boston} & \textbf{Concrete} & \textbf{Energy} & \textbf{Kin8nm} & \textbf{Naval} & \textbf{Power} & \textbf{Protein} & \textbf{Wine red} & \textbf{Yacht} \\
Grid & $\textcolor{ablationBest}{\mathbf{0.117 \pm 0.013}}$ & $\textcolor{ablationBest}{\mathbf{0.043 \pm 0.010}}$ & $0.085 \pm 0.041$ & $\textcolor{ablationBest}{\mathbf{0.031 \pm 0.011}}$ & $0.301 \pm 0.044$ & $\textcolor{ablationBest}{\mathbf{0.036 \pm 0.011}}$ & $\textcolor{ablationBest}{\mathbf{0.020 \pm 0.007}}$ & $\textcolor{ablationBest}{\mathbf{0.045 \pm 0.029}}$ & $\textcolor{ablationSecond}{\mathbf{0.194 \pm 0.048}}$ \\
Random subset & $0.122 \pm 0.060$ & $\textcolor{ablationSecond}{\mathbf{0.050 \pm 0.028}}$ & $\textcolor{ablationBest}{\mathbf{0.076 \pm 0.027}}$ & $\textcolor{ablationSecond}{\mathbf{0.034 \pm 0.010}}$ & $\textcolor{ablationBest}{\mathbf{0.232 \pm 0.068}}$ & $0.038 \pm 0.011$ & $0.033 \pm 0.005$ & $0.051 \pm 0.016$ & $\textcolor{ablationBest}{\mathbf{0.183 \pm 0.022}}$ \\
$k$-means & $\textcolor{ablationSecond}{\mathbf{0.119 \pm 0.065}}$ & $0.056 \pm 0.031$ & $\textcolor{ablationSecond}{\mathbf{0.077 \pm 0.026}}$ & $0.034 \pm 0.010$ & $\textcolor{ablationSecond}{\mathbf{0.244 \pm 0.044}}$ & $\textcolor{ablationSecond}{\mathbf{0.037 \pm 0.010}}$ & $\textcolor{ablationSecond}{\mathbf{0.031 \pm 0.005}}$ & $\textcolor{ablationSecond}{\mathbf{0.046 \pm 0.012}}$ & $0.198 \pm 0.053$ \\
\bottomrule
\end{tabular}%
}
\end{table*}

\begin{table*}[t]
\centering
\caption{Number of inducing locations ablation on the nine UCI regression datasets ($\downarrow$). Entries are mean $\pm$ standard deviation across five seeds. The best mean per dataset is shown in \textcolor{ablationBest}{purple}; the second-best is shown in \textcolor{ablationSecond}{teal}. Naval RMSE and CRPS are reported in units of $10^{-4}$.}\label{tab:gmvip-ablation-M}
\resizebox{\textwidth}{!}{%
\begin{tabular}{@{}l*{9}{c}@{}}
\toprule
\textbf{RMSE} & \textbf{Boston} & \textbf{Concrete} & \textbf{Energy} & \textbf{Kin8nm} & \textbf{Naval ($\times 10^{-4}$)} & \textbf{Power} & \textbf{Protein} & \textbf{Wine red} & \textbf{Yacht} \\
$M=16$ & $3.50 \pm 1.30$ & $4.75 \pm 0.44$ & $0.51 \pm 0.05$ & $0.076 \pm 0.002$ & $1.98 \pm 0.19$ & $3.94 \pm 0.12$ & $4.57 \pm 0.03$ & $\textcolor{ablationSecond}{\mathbf{0.61 \pm 0.03}}$ & $\textcolor{ablationSecond}{\mathbf{0.46 \pm 0.14}}$ \\
$M=32$ & $\textcolor{ablationSecond}{\mathbf{3.48 \pm 0.92}}$ & $\textcolor{ablationSecond}{\mathbf{4.39 \pm 0.54}}$ & $0.52 \pm 0.05$ & $0.072 \pm 0.001$ & $1.76 \pm 0.16$ & $3.86 \pm 0.14$ & $4.47 \pm 0.04$ & $0.61 \pm 0.03$ & $0.60 \pm 0.22$ \\
$M=64$ & $\textcolor{ablationBest}{\mathbf{3.24 \pm 1.34}}$ & $4.55 \pm 0.40$ & $\textcolor{ablationSecond}{\mathbf{0.51 \pm 0.05}}$ & $\textcolor{ablationSecond}{\mathbf{0.070 \pm 0.001}}$ & $\textcolor{ablationSecond}{\mathbf{1.56 \pm 0.46}}$ & $\textcolor{ablationSecond}{\mathbf{3.84 \pm 0.12}}$ & $\textcolor{ablationSecond}{\mathbf{4.37 \pm 0.04}}$ & $0.61 \pm 0.04$ & $0.58 \pm 0.36$ \\
$M=128$ & $3.51 \pm 1.19$ & $\textcolor{ablationBest}{\mathbf{4.27 \pm 0.73}}$ & $\textcolor{ablationBest}{\mathbf{0.51 \pm 0.05}}$ & $\textcolor{ablationBest}{\mathbf{0.068 \pm 0.001}}$ & $\textcolor{ablationBest}{\mathbf{1.54 \pm 0.38}}$ & $\textcolor{ablationBest}{\mathbf{3.76 \pm 0.12}}$ & $\textcolor{ablationBest}{\mathbf{4.29 \pm 0.05}}$ & $\textcolor{ablationBest}{\mathbf{0.60 \pm 0.02}}$ & $\textcolor{ablationBest}{\mathbf{0.40 \pm 0.24}}$ \\
\midrule
\textbf{NLL} & \textbf{Boston} & \textbf{Concrete} & \textbf{Energy} & \textbf{Kin8nm} & \textbf{Naval} & \textbf{Power} & \textbf{Protein} & \textbf{Wine red} & \textbf{Yacht} \\
$M=16$ & $\textcolor{ablationSecond}{\mathbf{3.40 \pm 1.77}}$ & $2.94 \pm 0.11$ & $0.75 \pm 0.05$ & $-1.16 \pm 0.03$ & $-6.54 \pm 0.06$ & $2.79 \pm 0.03$ & $2.94 \pm 0.01$ & $\textcolor{ablationSecond}{\mathbf{0.95 \pm 0.07}}$ & $\textcolor{ablationBest}{\mathbf{0.46 \pm 0.16}}$ \\
$M=32$ & $\textcolor{ablationBest}{\mathbf{3.31 \pm 1.18}}$ & $\textcolor{ablationSecond}{\mathbf{2.86 \pm 0.11}}$ & $0.75 \pm 0.06$ & $-1.21 \pm 0.01$ & $-6.78 \pm 0.06$ & $2.77 \pm 0.03$ & $2.92 \pm 0.01$ & $0.98 \pm 0.06$ & $0.66 \pm 0.42$ \\
$M=64$ & $3.47 \pm 2.29$ & $2.91 \pm 0.10$ & $\textcolor{ablationSecond}{\mathbf{0.74 \pm 0.06}}$ & $\textcolor{ablationSecond}{\mathbf{-1.24 \pm 0.01}}$ & $\textcolor{ablationSecond}{\mathbf{-6.99 \pm 0.08}}$ & $\textcolor{ablationSecond}{\mathbf{2.76 \pm 0.03}}$ & $\textcolor{ablationSecond}{\mathbf{2.89 \pm 0.01}}$ & $0.99 \pm 0.09$ & $0.60 \pm 0.55$ \\
$M=128$ & $3.48 \pm 1.72$ & $\textcolor{ablationBest}{\mathbf{2.85 \pm 0.16}}$ & $\textcolor{ablationBest}{\mathbf{0.74 \pm 0.08}}$ & $\textcolor{ablationBest}{\mathbf{-1.28 \pm 0.01}}$ & $\textcolor{ablationBest}{\mathbf{-7.11 \pm 0.09}}$ & $\textcolor{ablationBest}{\mathbf{2.74 \pm 0.03}}$ & $\textcolor{ablationBest}{\mathbf{2.87 \pm 0.01}}$ & $\textcolor{ablationBest}{\mathbf{0.95 \pm 0.05}}$ & $\textcolor{ablationSecond}{\mathbf{0.46 \pm 0.86}}$ \\
\midrule
\textbf{CRPS} & \textbf{Boston} & \textbf{Concrete} & \textbf{Energy} & \textbf{Kin8nm} & \textbf{Naval ($\times 10^{-4}$)} & \textbf{Power} & \textbf{Protein} & \textbf{Wine red} & \textbf{Yacht} \\
$M=16$ & $1.90 \pm 0.66$ & $2.52 \pm 0.23$ & $0.28 \pm 0.02$ & $0.042 \pm 0.001$ & $1.55 \pm 0.10$ & $2.17 \pm 0.06$ & $2.56 \pm 0.01$ & $\textcolor{ablationSecond}{\mathbf{0.34 \pm 0.01}}$ & $\textcolor{ablationSecond}{\mathbf{0.22 \pm 0.04}}$ \\
$M=32$ & $1.91 \pm 0.36$ & $\textcolor{ablationSecond}{\mathbf{2.34 \pm 0.24}}$ & $0.28 \pm 0.02$ & $0.040 \pm 0.001$ & $1.25 \pm 0.08$ & $2.13 \pm 0.07$ & $2.51 \pm 0.02$ & $0.34 \pm 0.02$ & $0.25 \pm 0.07$ \\
$M=64$ & $\textcolor{ablationBest}{\mathbf{1.65 \pm 0.43}}$ & $2.42 \pm 0.19$ & $\textcolor{ablationSecond}{\mathbf{0.28 \pm 0.02}}$ & $\textcolor{ablationSecond}{\mathbf{0.039 \pm 0.000}}$ & $\textcolor{ablationSecond}{\mathbf{1.05 \pm 0.17}}$ & $\textcolor{ablationSecond}{\mathbf{2.11 \pm 0.06}}$ & $\textcolor{ablationSecond}{\mathbf{2.44 \pm 0.02}}$ & $0.34 \pm 0.02$ & $0.24 \pm 0.11$ \\
$M=128$ & $\textcolor{ablationSecond}{\mathbf{1.85 \pm 0.41}}$ & $\textcolor{ablationBest}{\mathbf{2.31 \pm 0.35}}$ & $\textcolor{ablationBest}{\mathbf{0.28 \pm 0.02}}$ & $\textcolor{ablationBest}{\mathbf{0.038 \pm 0.000}}$ & $\textcolor{ablationBest}{\mathbf{0.97 \pm 0.16}}$ & $\textcolor{ablationBest}{\mathbf{2.06 \pm 0.05}}$ & $\textcolor{ablationBest}{\mathbf{2.39 \pm 0.03}}$ & $\textcolor{ablationBest}{\mathbf{0.33 \pm 0.01}}$ & $\textcolor{ablationBest}{\mathbf{0.17 \pm 0.07}}$ \\
\midrule
\textbf{CQM} & \textbf{Boston} & \textbf{Concrete} & \textbf{Energy} & \textbf{Kin8nm} & \textbf{Naval} & \textbf{Power} & \textbf{Protein} & \textbf{Wine red} & \textbf{Yacht} \\
$M=16$ & $\textcolor{ablationSecond}{\mathbf{0.133 \pm 0.071}}$ & $0.047 \pm 0.017$ & $0.085 \pm 0.031$ & $\textcolor{ablationBest}{\mathbf{0.028 \pm 0.006}}$ & $0.294 \pm 0.014$ & $\textcolor{ablationBest}{\mathbf{0.027 \pm 0.008}}$ & $\textcolor{ablationBest}{\mathbf{0.017 \pm 0.005}}$ & $\textcolor{ablationBest}{\mathbf{0.041 \pm 0.009}}$ & $\textcolor{ablationBest}{\mathbf{0.168 \pm 0.034}}$ \\
$M=32$ & $0.161 \pm 0.062$ & $0.050 \pm 0.023$ & $0.074 \pm 0.027$ & $\textcolor{ablationSecond}{\mathbf{0.034 \pm 0.009}}$ & $0.262 \pm 0.022$ & $\textcolor{ablationSecond}{\mathbf{0.032 \pm 0.013}}$ & $\textcolor{ablationSecond}{\mathbf{0.020 \pm 0.007}}$ & $0.047 \pm 0.010$ & $0.186 \pm 0.037$ \\
$M=64$ & $\textcolor{ablationBest}{\mathbf{0.107 \pm 0.034}}$ & $\textcolor{ablationBest}{\mathbf{0.043 \pm 0.019}}$ & $\textcolor{ablationBest}{\mathbf{0.070 \pm 0.029}}$ & $0.034 \pm 0.009$ & $\textcolor{ablationSecond}{\mathbf{0.258 \pm 0.058}}$ & $0.034 \pm 0.010$ & $0.028 \pm 0.007$ & $0.055 \pm 0.023$ & $\textcolor{ablationSecond}{\mathbf{0.184 \pm 0.043}}$ \\
$M=128$ & $0.138 \pm 0.055$ & $\textcolor{ablationSecond}{\mathbf{0.044 \pm 0.029}}$ & $\textcolor{ablationSecond}{\mathbf{0.071 \pm 0.026}}$ & $0.036 \pm 0.010$ & $\textcolor{ablationBest}{\mathbf{0.203 \pm 0.053}}$ & $0.038 \pm 0.010$ & $0.032 \pm 0.004$ & $\textcolor{ablationSecond}{\mathbf{0.045 \pm 0.011}}$ & $0.207 \pm 0.020$ \\
\bottomrule
\end{tabular}%
}
\end{table*}

\begin{table*}[t]
\centering
\caption{Prior-bank size ablation on the nine UCI regression datasets ($\downarrow$). Entries are mean $\pm$ standard deviation across five seeds. The best mean per dataset is shown in \textcolor{ablationBest}{purple}; the second-best is shown in \textcolor{ablationSecond}{teal}. Naval RMSE and CRPS are reported in units of $10^{-4}$.}\label{tab:gmvip-ablation-B}
\resizebox{\textwidth}{!}{%
\begin{tabular}{@{}l*{9}{c}@{}}
\toprule
\textbf{RMSE} & \textbf{Boston} & \textbf{Concrete} & \textbf{Energy} & \textbf{Kin8nm} & \textbf{Naval ($\times 10^{-4}$)} & \textbf{Power} & \textbf{Protein} & \textbf{Wine red} & \textbf{Yacht} \\
$B=32$ & $4.15 \pm 1.28$ & $4.47 \pm 0.89$ & $0.51 \pm 0.06$ & $0.073 \pm 0.003$ & $\textcolor{ablationBest}{\mathbf{1.29 \pm 0.17}}$ & $3.81 \pm 0.12$ & $4.36 \pm 0.05$ & $0.68 \pm 0.01$ & $0.39 \pm 0.17$ \\
$B=64$ & $3.71 \pm 1.17$ & $\textcolor{ablationSecond}{\mathbf{4.37 \pm 0.47}}$ & $0.51 \pm 0.05$ & $0.071 \pm 0.001$ & $\textcolor{ablationSecond}{\mathbf{1.32 \pm 0.20}}$ & $\textcolor{ablationSecond}{\mathbf{3.78 \pm 0.12}}$ & $4.33 \pm 0.05$ & $0.65 \pm 0.02$ & $0.44 \pm 0.22$ \\
$B=128$ & $3.86 \pm 1.25$ & $4.37 \pm 0.70$ & $0.52 \pm 0.04$ & $0.070 \pm 0.002$ & $1.39 \pm 0.18$ & $3.78 \pm 0.14$ & $\textcolor{ablationSecond}{\mathbf{4.31 \pm 0.05}}$ & $0.65 \pm 0.01$ & $0.76 \pm 0.88$ \\
$B=256$ & $\textcolor{ablationSecond}{\mathbf{3.68 \pm 1.35}}$ & $4.40 \pm 0.60$ & $\textcolor{ablationBest}{\mathbf{0.51 \pm 0.05}}$ & $\textcolor{ablationBest}{\mathbf{0.069 \pm 0.001}}$ & $1.52 \pm 0.19$ & $\textcolor{ablationBest}{\mathbf{3.77 \pm 0.13}}$ & $\textcolor{ablationBest}{\mathbf{4.31 \pm 0.04}}$ & $\textcolor{ablationSecond}{\mathbf{0.64 \pm 0.03}}$ & $\textcolor{ablationSecond}{\mathbf{0.33 \pm 0.12}}$ \\
$B=512$ & $\textcolor{ablationBest}{\mathbf{3.32 \pm 0.97}}$ & $\textcolor{ablationBest}{\mathbf{4.24 \pm 0.74}}$ & $\textcolor{ablationSecond}{\mathbf{0.51 \pm 0.05}}$ & $\textcolor{ablationSecond}{\mathbf{0.069 \pm 0.001}}$ & $1.36 \pm 0.19$ & $3.78 \pm 0.12$ & $4.32 \pm 0.04$ & $\textcolor{ablationBest}{\mathbf{0.60 \pm 0.03}}$ & $\textcolor{ablationBest}{\mathbf{0.32 \pm 0.11}}$ \\
\midrule
\textbf{NLL} & \textbf{Boston} & \textbf{Concrete} & \textbf{Energy} & \textbf{Kin8nm} & \textbf{Naval} & \textbf{Power} & \textbf{Protein} & \textbf{Wine red} & \textbf{Yacht} \\
$B=32$ & $6.30 \pm 4.87$ & $2.94 \pm 0.26$ & $\textcolor{ablationBest}{\mathbf{0.72 \pm 0.10}}$ & $-1.21 \pm 0.03$ & $\textcolor{ablationBest}{\mathbf{-7.29 \pm 0.08}}$ & $2.76 \pm 0.03$ & $2.89 \pm 0.01$ & $1.34 \pm 0.07$ & $0.31 \pm 0.31$ \\
$B=64$ & $4.52 \pm 2.89$ & $2.86 \pm 0.11$ & $\textcolor{ablationSecond}{\mathbf{0.72 \pm 0.08}}$ & $-1.23 \pm 0.02$ & $\textcolor{ablationSecond}{\mathbf{-7.26 \pm 0.06}}$ & $2.75 \pm 0.03$ & $2.88 \pm 0.01$ & $1.21 \pm 0.05$ & $0.49 \pm 0.20$ \\
$B=128$ & $4.36 \pm 2.58$ & $2.88 \pm 0.18$ & $0.75 \pm 0.06$ & $-1.24 \pm 0.03$ & $-7.20 \pm 0.03$ & $\textcolor{ablationSecond}{\mathbf{2.75 \pm 0.04}}$ & $\textcolor{ablationSecond}{\mathbf{2.88 \pm 0.01}}$ & $1.16 \pm 0.02$ & $0.41 \pm 0.62$ \\
$B=256$ & $\textcolor{ablationSecond}{\mathbf{3.75 \pm 1.84}}$ & $\textcolor{ablationSecond}{\mathbf{2.86 \pm 0.13}}$ & $0.73 \pm 0.07$ & $\textcolor{ablationBest}{\mathbf{-1.26 \pm 0.01}}$ & $-7.13 \pm 0.06$ & $\textcolor{ablationBest}{\mathbf{2.75 \pm 0.03}}$ & $\textcolor{ablationBest}{\mathbf{2.88 \pm 0.01}}$ & $\textcolor{ablationSecond}{\mathbf{1.09 \pm 0.08}}$ & $\textcolor{ablationSecond}{\mathbf{0.19 \pm 0.31}}$ \\
$B=512$ & $\textcolor{ablationBest}{\mathbf{3.11 \pm 1.03}}$ & $\textcolor{ablationBest}{\mathbf{2.82 \pm 0.16}}$ & $0.75 \pm 0.09$ & $\textcolor{ablationSecond}{\mathbf{-1.26 \pm 0.02}}$ & $-7.12 \pm 0.04$ & $2.75 \pm 0.03$ & $2.88 \pm 0.01$ & $\textcolor{ablationBest}{\mathbf{0.94 \pm 0.08}}$ & $\textcolor{ablationBest}{\mathbf{0.18 \pm 0.22}}$ \\
\midrule
\textbf{CRPS} & \textbf{Boston} & \textbf{Concrete} & \textbf{Energy} & \textbf{Kin8nm} & \textbf{Naval ($\times 10^{-4}$)} & \textbf{Power} & \textbf{Protein} & \textbf{Wine red} & \textbf{Yacht} \\
$B=32$ & $2.32 \pm 0.52$ & $2.38 \pm 0.38$ & $\textcolor{ablationBest}{\mathbf{0.27 \pm 0.03}}$ & $0.041 \pm 0.001$ & $\textcolor{ablationBest}{\mathbf{0.81 \pm 0.08}}$ & $2.09 \pm 0.05$ & $2.43 \pm 0.03$ & $0.38 \pm 0.01$ & $0.19 \pm 0.07$ \\
$B=64$ & $\textcolor{ablationSecond}{\mathbf{2.00 \pm 0.40}}$ & $2.33 \pm 0.25$ & $0.28 \pm 0.02$ & $0.039 \pm 0.001$ & $\textcolor{ablationSecond}{\mathbf{0.83 \pm 0.08}}$ & $\textcolor{ablationSecond}{\mathbf{2.07 \pm 0.06}}$ & $2.41 \pm 0.03$ & $0.36 \pm 0.00$ & $0.20 \pm 0.06$ \\
$B=128$ & $2.09 \pm 0.54$ & $\textcolor{ablationSecond}{\mathbf{2.31 \pm 0.31}}$ & $0.28 \pm 0.02$ & $0.039 \pm 0.001$ & $0.88 \pm 0.06$ & $2.07 \pm 0.07$ & $\textcolor{ablationSecond}{\mathbf{2.40 \pm 0.02}}$ & $0.36 \pm 0.01$ & $0.28 \pm 0.24$ \\
$B=256$ & $2.01 \pm 0.47$ & $2.33 \pm 0.29$ & $\textcolor{ablationSecond}{\mathbf{0.28 \pm 0.02}}$ & $\textcolor{ablationBest}{\mathbf{0.038 \pm 0.000}}$ & $0.96 \pm 0.08$ & $\textcolor{ablationBest}{\mathbf{2.07 \pm 0.06}}$ & $\textcolor{ablationBest}{\mathbf{2.40 \pm 0.03}}$ & $\textcolor{ablationSecond}{\mathbf{0.35 \pm 0.02}}$ & $\textcolor{ablationBest}{\mathbf{0.15 \pm 0.03}}$ \\
$B=512$ & $\textcolor{ablationBest}{\mathbf{1.79 \pm 0.37}}$ & $\textcolor{ablationBest}{\mathbf{2.26 \pm 0.35}}$ & $0.28 \pm 0.02$ & $\textcolor{ablationSecond}{\mathbf{0.038 \pm 0.001}}$ & $0.92 \pm 0.07$ & $2.07 \pm 0.06$ & $2.41 \pm 0.02$ & $\textcolor{ablationBest}{\mathbf{0.33 \pm 0.01}}$ & $\textcolor{ablationSecond}{\mathbf{0.16 \pm 0.04}}$ \\
\midrule
\textbf{CQM} & \textbf{Boston} & \textbf{Concrete} & \textbf{Energy} & \textbf{Kin8nm} & \textbf{Naval} & \textbf{Power} & \textbf{Protein} & \textbf{Wine red} & \textbf{Yacht} \\
$B=32$ & $0.210 \pm 0.062$ & $0.046 \pm 0.012$ & $\textcolor{ablationBest}{\mathbf{0.056 \pm 0.028}}$ & $\textcolor{ablationSecond}{\mathbf{0.024 \pm 0.009}}$ & $\textcolor{ablationBest}{\mathbf{0.200 \pm 0.029}}$ & $\textcolor{ablationBest}{\mathbf{0.033 \pm 0.011}}$ & $\textcolor{ablationSecond}{\mathbf{0.029 \pm 0.008}}$ & $0.104 \pm 0.026$ & $\textcolor{ablationBest}{\mathbf{0.171 \pm 0.064}}$ \\
$B=64$ & $0.184 \pm 0.041$ & $\textcolor{ablationBest}{\mathbf{0.033 \pm 0.010}}$ & $0.067 \pm 0.027$ & $\textcolor{ablationBest}{\mathbf{0.024 \pm 0.007}}$ & $0.205 \pm 0.036$ & $\textcolor{ablationSecond}{\mathbf{0.034 \pm 0.010}}$ & $\textcolor{ablationBest}{\mathbf{0.028 \pm 0.007}}$ & $0.076 \pm 0.012$ & $0.201 \pm 0.061$ \\
$B=128$ & $0.166 \pm 0.052$ & $0.042 \pm 0.009$ & $\textcolor{ablationSecond}{\mathbf{0.066 \pm 0.026}}$ & $0.024 \pm 0.012$ & $0.218 \pm 0.039$ & $0.036 \pm 0.010$ & $0.031 \pm 0.007$ & $0.082 \pm 0.009$ & $0.219 \pm 0.040$ \\
$B=256$ & $\textcolor{ablationSecond}{\mathbf{0.163 \pm 0.021}}$ & $\textcolor{ablationSecond}{\mathbf{0.038 \pm 0.015}}$ & $0.075 \pm 0.025$ & $0.030 \pm 0.009$ & $\textcolor{ablationSecond}{\mathbf{0.204 \pm 0.038}}$ & $0.037 \pm 0.011$ & $0.032 \pm 0.007$ & $\textcolor{ablationSecond}{\mathbf{0.063 \pm 0.023}}$ & $\textcolor{ablationSecond}{\mathbf{0.194 \pm 0.032}}$ \\
$B=512$ & $\textcolor{ablationBest}{\mathbf{0.119 \pm 0.065}}$ & $0.056 \pm 0.031$ & $0.077 \pm 0.026$ & $0.034 \pm 0.010$ & $0.244 \pm 0.044$ & $0.037 \pm 0.010$ & $0.031 \pm 0.005$ & $\textcolor{ablationBest}{\mathbf{0.046 \pm 0.012}}$ & $0.198 \pm 0.053$ \\
\bottomrule
\end{tabular}%
}
\end{table*}

\paragraph{Residual-path ablation.}
We compare full GMVIP with an inducing-only variant that retains the same
empirical mean, inducing covariance, replacement operator, variational
family, likelihood, and optimization protocol, but removes the coherent
residual prior path.  Let
\(\Z=(\vz_1,\ldots,\vz_M)^\top\) denote the inducing matrix and let
\[
g(\Z)
=
\bigl(g(\vz_1),\ldots,g(\vz_M)\bigr)^\top.
\]
For an input \(\vx\), the residual path is
\[
\vr_g(\vx)
=
g(\vx)-\hat\vmu(\vx)
-
\Psib_{\Z}(\vx)
\left(
g(\Z)-\hat\vmu(\Z)
\right).
\]
The full and inducing-only models are therefore
\[
f_{\mathrm{full}}(\vx)
=
\hat\vmu(\vx)
+
\Psib_{\Z}(\vx)\Dmat\va
+
\vr_g(\vx),
\qquad
f_{\mathrm{ind}}(\vx)
=
\hat\vmu(\vx)
+
\Psib_{\Z}(\vx)\Dmat\va,
\]
where \(\va\) is independent of \(g\) and
\[
q_{\phi}(\va)
=
\N\!\left(\vm_a,\Smat_a\right),
\qquad
\Dmat\Dmat^{\top}
=
\hat\K(\Z,\Z).
\]
Consequently,
\[
\Cov\!\left[\vf_{\mathrm{full}}(\X)\right]
=
\Cov\!\left[\vf_{\mathrm{ind}}(\X)\right]
+
\Cov\!\left[\vr_g(\X)\right].
\]
The inducing-only model is thus restricted to a rank-\(M\) stochastic
representation and omits the residual uncertainty outside the inducing
subspace.

The test results in Table~\ref{tab:uci-gmvip-path-ablation} show that this
restriction is adequate on some datasets:
inducing-only GMVIP improves all four metrics on Naval, Power, and Protein.
Their train--test NLL gaps also remain small; for example, inducing-only NLL
changes from \(-7.59\) on the Naval training set to \(-7.58\) on the test set,
and from \(2.82\) to \(2.86\) on Protein. On these datasets, the learned
inducing subspace appears sufficient to represent both the predictive mean
and uncertainty.

In contrast, the inducing-only model exhibits pronounced overfitting on
Boston, Concrete, Wine Red, and Yacht. On Boston, it achieves a training RMSE
of \(0.05\) and training NLL of \(-1.38\), but its test RMSE and NLL deteriorate
to \(4.09\) and \(2083.56\), respectively. Full GMVIP has a substantially
smaller discrepancy, with train/test RMSE \(0.84/3.32\) and train/test NLL
\(1.55/3.11\). Similar behavior occurs on Wine Red, where inducing-only
train/test NLL changes from \(-0.87\) to \(26.66\), and on Yacht, where it
changes from \(-1.18\) to \(6.16\). On Concrete, inducing-only fits the
training data more closely than full GMVIP---training RMSE \(1.66\) versus
\(2.75\)---but generalizes substantially worse, with test RMSE \(5.39\)
versus \(4.24\).

The overfitting cases also have much larger coefficient-space
regularization terms. The final inducing-coefficient KL increases from
\(176\) to \(481\) on Boston, from \(237\) to \(417\) on Wine Red, and from
\(117\) to \(253\) on Yacht when the residual path is removed. This indicates
that the inducing-only posterior moves further from its prior in order to
explain variation that would otherwise be carried by \(r_g(\vx)\).
Together with the omitted residual covariance, this produces excessively
concentrated predictive distributions: point errors increase moderately,
whereas test NLL can deteriorate by several orders of magnitude. Overall,
the residual path acts both as an additional source of predictive
uncertainty and as a structural regularizer against overfitting the
finite-dimensional inducing representation.

The residual path is not uniformly beneficial, but on several datasets it
acts as a structural regularizer that prevents severe overconfidence and
poor test likelihood.

\begin{table*}[t]
\centering
\scriptsize
\setlength{\tabcolsep}{3.5pt}
\renewcommand{\arraystretch}{0.92}
\caption{GMVIP path ablation on UCI regression ($\downarrow$). Entries are mean $\pm$ standard deviation across five seeds. The full-path model retains the residual prior path, whereas the inducing-only variant removes it. The best mean per dataset is shown in bold. Naval RMSE and CRPS are reported in units of $10^{-4}$.}
\label{tab:uci-gmvip-path-ablation}
\resizebox{\textwidth}{!}{%
\begin{tabular}{lccccccccc}
\toprule
\textbf{RMSE} & Boston & Concrete & Energy & Kin8nm & Naval ($\times 10^{-4}$) & Power & Protein & Wine Red & Yacht \\
\midrule
Full & $\bm{3.32 \pm 0.97}$ & $\bm{4.24 \pm 0.74}$ & $0.51 \pm 0.05$ & $\bm{0.069 \pm 0.001}$ & $1.36 \pm 0.19$ & $3.78 \pm 0.12$ & $4.32 \pm 0.04$ & $\bm{0.60 \pm 0.03}$ & $\bm{0.32 \pm 0.11}$ \\
Induc. only & $4.09 \pm 1.10$ & $5.39 \pm 0.65$ & $\bm{0.45 \pm 0.03}$ & $0.085 \pm 0.003$ & $\bm{1.16 \pm 0.25}$ & $\bm{3.72 \pm 0.08}$ & $\bm{4.22 \pm 0.05}$ & $0.87 \pm 0.05$ & $0.46 \pm 0.15$ \\
\midrule
\textbf{NLL} & Boston & Concrete & Energy & Kin8nm & Naval & Power & Protein & Wine Red & Yacht \\
\midrule
Full & $\bm{3.11 \pm 1.03}$ & $\bm{2.82 \pm 0.16}$ & $\bm{0.75 \pm 0.09}$ & $\bm{-1.26 \pm 0.02}$ & $-7.12 \pm 0.04$ & $2.75 \pm 0.03$ & $2.88 \pm 0.01$ & $\bm{0.94 \pm 0.08}$ & $\bm{0.18 \pm 0.22}$ \\
Induc. only & $2083.56 \pm 1408.57$ & $5.57 \pm 1.00$ & $1.49 \pm 0.43$ & $-1.02 \pm 0.04$ & $\bm{-7.58 \pm 0.12}$ & $\bm{2.73 \pm 0.02}$ & $\bm{2.86 \pm 0.01}$ & $26.66 \pm 2.95$ & $6.16 \pm 5.43$ \\
\midrule
\textbf{CRPS} & Boston & Concrete & Energy & Kin8nm & Naval ($\times 10^{-4}$) & Power & Protein & Wine Red & Yacht \\
\midrule
Full & $\bm{1.79 \pm 0.37}$ & $\bm{2.26 \pm 0.35}$ & $0.28 \pm 0.02$ & $\bm{0.038 \pm 0.001}$ & $0.92 \pm 0.07$ & $2.07 \pm 0.06$ & $2.41 \pm 0.02$ & $\bm{0.33 \pm 0.01}$ & $\bm{0.16 \pm 0.04}$ \\
Induc. only & $2.86 \pm 0.40$ & $3.06 \pm 0.27$ & $\bm{0.24 \pm 0.02}$ & $0.048 \pm 0.002$ & $\bm{0.64 \pm 0.11}$ & $\bm{2.02 \pm 0.04}$ & $\bm{2.33 \pm 0.03}$ & $0.56 \pm 0.03$ & $0.20 \pm 0.06$ \\
\midrule
\textbf{CQM} & Boston & Concrete & Energy & Kin8nm & Naval & Power & Protein & Wine Red & Yacht \\
\midrule
Full & $\bm{0.119 \pm 0.065}$ & $\bm{0.056 \pm 0.031}$ & $\bm{0.077 \pm 0.026}$ & $\bm{0.034 \pm 0.010}$ & $0.244 \pm 0.044$ & $0.037 \pm 0.010$ & $0.031 \pm 0.005$ & $\bm{0.046 \pm 0.012}$ & $0.198 \pm 0.053$ \\
Induc. only & $0.483 \pm 0.009$ & $0.167 \pm 0.011$ & $0.139 \pm 0.037$ & $0.040 \pm 0.011$ & $\bm{0.166 \pm 0.041}$ & $\bm{0.021 \pm 0.008}$ & $\bm{0.029 \pm 0.004}$ & $0.302 \pm 0.012$ & $\bm{0.109 \pm 0.024}$ \\
\bottomrule
\end{tabular}
}
\end{table*}

\paragraph{Frozen-prior ablation.}
We study whether learning the implicit BNN prior is necessary for VIP, FTIP,
and GMVIP. For each method, we compare a tunable-prior variant with a matched
frozen-prior variant. If
$g_{\vtheta}(\vx)$ denotes the BNN prior, the two treatments differ
only in whether $\vtheta$ is optimized:
\[
    \vtheta=\vtheta_0
    \quad\text{(frozen prior)},
    \qquad
    (\vtheta,\vvartheta)
    \in \arg\max_{\vtheta,\vvartheta}
    \mathcal{L}(\vtheta,\vvartheta)
    \quad\text{(tunable prior)},
\]
where $\vvartheta$ collects the remaining variational, likelihood, and,
when applicable, inducing parameters. All other settings are matched,
including data splits, model architecture, inducing configuration, Monte
Carlo budgets, optimization schedule, and seeds. Table~\ref{tab:uci-frozen-prior-ablation}
reports RMSE, NLL, CRPS, and CQM over the nine UCI datasets.

Prior learning is especially important for GMVIP. As shown in
Table~\ref{tab:uci-frozen-prior-ablation}, tunable-prior GMVIP outperforms its
frozen counterpart on every dataset and every reported metric. The difference
is substantial on Concrete, where RMSE increases from $4.24$ to $7.34$ when
the prior is frozen, and on Energy, where it increases from $0.51$ to $2.80$.
The largest degradation occurs on Yacht: RMSE increases from $0.32$ to $4.11$
and NLL from $0.18$ to $3.28$. These results indicate that the empirical
GMVIP operator cannot, in general, compensate for a poorly adapted function
prior merely by learning the inducing locations and coefficient posterior.

The behavior of VIP and FTIP is more nuanced. Their tunable-prior variants
obtain lower test RMSE on eight of the nine datasets; the only exception is
Wine Red, where freezing the prior reduces RMSE from $0.79$ to $0.70$ for VIP
and from $0.82$ to $0.70$ for FTIP. The final training metrics nevertheless
show clear overfitting on several datasets. On Boston, VIP attains
train/test RMSEs of $0.055/6.49$ and train/test NLLs of $-1.41/4245.41$;
FTIP exhibits the same failure, with RMSEs of $0.135/6.68$ and NLLs of
$-0.40/1789.45$. On Yacht, the corresponding train/test gaps are
$0.045/0.84$ in RMSE and $-1.62/96.61$ in NLL for VIP, and $0.100/1.17$ and
$-0.84/56.66$ for FTIP. Wine Red shows a milder version of the same pattern.
Thus, the extreme test NLLs are accompanied by near-interpolation of the
training data and very favorable training NLLs. Learning the prior can make
VIP and FTIP severely under-dispersed: the predictive distribution contracts
around the training observations and assigns negligible mass to modest
held-out residuals.

Freezing the prior removes this train--test discrepancy, but typically by
underfitting both splits. For example, frozen VIP has train/test RMSEs of
$6.49/6.70$ on Boston and $11.73/10.79$ on Yacht, with nearly identical FTIP
values. Consequently, freezing reduces the Yacht test NLL to $3.83$ for VIP
and $3.84$ for FTIP, but increases test RMSE to $10.79$ and $10.80$,
respectively. The frozen variants also obtain lower CQM on eight datasets.
Their improved NLL and CQM should therefore not be read as uniformly better
predictions: they often reflect much more diffuse distributions with poor
point accuracy. Table~\ref{tab:uci-frozen-prior-ablation} exposes a sharp
trade-off for VIP and FTIP between overfit, overconfident tunable priors and
underfit, overly diffuse frozen priors, motivating the joint interpretation
of NLL and CQM with RMSE and CRPS.

Finally, frozen VIP and frozen FTIP produce nearly identical results: across
the matched runs, their mean absolute differences are only $0.013$ in RMSE
and $0.009$ in NLL. Thus, under the frozen BNN prior used here, the additional
FTIP flow provides little measurable benefit over VIP. Overall,
Table~\ref{tab:uci-frozen-prior-ablation} shows that prior adaptation is
essential for GMVIP and generally important for the point-prediction accuracy
of VIP and FTIP.

\begin{table*}[t]
\centering
\scriptsize
\setlength{\tabcolsep}{3.5pt}
\renewcommand{\arraystretch}{0.92}
\caption{Frozen-prior ablation for VIP, FTIP, and GMVIP on UCI regression ($\downarrow$). Entries are mean $\pm$ standard deviation across five matched seeds. Tunable variants optimize the BNN prior jointly with the remaining model parameters, whereas frozen variants keep the prior fixed at initialization. The best mean per dataset among the compared variants is shown in bold. Naval RMSE and CRPS are reported in units of $10^{-4}$.}
\label{tab:uci-frozen-prior-ablation}
\resizebox{\textwidth}{!}{%
\begin{tabular}{lccccccccc}
\toprule
\textbf{RMSE} & Boston & Concrete & Energy & Kin8nm & Naval ($\times 10^{-4}$) & Power & Protein & Wine Red & Yacht \\
\midrule
VIP (tunable) & $6.49 \pm 1.27$ & $5.76 \pm 0.97$ & $0.58 \pm 0.09$ & $0.080 \pm 0.003$ & $2.59 \pm 0.96$ & $\bm{3.76 \pm 0.07}$ & $4.37 \pm 0.06$ & $0.79 \pm 0.05$ & $0.84 \pm 0.30$ \\
VIP (frozen) & $6.70 \pm 1.26$ & $11.79 \pm 0.90$ & $4.97 \pm 0.59$ & $0.234 \pm 0.007$ & $96.20 \pm 8.64$ & $7.01 \pm 0.94$ & $5.52 \pm 0.05$ & $0.70 \pm 0.04$ & $10.79 \pm 2.21$ \\
FTIP (tunable) & $6.68 \pm 1.58$ & $6.02 \pm 1.08$ & $0.55 \pm 0.07$ & $0.079 \pm 0.002$ & $3.36 \pm 1.58$ & $3.77 \pm 0.11$ & $4.36 \pm 0.05$ & $0.82 \pm 0.03$ & $1.17 \pm 0.52$ \\
FTIP (frozen) & $6.71 \pm 1.27$ & $11.77 \pm 0.86$ & $4.97 \pm 0.60$ & $0.234 \pm 0.006$ & $90.54 \pm 9.39$ & $7.02 \pm 0.93$ & $5.53 \pm 0.06$ & $0.70 \pm 0.04$ & $10.80 \pm 2.22$ \\
GMVIP (tunable) & $\bm{3.32 \pm 0.97}$ & $\bm{4.24 \pm 0.74}$ & $\bm{0.51 \pm 0.05}$ & $\bm{0.069 \pm 0.001}$ & $\bm{1.36 \pm 0.19}$ & $3.78 \pm 0.12$ & $\bm{4.32 \pm 0.04}$ & $\bm{0.60 \pm 0.03}$ & $\bm{0.32 \pm 0.11}$ \\
GMVIP (frozen) & $3.86 \pm 0.97$ & $7.34 \pm 0.45$ & $2.80 \pm 0.13$ & $0.144 \pm 0.005$ & $4.10 \pm 0.56$ & $4.01 \pm 0.11$ & $4.64 \pm 0.05$ & $0.60 \pm 0.03$ & $4.11 \pm 1.07$ \\
\midrule
\textbf{NLL} & Boston & Concrete & Energy & Kin8nm & Naval & Power & Protein & Wine Red & Yacht \\
\midrule
VIP (tunable) & $4245.41 \pm 2526.37$ & $4.80 \pm 1.11$ & $2.48 \pm 0.60$ & $-1.09 \pm 0.04$ & $-6.78 \pm 0.49$ & $\bm{2.74 \pm 0.02}$ & $2.89 \pm 0.01$ & $3.16 \pm 0.39$ & $96.61 \pm 53.56$ \\
VIP (frozen) & $3.34 \pm 0.17$ & $3.89 \pm 0.07$ & $3.02 \pm 0.12$ & $-0.04 \pm 0.03$ & $-3.23 \pm 0.09$ & $3.36 \pm 0.13$ & $3.13 \pm 0.01$ & $1.07 \pm 0.05$ & $3.83 \pm 0.17$ \\
FTIP (tunable) & $1789.45 \pm 735.75$ & $4.82 \pm 1.21$ & $2.13 \pm 0.36$ & $-1.11 \pm 0.02$ & $-6.47 \pm 0.59$ & $2.75 \pm 0.03$ & $2.89 \pm 0.01$ & $3.52 \pm 0.45$ & $56.66 \pm 56.15$ \\
FTIP (frozen) & $3.34 \pm 0.18$ & $3.89 \pm 0.07$ & $3.02 \pm 0.12$ & $-0.03 \pm 0.03$ & $-3.29 \pm 0.10$ & $3.36 \pm 0.13$ & $3.13 \pm 0.01$ & $1.07 \pm 0.05$ & $3.84 \pm 0.17$ \\
GMVIP (tunable) & $\bm{3.11 \pm 1.03}$ & $\bm{2.82 \pm 0.16}$ & $\bm{0.75 \pm 0.09}$ & $\bm{-1.26 \pm 0.02}$ & $\bm{-7.12 \pm 0.04}$ & $2.75 \pm 0.03$ & $\bm{2.88 \pm 0.01}$ & $\bm{0.94 \pm 0.08}$ & $\bm{0.18 \pm 0.22}$ \\
GMVIP (frozen) & $3.62 \pm 0.02$ & $4.04 \pm 0.03$ & $3.45 \pm 0.01$ & $0.34 \pm 0.00$ & $-5.15 \pm 0.02$ & $3.42 \pm 0.01$ & $3.11 \pm 0.00$ & $1.46 \pm 0.00$ & $3.28 \pm 0.03$ \\
\midrule
\textbf{CRPS} & Boston & Concrete & Energy & Kin8nm & Naval ($\times 10^{-4}$) & Power & Protein & Wine Red & Yacht \\
\midrule
VIP (tunable) & $4.70 \pm 0.88$ & $3.26 \pm 0.48$ & $0.33 \pm 0.05$ & $0.045 \pm 0.002$ & $1.50 \pm 0.64$ & $\bm{2.06 \pm 0.05}$ & $2.43 \pm 0.03$ & $0.48 \pm 0.02$ & $0.40 \pm 0.12$ \\
VIP (frozen) & $3.74 \pm 0.62$ & $6.67 \pm 0.50$ & $2.78 \pm 0.38$ & $0.132 \pm 0.004$ & $53.79 \pm 5.37$ & $3.94 \pm 0.54$ & $3.15 \pm 0.03$ & $0.39 \pm 0.02$ & $6.10 \pm 1.02$ \\
FTIP (tunable) & $4.92 \pm 1.25$ & $3.40 \pm 0.65$ & $0.32 \pm 0.04$ & $0.044 \pm 0.001$ & $2.03 \pm 1.05$ & $2.06 \pm 0.05$ & $2.43 \pm 0.03$ & $0.48 \pm 0.02$ & $0.56 \pm 0.18$ \\
FTIP (frozen) & $3.75 \pm 0.64$ & $6.66 \pm 0.47$ & $2.78 \pm 0.38$ & $0.132 \pm 0.004$ & $50.24 \pm 5.36$ & $3.95 \pm 0.54$ & $3.15 \pm 0.03$ & $0.39 \pm 0.02$ & $6.10 \pm 1.02$ \\
GMVIP (tunable) & $\bm{1.79 \pm 0.37}$ & $\bm{2.26 \pm 0.35}$ & $\bm{0.28 \pm 0.02}$ & $\bm{0.038 \pm 0.001}$ & $\bm{0.92 \pm 0.07}$ & $2.07 \pm 0.06$ & $\bm{2.41 \pm 0.02}$ & $\bm{0.33 \pm 0.01}$ & $\bm{0.16 \pm 0.04}$ \\
GMVIP (frozen) & $3.82 \pm 0.20$ & $6.09 \pm 0.14$ & $3.16 \pm 0.02$ & $0.143 \pm 0.001$ & $5.69 \pm 0.14$ & $3.27 \pm 0.03$ & $2.79 \pm 0.02$ & $0.46 \pm 0.00$ & $2.90 \pm 0.21$ \\
\midrule
\textbf{CQM} & Boston & Concrete & Energy & Kin8nm & Naval & Power & Protein & Wine Red & Yacht \\
\midrule
VIP (tunable) & $0.495 \pm 0.003$ & $0.165 \pm 0.036$ & $0.172 \pm 0.033$ & $0.020 \pm 0.013$ & $0.130 \pm 0.062$ & $0.023 \pm 0.010$ & $0.021 \pm 0.001$ & $0.217 \pm 0.009$ & $0.342 \pm 0.052$ \\
VIP (frozen) & $0.054 \pm 0.028$ & $\bm{0.025 \pm 0.009}$ & $\bm{0.034 \pm 0.020}$ & $\bm{0.009 \pm 0.004}$ & $\bm{0.026 \pm 0.022}$ & $\bm{0.014 \pm 0.003}$ & $0.031 \pm 0.009$ & $\bm{0.034 \pm 0.015}$ & $0.069 \pm 0.038$ \\
FTIP (tunable) & $0.483 \pm 0.008$ & $0.159 \pm 0.043$ & $0.213 \pm 0.043$ & $0.016 \pm 0.006$ & $0.171 \pm 0.070$ & $0.022 \pm 0.008$ & $\bm{0.017 \pm 0.004}$ & $0.209 \pm 0.021$ & $0.331 \pm 0.055$ \\
FTIP (frozen) & $\bm{0.053 \pm 0.028}$ & $0.026 \pm 0.009$ & $0.035 \pm 0.023$ & $0.010 \pm 0.005$ & $0.035 \pm 0.022$ & $0.014 \pm 0.003$ & $0.033 \pm 0.011$ & $0.034 \pm 0.015$ & $\bm{0.068 \pm 0.039}$ \\
GMVIP (tunable) & $0.119 \pm 0.065$ & $0.056 \pm 0.031$ & $0.077 \pm 0.026$ & $0.034 \pm 0.010$ & $0.244 \pm 0.044$ & $0.037 \pm 0.010$ & $0.031 \pm 0.005$ & $0.046 \pm 0.012$ & $0.198 \pm 0.053$ \\
GMVIP (frozen) & $0.353 \pm 0.014$ & $0.305 \pm 0.013$ & $0.371 \pm 0.007$ & $0.340 \pm 0.005$ & $0.407 \pm 0.008$ & $0.293 \pm 0.005$ & $0.144 \pm 0.005$ & $0.278 \pm 0.013$ & $0.310 \pm 0.037$ \\
\bottomrule
\end{tabular}%
}
\end{table*}

\subsection{VIP and FTIP Basis-Size Sensitivity on Lotka--Volterra}
\label{app:lv-basis-sweep}

To test whether the forecasting comparison is limited by the finite
sampled-function basis, we rerun VIP and FTIP with
\(S\in\{20, 64, 128, 256\}\) on the same 20 target trajectories and
evaluation protocol as Appendix~\ref{app:lv-implicit-prior}. The bases are
nested within a common master prior bank, and all other method-specific
settings are held fixed. Each coefficient-space method receives 800 total
optimization steps: VIP is trained directly for 800 steps, while FTIP uses
a separate 400-step VIP warm start followed by 400 fine-tuning steps.
Every run is evaluated at its final scheduled iterate, without validation or
checkpoint selection, using the same 1,024-sample evaluation protocol.
Because \(S\) counts global sampled-function basis elements, whereas GMVIP
uses inducing values, this experiment is a within-method capacity
sensitivity analysis rather than a direct equality of representation
dimensions.

\begin{table*}[t]
\centering
\caption{Sensitivity of VIP and FTIP to the number \(S\) of sampled
prior basis functions in the Lotka--Volterra experiment. Entries are
means \(\pm\) standard errors across 20 target trajectories on the test
interval \(t\in(20,30]\). VIP is trained for 800 iterations. FTIP is
warm-started from a separate 400-iteration VIP fit at the same \(S\) and
then fine-tuned for 400 additional iterations. GMVIP uses 96 inducing
locations and is included as a reference. Coverage is ranked by proximity
to the nominal \(90\%\) level; lower is better otherwise. Best and
second-best results, determined using unrounded values, are highlighted.}
\label{tab:lv-basis-sweep}
\small
\scalebox{0.9}{
\begin{tabular}{llccccc}
\toprule
Method & Representation size
& RMSE \(\downarrow\)
& NLL \(\downarrow\)
& CRPS \(\downarrow\)
& Cov$_{90}$
& ODE residual \(\downarrow\) \\
\midrule
VIP
& \(S=20\)
& \(0.50 \pm 0.08\)
& \(11.32 \pm 5.01\)
& \(0.31 \pm 0.05\)
& \(0.72 \pm 0.06\)
& \(0.65 \pm 0.03\) \\

VIP
& \(S=64\)
& \(0.37 \pm 0.08\)
& \(6.20 \pm 5.87\)
& \(0.23 \pm 0.05\)
& \(0.87 \pm 0.04\)
& \(0.69 \pm 0.02\) \\

VIP
& \(S=128\)
& \(0.36 \pm 0.09\)
& \(7.93 \pm 7.62\)
& \(0.23 \pm 0.06\)
& \(\second{0.89 \pm 0.04}\)
& \(0.71 \pm 0.03\) \\

VIP
& \(S=256\)
& \(0.35 \pm 0.09\)
& \(9.97 \pm 9.71\)
& \(\second{0.22 \pm 0.06}\)
& \(0.89 \pm 0.04\)
& \(0.69 \pm 0.03\) \\
\midrule
FTIP
& \(S=20\)
& \(0.57 \pm 0.08\)
& \(1.04 \pm 0.22\)
& \(0.34 \pm 0.05\)
& \(0.82 \pm 0.04\)
& \(\second{0.62 \pm 0.03}\) \\

FTIP
& \(S=64\)
& \(0.43 \pm 0.13\)
& \(1.80 \pm 1.41\)
& \(0.29 \pm 0.10\)
& \(0.92 \pm 0.04\)
& \(0.67 \pm 0.02\) \\

FTIP
& \(S=128\)
& \(0.37 \pm 0.10\)
& \(1.22 \pm 0.84\)
& \(0.25 \pm 0.07\)
& \(0.95 \pm 0.04\)
& \(0.68 \pm 0.02\) \\

FTIP
& \(S=256\)
& \(\second{0.33 \pm 0.08}\)
& \(\second{0.72 \pm 0.37}\)
& \(0.23 \pm 0.06\)
& \(0.95 \pm 0.03\)
& \(0.68 \pm 0.02\) \\
\midrule
GMVIP
& \(M=96\)
& \(\best{0.32 \pm 0.06}\)
& \(\best{0.51 \pm 0.29}\)
& \(\best{0.20 \pm 0.04}\)
& \(\best{0.90 \pm 0.04}\)
& \(\best{0.31 \pm 0.02}\) \\
\bottomrule
\end{tabular}}
\end{table*}

The results in Table~\ref{tab:lv-basis-sweep} show that increasing \(S\)
generally improves point forecasting, but does not yield monotone gains in
every metric. GMVIP has the lowest mean RMSE, NLL, CRPS, and ODE
residual among the learned configurations, together with coverage closest to
the nominal \(0.90\). FTIP at \(S=256\) is second in RMSE and NLL, while VIP
at \(S=128\) is second in coverage.

\subsection{Efficiency and capacity diagnostics}
To contextualize the empirical gains, Table~\ref{tab:airline-efficiency-diagnostics}
reports GPU wall-clock setup time, per-step training time, evaluation time, peak
CUDA memory, latent or coefficient dimension, and the number of prior samples
used by each method on Airline; on a NVIDIA RTX 5070. This diagnostic is especially important for
empirical GMVIP and SIP: both use inducing-variable representations with
\(100\) inducing variables and \(512\) prior samples, while empirical GMVIP
additionally constructs a Matheron operator from prior-bank evaluations and
empirical covariance estimates. The table therefore separates predictive
performance from the computational capacity and overhead required to obtain it.

\begin{table}[t]
\centering
\caption{GPU efficiency diagnostics on Airline. Setup includes model construction and method-specific preparation; train time is mean wall-clock time per measured optimizer step after warmup. Memory reports peak CUDA allocated/reserved memory.}
\begin{tabular}{lcccccc}
\toprule
Method & Setup & Train / step & Eval time & CUDA memory & Latent dim. & Prior samples \\
\midrule
VIP   & 0.04s & 0.006s & 0.247s & 28 / 48 MB   & 20  & -- \\
FTIP  & 0.16s & 0.023s & 0.216s & 84 / 112 MB  & 20  & 512 \\
FBNN  & 0.03s & 0.044s & 0.247s & 176 / 290 MB & 512 & 512 \\
TFSVI & 4.09s & 0.067s & 0.232s & 249 / 378 MB & 512 & -- \\
SIP   & 2.34s & 0.021s & 0.357s & 626 / 672 MB & 100 & 512 \\
GMVIP & 1.17s & 0.031s & 0.489s & 626 / 672 MB & 100 & 512 \\
\bottomrule
\end{tabular}\label{tab:airline-efficiency-diagnostics}
\end{table}

\end{document}

%% file: tables/uci_rmse.tex
\begin{table*}[t]
\centering
\scriptsize
\setlength{\tabcolsep}{3.5pt}
\renewcommand{\arraystretch}{0.92}
\caption{UCI regression metrics ($\downarrow$). Entries are mean $\pm$ standard deviation across five seeds. Best mean per dataset is shown in \textcolor{purple}{\textbf{purple}}; second-best mean is shown in \textcolor{teal}{\textbf{teal}}.}
\label{tab:uci-all-metrics-gmvip-wandb}
\resizebox{\textwidth}{!}{%
\begin{tabular}{lccccccccc}
\toprule
\textbf{RMSE} & Boston & Concrete & Energy & Kin8nm & Naval ($\times 10^{-4}$) & Power & Protein & Wine Red &Yacht \\
\midrule
MAP & $4.34 \pm 1.41$ & $5.45 \pm 0.57$ & $0.52 \pm 0.05$ & $0.082 \pm 0.001$ & $4.04 \pm 1.97$ & $3.94 \pm 0.13$ & $4.62 \pm 0.07$ & $0.69 \pm 0.01$ & $0.52 \pm 0.20$ \\
MFVI & $3.93 \pm 0.70$ & $6.23 \pm 0.51$ & $0.62 \pm 0.09$ & $0.091 \pm 0.002$ & $147.61 \pm 0.90$ & $4.12 \pm 0.17$ & $4.94 \pm 0.08$ & $\second{0.62 \pm 0.03}$ & $0.68 \pm 0.10$ \\
VIP & $6.49 \pm 1.27$ & $5.76 \pm 0.97$ & $0.58 \pm 0.09$ & $0.080 \pm 0.003$ & $2.59 \pm 0.96$ & $\best{3.76 \pm 0.07}$ & $4.37 \pm 0.06$ & $0.79 \pm 0.05$ & $0.84 \pm 0.30$ \\
FBNN & $4.17 \pm 1.27$ & $\second{5.21 \pm 0.52}$ & $\second{0.51 \pm 0.05}$ & $0.086 \pm 0.001$ & $3.95 \pm 0.47$ & $4.05 \pm 0.16$ & $4.81 \pm 0.05$ & $0.65 \pm 0.04$ & $\best{0.23 \pm 0.06}$ \\
SIP & $\best{3.22 \pm 0.99}$ & $5.30 \pm 0.27$ & $0.52 \pm 0.04$ & $\second{0.071 \pm 0.001}$ & $\second{1.48 \pm 0.23}$ & $3.84 \pm 0.11$ & $4.62 \pm 0.05$ & $0.66 \pm 0.03$ & $0.41 \pm 0.15$ \\
TFSVI & $6.39 \pm 0.63$ & $5.97 \pm 0.85$ & $0.59 \pm 0.15$ & $0.081 \pm 0.002$ & $3.95 \pm 0.81$ & $3.90 \pm 0.13$ & $4.61 \pm 0.05$ & $0.73 \pm 0.02$ & $0.83 \pm 0.38$ \\
FTIP & $6.68 \pm 1.58$ & $6.02 \pm 1.08$ & $0.55 \pm 0.07$ & $0.079 \pm 0.002$ & $3.36 \pm 1.58$ & $\second{3.77 \pm 0.11}$ & $\second{4.36 \pm 0.05}$ & $0.82 \pm 0.03$ & $1.17 \pm 0.52$ \\
GMVIP & $\second{3.32 \pm 0.97}$ & $\best{4.24 \pm 0.74}$ & $\best{0.51 \pm 0.05}$ & $\best{0.069 \pm 0.001}$ & $\best{1.36 \pm 0.19}$ & $3.78 \pm 0.12$ & $\best{4.32 \pm 0.04}$ & $\best{0.60 \pm 0.03}$ & $\second{0.32 \pm 0.11}$ \\
\midrule
\textbf{NLL} & Boston & Concrete & Energy & Kin8nm & Naval & Power & Protein & Wine Red &Yacht \\
\midrule
MAP & $7.62 \pm 3.48$ & $3.24 \pm 0.17$ & $\second{0.81 \pm 0.14}$ & $-1.08 \pm 0.02$ & $-6.14 \pm 1.03$ & $2.79 \pm 0.03$ & $2.95 \pm 0.01$ & $1.27 \pm 0.04$ & $16.18 \pm 15.06$ \\
MFVI & $\best{2.76 \pm 0.14}$ & $3.25 \pm 0.08$ & $0.94 \pm 0.14$ & $-0.98 \pm 0.03$ & $-2.80 \pm 0.01$ & $2.84 \pm 0.04$ & $3.02 \pm 0.02$ & $\best{0.94 \pm 0.04}$ & $1.08 \pm 0.05$ \\
VIP & $4245.41 \pm 2526.37$ & $4.80 \pm 1.11$ & $2.48 \pm 0.60$ & $-1.09 \pm 0.04$ & $-6.78 \pm 0.49$ & $\best{2.74 \pm 0.02}$ & $2.89 \pm 0.01$ & $3.16 \pm 0.39$ & $96.61 \pm 53.56$ \\
FBNN & $5.24 \pm 2.45$ & $\second{3.13 \pm 0.16}$ & $0.84 \pm 0.22$ & $-1.04 \pm 0.01$ & $-6.42 \pm 0.11$ & $2.82 \pm 0.04$ & $2.99 \pm 0.01$ & $1.08 \pm 0.08$ & $\second{0.30 \pm 0.84}$ \\
SIP & $3.27 \pm 1.22$ & $3.18 \pm 0.09$ & $0.84 \pm 0.14$ & $\second{-1.22 \pm 0.02}$ & $\second{-7.04 \pm 0.29}$ & $2.76 \pm 0.03$ & $2.95 \pm 0.01$ & $1.09 \pm 0.07$ & $5.65 \pm 5.09$ \\
TFSVI & $11.32 \pm 2.79$ & $3.38 \pm 0.23$ & $1.05 \pm 0.44$ & $-1.09 \pm 0.03$ & $-6.41 \pm 0.22$ & $2.78 \pm 0.03$ & $2.95 \pm 0.01$ & $1.29 \pm 0.07$ & $14.30 \pm 10.07$ \\
FTIP & $1789.45 \pm 735.75$ & $4.82 \pm 1.21$ & $2.13 \pm 0.36$ & $-1.11 \pm 0.02$ & $-6.47 \pm 0.59$ & $\second{2.75 \pm 0.03}$ & $\second{2.89 \pm 0.01}$ & $3.52 \pm 0.45$ & $56.66 \pm 56.15$ \\
GMVIP & $\second{3.11 \pm 1.03}$ & $\best{2.82 \pm 0.16}$ & $\best{0.75 \pm 0.09}$ & $\best{-1.26 \pm 0.02}$ & $\best{-7.12 \pm 0.04}$ & $2.75 \pm 0.03$ & $\best{2.88 \pm 0.01}$ & $\second{0.94 \pm 0.08}$ & $\best{0.18 \pm 0.22}$ \\
\bottomrule
\end{tabular}%
}
\end{table*}

%% file: tables/regression.tex
\begin{table*}[t]
\centering
\scriptsize
\setlength{\tabcolsep}{3.0pt}
\renewcommand{\arraystretch}{0.92}
\caption{Large regression metrics ($\downarrow$). Entries are mean $\pm$ standard deviation across seeds. Best method is shown in \textcolor{purple}{\textbf{purple}}; second-best is shown in \textcolor{teal}{\textbf{teal}}.}
\label{tab:large-regression-all-metrics}
\resizebox{\textwidth}{!}{%
\begin{tabular}{clccccccc}
\toprule
\textbf{Metric} & \textbf{Dataset} & MFVI & VIP & FBNN & SIP & TFSVI & FTIP & GMVIP \\
\midrule
\multirow{2}{*}{\textbf{RMSE}}
& Year    & $\second{8.95 \pm 0.01}$ & $9.17 \pm 0.01$ & $\best{8.91 \pm 0.02}$ & $9.06 \pm 0.01$ & $9.05 \pm 0.01$ & $9.15 \pm 0.01$ & $8.98 \pm 0.02$ \\
& Airline & $37.87 \pm 0.12$ & $38.03 \pm 0.16$ & $37.90 \pm 0.09$ & $37.99 \pm 0.07$ & $\second{37.72 \pm 0.13}$ & $38.04 \pm 0.12$ & $\best{37.62 \pm 0.10}$ \\
\midrule
\multirow{2}{*}{\textbf{NLL}}
& Year    & $\second{3.61 \pm 0.00}$ & $3.64 \pm 0.00$ & $\best{3.60 \pm 0.00}$ & $3.62 \pm 0.00$ & $3.62 \pm 0.00$ & $3.63 \pm 0.00$ & $\second{3.61 \pm 0.00}$ \\
& Airline & $\best{5.07 \pm 0.00}$ & $5.09 \pm 0.01$ & $\best{5.07 \pm 0.00}$ & $\second{5.08 \pm 0.00}$ & $5.09 \pm 0.00$ & $5.09 \pm 0.01$ & $\best{5.07 \pm 0.01}$ \\
\midrule
\multirow{2}{*}{\textbf{CRPS}}
& Year    & $\second{4.717 \pm 0.007}$ & $4.859 \pm 0.010$ & $\best{4.675 \pm 0.012}$ & $4.785 \pm 0.010$ & $4.777 \pm 0.012$ & $4.845 \pm 0.004$ & $4.740 \pm 0.014$ \\
& Airline & $17.67 \pm 0.13$ & $17.79 \pm 0.22$ & $17.70 \pm 0.17$ & $17.71 \pm 0.08$ & $\best{17.47 \pm 0.13}$ & $17.79 \pm 0.17$ & $\second{17.61 \pm 0.04}$ \\
\midrule
\multirow{2}{*}{\textbf{CQM}}
& Year    & $0.083 \pm 0.002$ & $\best{0.064 \pm 0.002}$ & $0.084 \pm 0.003$ & $\second{0.078 \pm 0.001}$ & $\second{0.078 \pm 0.003}$ & $\best{0.064 \pm 0.002}$ & $\second{0.078 \pm 0.003}$ \\
& Airline & $0.107 \pm 0.009$ & $\best{0.096 \pm 0.015}$ & $0.109 \pm 0.014$ & $0.106 \pm 0.003$ & $0.111 \pm 0.010$ & $\best{0.096 \pm 0.012}$ & $\second{0.102 \pm 0.004}$ \\
\bottomrule
\end{tabular}%
}
\end{table*}